\documentclass[article]{article}
\pdfoutput=1
\usepackage[utf8]{inputenc} 
\usepackage[T1]{fontenc}    
\usepackage{url}            

\usepackage{fullpage}
\usepackage[round]{natbib}

\usepackage{booktabs}       
\usepackage{amsfonts}       
\usepackage{nicefrac}       
\usepackage{microtype} 
\usepackage[dvipsnames]{xcolor}
\usepackage[backref=page,colorlinks=true,linkcolor=magenta, citecolor=NavyBlue]{hyperref}%
\usepackage{epsf}
\usepackage{graphics}
\usepackage{pifont}
\usepackage{subcaption}
\usepackage{psfrag}

\usepackage{color}

\usepackage{mathtools}
\usepackage{amsfonts}
\usepackage{amsthm}
\usepackage{amsmath}
\usepackage{amssymb}

\newtheorem{theorem}{Theorem}

\newtheorem{lemma}{Lemma}
\newtheorem{corollary}{Corollary}
\newtheorem{definition}{Definition}
\newtheorem{remark}{Remark}

\long\def\comment#1{}

\newcommand{\bm}[1]{\mathbf{#1}}

\DeclareMathOperator{\trace}{\mathsf{tr}}

\usepackage{marvosym} 

\usepackage[english]{babel}
\usepackage[ruled,algo2e]{algorithm2e}

\SetCommentSty{mycommfont}
\usepackage{appendix}

\newcommand{\order}{\ensuremath{\mathcal{O}}}

\usepackage{amsfonts}
\definecolor{aoenglish}{rgb}{0.0, 0.5, 0.0}
\definecolor{burgundy}{rgb}{0.5, 0.0, 0.13}
\begin{document}

\begin{center}

{\bf{\LARGE{Efficient Fair Principal Component Analysis }}}

\vspace{.2in}

{\large{
\begin{tabular}{cc}
Mohammad\,Mahdi Kamani$^\dagger$~~Farzin Haddadpour$^\ddagger$~~Rana Forsati$^\ast$~~Mehrdad Mahdavi$^\ddagger$ \\
\end{tabular}
}}
\vspace{.2in}

\begin{tabular}{c}
$^\dagger$College of Information Sciences and Technology \\
$^\ddagger$School of  Electrical Engineering and Computer Science\\
The Pennsylvania State University\\
University Park, PA, USA \\
\texttt{\{mqk5591, fxh18, mzm616\}@psu.edu}
\\
\\
$^\ast$Microsoft Bing \\
Bellevue, WA, USA\\
\texttt{raforsat@microsoft.com}
\end{tabular}
\vspace*{.2in}

\date{\today}

\end{center}

\begin{abstract}
It has been shown that dimension reduction methods such as PCA may be inherently prone to unfairness and treat data from different sensitive groups such as race, color, sex, etc., unfairly. In pursuit of fairness-enhancing dimensionality reduction, using the notion of Pareto optimality, we propose an adaptive first-order algorithm to learn a  subspace that preserves fairness, while slightly compromising the reconstruction loss. Theoretically, we provide sufficient conditions that the solution of the proposed algorithm belongs to the Pareto frontier for all sensitive groups; thereby, the optimal trade-off between overall reconstruction loss and fairness constraints is guaranteed. We also provide the convergence analysis of our algorithm and show its efficacy through empirical studies on different datasets, which demonstrates superior performance in comparison with state-of-the-art algorithms. The proposed fairness-aware PCA algorithm can be efficiently generalized to multiple group sensitive features and effectively reduce the unfairness decisions in downstream tasks such as classification. 
\end{abstract}
\section{Introduction}\label{sec:intro}
Recent advances in machine learning (ML) have vastly improved the capabilities of computational reasoning in complex domains. From tasks like image and video processing, game playing, text classification, to complex data analysis, machine learning is continually finding new applications and exceeding human-level performance in some cases. Nevertheless,  when machine learning models are trained on real data, the existing societal inequalities in data are manifested on the systems built upon them that could mislead models in ways that can have profound fairness implications such as being biased to sensitive features like race or gender. As more critical systems employ ML, such as financial systems, hiring and admissions, healthcare, and law, it is vitally important that we develop rigorous fair algorithms that are as accurate as possible.  

Recently, the growing attention to the fairness problem in algorithmic decision-making systems has led to unprecedented attempts to revisit machine learning models for supervised and unsupervised tasks to satisfy fairness constraints~\citep{executive2016big}.
An expanding line of work dedicated to define different metrics for fairness problems and mechanisms to satisfy those measures in learning tasks such as~\cite{hardt2016equality,zafar2017parity,calders2009building,calders2010three,kamishima2011fairness, agarwal2018reductions,zafar2017fairness,zafar2015fairness}. The work on this realm is focused on biased data or biased algorithms; however, using these biased algorithms in decision-making systems would lead to generating more biased data. This makes the causality of the fairness problem more complicated that exacerbates the problem even further~\citep{barocas2017fairness,ghili2019eliminating}. 

Notwithstanding these flourishing efforts for fairness problem in supervised learning, fairness in unsupervised learning tasks has not been explored thoroughly. This is despite the fact that unsupervised learning tasks such as dimension reductions are mostly preceding those supervised ones, in the training procedures. Hence, having fair unsupervised learning models is as crucial as supervised ones. For instance, Principal Component Analysis (PCA) is widely used to reduce the dimension of the data before applying classification models. In addition to that, these unsupervised methods such as dimension reductions or clustering methods are commonly used for data visualizations, identifying common behaviors or trends, reduce the size of data, to name but a few. This ubiquitousness of unsupervised methods in machine learning models can affect decision-making systems if they unfairly treat different groups in data.

In this paper, we aim at defining a fairness measure for dimension reduction algorithms like PCA and propose an algorithm to enforce these criteria in finding the subspace with minimum reconstruction loss. It is important to note that, despite supervised learning that fairness metrics are mostly focused on the beneficial outcome (usually the positive label), in an unsupervised task, there is no label to be used. Hence, we seek to find a subspace that is ``good'' enough for each protected group in the data. Indeed, when we apply PCA on a dataset, the resulting subspace found by a standard algorithm is different from what we achieve when only using the data of each group individually. This difference can be reflected as the difference between the reconstruction error of each group's data on both subspaces. Thus, when a dimension reduction algorithm is applied to the joint data, the reconstruction loss of some of the groups is degraded  (from what they can achieve with their data only), while others are benefiting from joint learning. Here, a fair algorithm is the one that can find a subspace with optimal trade-offs between these degradations and benefits.

\begin{table}[t]
\centering
\begin{tabular}{|c|c|c|c|}
\hline
                                  Scheme            & Time Complexity     & Fair Dimension   & Algorithm                                       \\ \hline
\multicolumn{1}{|c|}{\cite{samadi2018price}}      & $\order\left(d^{6.5}\log\left(\frac{1}{\epsilon}\right)\right)$ & $r+k-1$   & SDP + LP              \\ \hline
\multicolumn{1}{|c|}{\cite{samadi2018price}}      & $\order(\frac{d^3}{\epsilon^2})$ & $r+k-1$ & SDP via MW  + LP              \\ \hline
\multicolumn{1}{|c|}{\cite{morgenstern2019fair}} & $\order\left(d^{6.5}\log\left(\frac{1}{\epsilon}\right)\right)$ & $r+\lfloor \sqrt{2k+\frac{1}{4}} - \frac{3}{2}\rfloor$ & SDP + LP    \\ \hline
\multicolumn{1}{|c|}{\cite{morgenstern2019fair}} & $\order(\frac{d^3}{\epsilon^2})$ & $r+\lfloor \sqrt{2k+\frac{1}{4}} - \frac{3}{2}\rfloor$ & SDP via MW + LP   \\ \hline
\multicolumn{1}{|c|}{\textbf{This Work}}                    & \boldmath{$\order(\frac{r^2d}{\epsilon^2})$}            & \boldmath{$r$}     & GD  \\ \hline

\end{tabular}
\caption{\sffamily{Comparison of time complexity of different fair PCA algorithms to achieve an $\epsilon$-fair subspace (please see Section~\ref{sec:form} for definition). Here $d$ denote the dimension of the original data, $r$ is the target dimension, and $k$ is the number of sensitive groups. Note that unlike previous studies that necessitates learning a  subspace with larger dimension to   guarantee fairness, our solution learns an exact $r$ dimensional subspace by imposing additional constrains captured by a new notion of fairness proposed in this work to distinguish between local optimal fair subspaces (we used the following abbreviations  above, SDP: Semi-Definite Programming, MW: Multiplicative Weight Algorithm, LP: Linear Programming, GD: Gradient Descent). Note that all the algorithms have an initial step of finding the optimal rank $r$ subspace for each group, in which its time complexity is not included here.}}\label{tab:comp}
\vspace{-0.5cm}
\end{table}

An attempt to impose the fairness constraint on learning the optimal subspace for two protected groups has been made recently in~\cite{samadi2018price,olfat2018convex,morgenstern2019fair}, where the fair subspace learning is sought by minimizing the maximum \textit{deviation of reconstruction error}  suffered by any protected group (i.e., the difference of per group reconstruction error and joint reconstruction error). Interestingly, it has been shown that at any optimal local solution of the optimization problem associated with learning such a fair subspace, all the groups suffer the same loss. Motivated by this observation,  a  semi-definite programming relaxation followed by linear programming is proposed to find a fair subspace~\citep{samadi2018price,olfat2018convex,morgenstern2019fair}. In addition to the computational inefficiency of algorithms proposed by these works, the generalization of them to multiple group sensitive features is not conspicuous. Furthermore, since all optimal solutions do not assign same loss to all groups, extra dimensions are needed to ensure that the total loss of the projection remains at most the optimal objective in the original target dimension (in particular,  $k -1$ extra dimensions are needed for $k$ groups in~\cite{samadi2018price} which is further tightened to $\sqrt{k}$ in a followup work~\citep{morgenstern2019fair}). 

The overarching goal of this paper is to define a fairness metric for dimension reduction, dubbed as pairwise disparity error, and propose a computationally efficient dimensional reduction algorithm to learn a fair subspace from multiple group sensitive features. Towards this end, we cast the problem of fairness in the PCA dimension reduction algorithm as a multi-objective optimization problem and propose an adaptive gradient descent based approach to find the optimal trade-offs with provable convergence rates. Interestingly, the proposed framework is not bounded to any specific notion of fairness metric and can be effortlessly applied to other metrics as well. Moreover, unlike the aforementioned prior works, no extra dimension is needed to ensure the loss suffered by each group matches the optimal fairness loss. The comparison of time complexity of exiting algorithms and current work is summarized in Table~\ref{tab:comp}.

\paragraph{Contributions}
The main contributions of this paper can be summed up as follows: 
\begin{itemize}
    \item We introduce the notion of Pareto fair PCA to ponder  conflicting objectives and achieve optimal trade-offs between them. Also, we introduce the notion of \textbf{pairwise disparity error} as a more efficient objective to learn fair subspaces. In addition, we provide conditions, under which a Pareto optimal solution exists.
    \item We propose a gradient descent algorithm to efficiently solve the obtained multi-objective optimization problem which is interesting by its own right, and provide theoretical  guarantees on  its convergence to optimal compromises or a Pareto stationary point.
    \item We empirically develop this algorithm and compare it to the state-of-the-art algorithm on two real-world datasets to demonstrate its efficacy that complements our theoretical results.
    \item We investigate the effect of fair projection  on supervised tasks such as classification  empirically and show that it  can significantly eliminate the unfairness in downstream tasks.
\end{itemize}

\section{Related Work}
The efforts to address fairness in algorithmic decision-making systems have roughly fallen into three different categories. Some scholars believe data itself could be biased, leading to unfair results; thus, they seek to solve this problem on data level and as a preprocessing step to the main learning task~\citep{dwork2012fairness,feldman2015certifying, kamiran2009classifying,calders2009building}. The goal is achieved by either changing the value of sensitive feature or label data or find a subspace, where labels and sensitive features are independent. However, since the main objective of the learning is not involved in this process, the optimal solution for the main objective is not guaranteed. The second category includes methods that try to impose the fairness criteria after the learning, in order to attain a fair model~\citep{hardt2016equality,kamishima2011fairness, goh2016satisfying,calders2010three}. The third approach, includes methods that try to satisfy fairness constraint during the training procedure, usually by imposing them as a constraint to the main learning objective~\citep{donini2018empirical,morgenstern2019fair,zafar2015fairness,samadi2018price,pleiss2017fairness}. Some of these approaches treat the fairness problem similar to imbalanced data or rare event prediction~\citep{yao2017beyond,kamani2019targeted,kamani2018skeleton,kamani2016shape,nikbakht2019direct}. While these approaches can achieve the state-of-the-art results in some problems, they still suffer from several issues. Solving a constrained optimization could be a very hard non-convex problem; hence, relaxation is needed to solve the problem that leads to sub-optimal solutions efficiently. Moreover, finding the optimal {penalization parameter} could be a difficult task, as discussed in~\cite{donini2018empirical}. Our approach belongs to the third category, yet, it differs from the prevailing trend of formulating the fairness problem as a constrained optimization. We will cast the fairness problem as a multi-objective optimization that can efficiently satisfy fairness objectives as well as the main learning objective and converge to a point with optimal compromises between objectives.

Fairness in dimension reduction algorithms is recently being vetted by~\cite{samadi2018price}, through which they propose a semi-definite programming and prove that its solution satisfies the proposed notion of fairness. Aside from the inefficiency of solving the SDP, their approach is developed for binary sensitive features and requires one extra dimension to guarantee fairness. To generalize it for multiple group sensitive features with $k$ groups, they propose to add  $k-1$ dimensions, which is impractical. The follow-up studies by~\cite{olfat2018convex,morgenstern2019fair} are still in line with the previous one, trying to relax and solve an SDP. We, on the other hand, propose an efficient gradient-based method to solve the aforementioned multi-objective optimization, with the capability of generalizing to multiple group sensitive features smoothly. 

Although it has been asserted that fairness problems are multi-objective problems in nature~\citep{kearns2019ethical, lipton2017does,morgenstern2019fair},  as mentioned before, most of the existing works apply different forms of relaxations and approximations to reduce the problem into a scalar-valued optimization problem. In this paper, we design the fairness problem at hand as a multi-objective optimization and solve it directly. Multi-objective or vector optimization is a well-studied problem in different domains for many years. The goal in this optimization is to achieve an optimal trade-off point between different objectives, known as Pareto optimal, named after Italian economist Vilfredo Pareto. We refer the reader to \citet{miettinen2012nonlinear} and the references therein as a rich resource on multi-objective optimization. We will elaborate that directly solving the vector-valued problem associated with fair learning is appealing to reduction based counterparts~\citep{ehrgott2006discussion, mahdavi2013stochastic} by being computationally efficient and providing provable guarantees on the fairness metric.

Beyond achieving fairness in unsupervised tasks such as PCA, the main goal of fairness in machine learning is to design a fair system as a whole. As it is noted by~\cite{dwork2018decoupled}, these machine learning models in isolation do not necessarily result in a fair system together and should be considered in composition with each other. Hence, in addition to what introduced by~\cite{dwork2018fairness} as compositions in fairness, we advocate for considering the composition of a stream of machine learning models together. Thus, we should investigate the effect of imposing fairness constraints on a machine learning model on downstream tasks using its output. For instance, the goal of defining such a metric for fairness, in our paper and other related works, is that having a fair loss in reducing the dimension would have a fair reduction in the quality of different groups in the new projection; then, it can have a balanced impact on the quality of a subsequent classifier learned on that projection. We empirically investigate the effect of this composition and leave its theoretical understanding to the future work.

\section{Problem Formulation}\label{sec:form}
We start by mathematically defining the problem we ought to solve, and then discuss what is the notion of fairness in PCA algorithm, which could be quite different from what is known as fairness measures in supervised learning. In what follows  we adapt the following notation. We use bold face upper case letters such as $\bm{X}$ to denote matrices and bold face lower case to denote vectors such as $\bm{f}$. The Frobenius norm and trace of  a matrix  $\bm{X}$ are denoted by $\|\bm{X}\|_{\mathrm{F}}$ and $\trace\left({\bm{X}}\right)$, respectively. The eigenvalues of a positive semi-definite matrix $\bm{\Sigma} \in \mathbb{R}^{d \times d}$ are denoted by $\gamma_{\max}(\bm{\Sigma}) = \gamma_1(\bm{\Sigma}) \geq \gamma_2(\bm{\Sigma}) \geq \ldots \geq \gamma_d(\bm{\Sigma}) = \gamma_{\min}(\bm{\Sigma})$. The set of  integers, $\{1,2,\ldots,m\}$, is represented by $\left[m\right]$. 

\subsection{PCA}\label{sec:pca}
The main objective of the PCA is to find the best representation of  the data $\bm{X} \in \mathbb{R}^{n\times d}$ with $n$ data points in $d$-dimensional space, in a lower dimension $r \leq d$ using a linear transformation, in order to have the minimum reconstruction error. This linear transformation can be represented by a projection matrix $\bm{U} \in \mathbb{R}^{d \times r}$. Thus, the objective of PCA is to find a projection matrix $\bm{U}$ and a recovery matrix $\bm{W} \in \mathbb{R}^{r \times d}$ to minimize this reconstruction error similar to~\cite{shalev2014understanding}:
\begin{equation}\label{eq:pca-recons1}
    \underset{\bm{U} \in \mathbb{R}^{d \times r}, \bm{W} \in \mathbb{R}^{r \times d} }{\arg \min} \left\|\bm{X} - \bm{X}\bm{U}\bm{W} \right\|_{\text{F}}^2
\end{equation}
It can be proved that in the solution of (\ref{eq:pca-recons1}), we have $\bm{W}=\bm{U}^\top$, and columns of $\bm{U}$ are orthonormal (i.e. $\bm{U}^\top\bm{U} = \bm{I}_{r \times r}$). Therefore we can define the reconstruction loss for any PCA projection as follows:
\begin{definition}[Reconstruction Loss]
For any given dataset $\bm{X}$ and any projection matrix $\bm{U}$, the total reconstruction loss of $\bm{X}$ using $\bm{U}$ is defined as:
\begin{equation}\label{eq:pca-reconst2}
    \mathcal{L}(\bm{U}) \triangleq \ell(\bm{X};\bm{U}) =  \left\|\bm{X} - \bm{X}\bm{U}\bm{U}^\top \right\|_{\mathrm{F}}^2
\end{equation}
\end{definition}
The optimal subspace with minimum reconstruction loss given $\bm{X}$ can be found by solving the above non-convex optimization problem.

\subsection{Fair PCA}\label{sec:fair-pca}
In this section, we will formally define the notion of fairness in dimension reduction algorithms such as PCA. As it was discussed before, the problem arises from having different reconstruction losses on different sensitive groups in a dataset. This means that finding an optimal projection matrix $\bm{U}^*$ by solving the minimization problem in (\ref{eq:pca-reconst2}), would have different reconstruction loss on data partitions from each sensitive group. However, in this problem, unlike supervised problems previously discussed, we are not able to reach equality between these reconstruction losses for different groups. The reason for that is the subspace for each group's data is different, and so is the reconstruction error of that data for that projection. We note that while learning a separate (local) subspace for each individual group has the optimal reconstruction error, our focus here is to learn a single global subspace for all groups due to statistical and ethical concerns. In particular, from a statistical standpoint, since the number of training samples for some groups might be small for skewed data sets, joint learning to have more samples to learn a subspace is preferable. Ethically, as elaborated in~\cite{lipton2017does} and \cite{kannan2019downstream}, learning separate subspaces (having disparate treatment like in affirmative action) constructs no trade-offs, and it poses several ethical and legal concerns. We note that the case of fairness with decoupled model representations has been investigated by several other works~\citep{dwork2018decoupled,ustun2019fairness,creager2019flexibly}.

In order to quantify to what extent each group suffers or benefits from  joint subspace learning,  we should compare the subspaces learned from each group's data alone and  the one with other groups' data included. Then, the idea of fairness is to reach a balance between these sacrifices and benefits of different groups. Formally, consider one of the $d$ features of $\bm{X}$ as a sensitive feature with $k$ different groups, $\mathcal{S} = \left\{ s_1, \ldots,s_k\right\}$. We denote the matrix of each group's data points as $\bm{X}_i \in \mathbb{R}^{n_i \times d}$, where $n_i$ is the number of samples belonging to the sensitive group $s_i$. Hence for any projection matrix $\bm{U} \in \mathbb{R}^{d \times r}$, the reconstruction loss for each group is defined as:
\begin{equation}\label{eq:pca-reconst-group}
    \mathcal{L}_i(\bm{U}) \triangleq \ell(\bm{X}_i;\bm{U}) =  \left\|\bm{X}_i - \bm{X}_i\bm{U}\bm{U}^\top \right\|_{\text{F}}^2, \;\; 1\leq i \leq k.
\end{equation}
Then, if we only use the dataset $\bm{X}_i$ to learn the projection matrix, we can find the subspace represented by $\bm{U}_i^*$ that has the optimal reconstruction loss on that dataset, denoted by $\mathcal{L}_i(\bm{U}_i^*)$. Therefore, a fair dimension reduction algorithm is the one that can learn a global projection matrix $\bm{U}^*$ on all data points with having equal distance between each group's reconstruction loss on the subspace learned by the whole data with the subspace learned only by its own data. To formally  define these fairness criteria, we introduce the notion of \textbf{disparity error} as follows:
\begin{definition}[Disparity Error]
Consider a dataset $\bm{X} \in \mathbb{R}^{n \times d}$ with $k$ sensitive groups with data matrix $\bm{X}_i, i=1,2,\ldots, k$ representing each sensitive group's data samples. Let $\bm{U}_i^* = \arg\min_{\bm{U}}\mathcal{L}_i(\bm{U})$ denote the projection matrix learned only based on $\bm{X}_i$. Then for any projection matrix $\bm{U}$  the disparity error for each sensitive group is defined as:
\begin{equation}\label{eq:de}
    \mathcal{E}_i\left(\bm{U}\right) = \mathcal{L}_i(\bm{U}) - \mathcal{L}_i(\bm{U}_i^*), \;\; 1\leq i\leq k.
\end{equation}
\label{def:de}\vspace{-0.5cm}
\end{definition}
This measure shows that how much reconstruction loss we are suffering or enjoying for any global projection matrix $\bm{U}$, with respect to the reconstruction loss of optimal projection matrix, we can learn locally based on data points $\bm{X}_i$. Note that calculating the optimal rank $r$ subspace for each group in $\mathcal{L}_i(\bm{U}_i^*)$ has a one-time overhead to the algorithm's time complexity overall. However, we ignore this overhead, as did other algorithms we are comparing to and leave the joint learning of both local and global subspaces as future work.

Using the Definition~\ref{def:de}, we can define a fair PCA algorithm as follows:
\begin{definition}[Fair PCA]
A PCA algorithm with projection matrix $\bm{U}^*$ is called \textbf{fair}, if the disparity error among different groups are equal. That is:
\begin{equation}\label{eq:fair-pca}
    \mathcal{E}_1\left(\bm{U}^*\right) = \mathcal{E}_2\left(\bm{U}^*\right) = \ldots = \mathcal{E}_k\left(\bm{U}^*\right).
\end{equation}
\label{def:fair-pca}
A subspace $\bm{U}^*$ that archives the same disparity error for all groups is called a fair subspace.
\end{definition}
\section{Pareto Fair Subspace}
In this section, we discuss the key challenges in finding a fair subspace using relaxation methods and motivate our formulation of Pareto fair subspace followed by   providing conditions sufficient to  guarantee the existence of such subspaces.

\subsection{Relaxation methods and their limitations}
A major challenge to find a fair subspace as defined in Definition~\ref{def:fair-pca} is to solve the optimization problem that satisfies~(\ref{eq:fair-pca}), which is essentially a multiple objective optimization problem by nature. To illustrate this and for ease of exposition, let us focus on the binary sensitive feature ($k = 2$), i.e., there are only  two groups in the sensitive feature of the data (e.g., male and female), in which the goal of fair PCA is to satisfy:
\begin{equation}\label{eq:fair-pca-binray}
    \mathcal{E}_1\left(\bm{U}^*\right) = \mathcal{E}_2\left(\bm{U}^*\right),
\end{equation}

In~\cite{samadi2018price}, it has been shown that by casting the multiobjective optimization problem as a minmax problem of the form
\begin{equation}\label{eq:fair:minimax}
    \underset{\bm{U} \in \mathbb{R}^{d \times r}, \texttt{rank}(\bm{U})\leq r}{\min} \max \left\{\mathcal{E}_1\left(\bm{U}\right), \mathcal{E}_2\left(\bm{U}\right) \right\},
\end{equation}
and using an additional dimension for the projection, the optimal solution of minmax problem results in the same loss for both groups (i.e., $\mathcal{E}_1\left(\bm{U}^*\right) = \mathcal{E}_2\left(\bm{U}^*\right))$.

Motivated by this observation, a  semi-definite relaxation to solve the optimization problem is proposed, which is not efficient for a large number of training samples. Also, to achieve their fairness criteria and ensure that the obtained local solution achieves the optimal fairness objective for all groups, the proposed solution requires adding an extra dimension for a binary sensitive feature and $k-1$ additional dimensions for a $k$-group sensitive feature, which is not reasonable for a large $k$. We note that in~\cite{morgenstern2019fair}, the requirement of  extra dimension is improved to $\lfloor \sqrt{2k+\frac{1}{4}} - \frac{3}{2}\rfloor$, but it still requires extra projection dimensions to satisfy the fairness constraint. Finally, the optimal trade-offs between fairness objectives and total reconstruction loss, in the case of the same target dimension, is not guaranteed, which would lead to a solution that sacrifices too much of the total reconstruction loss to achieve the fairness criteria.  In fact, in order to guarantee that the solution of minmax optimization results in a  rank $r$ subspace with \textit{optimal fairness} objective, one could choose the  target dimension to be $r-s$, where $s$ is the extra dimensions needed with $\order\left(\sqrt{k}\right)$ followed by a rounding to reach to the target $r$ dimensional subspace as proposed in~\cite{morgenstern2019fair}; but this remedy hurts the \textit{optimal objective} value by a multiplicative factor of $s/r$. This issue becomes more concerning as the number of groups  $k$, and hence $s$ increases due to the fact that all local optimal solutions might not achieve the same loss for all groups.  Consequently, any solution to fair PCA necessitates jointly minimizing the main objective, which is total reconstruction loss in (\ref{eq:pca-reconst2}), and fairness criteria to balance the trade-off between them. 

An alternative solution to alleviate aforementioned issues which is explored in~\cite{donini2018empirical} is to impose fairness constraints in minimizing the reconstruction loss in~(\ref{eq:pca-recons1}) as additional constraints, i.e., 
\begin{equation}
\begin{aligned}
& \underset{\bm{U}}{\min}
& & \mathcal{L}(\bm{U}) \\
& \text{subject to}
& & \mathcal{E}_i\left(\bm{U}\right) \leq \epsilon, \; i \in \left[k\right].
\label{eq:const-opt}
\end{aligned}
\end{equation}
which reduces the problem into an instance of \textit{non-convex constrained} optimization problem to find a fair subspace to all sensitive groups. Relaxing the problem of finding the fair subspace as a constrained optimization similar to (\ref{eq:const-opt}), apart from being a  hard non-convex problem which is not evident to solve due to presence of non-convex constraints, requires the optimal constraint violation parameter, $\boldsymbol{\epsilon}$, to be decided heuristically which is a burden on the use and makes the problem even harder. Although using the Lagrangian method we can turn the problem into an unconstrained non-convex optimization problem-- a method known as scalarization relaxation for multi-objective optimization counterpart (e.g., please see~\cite{ehrgott2006discussion}), deciding the Lagrangian multipliers is as hard as solving the original problem and does not guarantee the optimally of the obtained solution. Also, since the scale of the objectives might be different, it could lead to infeasibility issues in the optimization problem, or some points from the Pareto frontier could not be attained.

To address challenges arising from the above reduction methods, and in order to achieve the optimal trade-offs between objectives and satisfy equality between disparity errors, we aim at directly solving the multi-objective programming ~\citep{miettinen2012nonlinear}. Towards this end, we note that  the optimization problem in (\ref{eq:const-opt}) is a relaxation of the following generalized multi-objective optimization problem:
\begin{equation}
\label{eq:multi-single}
     \arg\min_{\bm{U}}
    \left[\mathcal{L}(\bm{U}), \psi\left(\mathcal{E}_1\left(\bm{U}\right)\right) , \ldots , \psi\left(\mathcal{E}_k\left(\bm{U}\right)\right) \right]
\end{equation}
where $\psi(\cdot):\mathbb{R}\to\mathbb{R}_{+}$ is any penalization function, such as $\psi(z)=|z|$, $\psi(z)=\frac{1}{2}z^2$, or $\psi(z)=e^{-z}$, however, for convergence analysis we will stick to squared or exponential penalization due to their smoothness.  We will define the optimization problem in more detail in the next section and then will introduce an adaptive gradient descent approach to solve it. 

\subsection{Pareto fair subspace}
To characterize the solutions obtained by directly solving the multi-objective optimization problem in~(\ref{eq:multi-single}), we have to compare the objective vector of different solutions with each other, analogous to the what we do in a scalar or single-objective optimization problem. If we only have a single objective function $f(\bm{U})$, we can say the solution $\bm{U}_1$ is better than $\bm{U}_2$ if $f(\bm{U}_1) < f(\bm{U}_2)$.
Similarly, in multi-objective programming, we define the notion of dominance as follows:
\begin{definition}[Dominance]
Let  $\mathrm{\bm{f}}(\bm{U}) = \left[{f}_1(\bm{U}), \ldots, f_m(\bm{U})\right]^\top$ denote  a vector-valued objective function with $m$ objectives. We say the solution $\bm{U}_1$ \textit{dominates} the solution $\bm{U}_2$ if ${f}_i(\bm{U}_1) \leq {f}_i(\bm{U}_2)$ for all $i \in \left[m\right]$, and ${f}_j(\bm{U}_1) < {f}_j(\bm{U}_2)$ for at least one $j \in \left[m\right]$. We denote this  dominance as:
\begin{equation}\label{eq:dominance}
    \mathrm{\bm{f}}(\bm{U}_1) \prec_{p} \mathrm{\bm{f}}(\bm{U}_2).
\end{equation}\vspace{-0.5cm}
\label{def:dom}
\end{definition}
The definition of dominance implies that when a solution cannot be dominated by any other solution in the search space, we cannot find any direction, to move to, from this solution without at least hurting one objective in the objective vector. Using this, now, we can define our notion of Pareto fair subspace as follows:
\begin{definition}[Pareto Fair Subspace]
Let  $\mathrm{\bm{f}}(\bm{U}) = \left[\mathcal{L}(\bm{U}), {f}_1(\bm{U}), \ldots, f_{m-1}(\bm{U})\right]^\top$ denote  a vector-valued objective function with $m$ objectives in the Fair PCA problem. Then, consider a set of fairness trade-off objectives ${f}_{i}(\bm{U}), i \in \left[m-1\right]$, (e.g. $\psi\left(\mathcal{E}_i\left(\bm{U}\right)\right)$ as in~(\ref{eq:multi-single})) that ought to be minimized in addition to the main objective, $\mathcal{L}(\bm{U})$. The solution $\bm{U}^*$ is called Pareto fair subspace, if it is not dominated by any other feasible solution.
\label{def:pef}
\end{definition}
The Pareto fair subspace is not unique, and the set of Pareto optimal solutions is called Pareto frontier~\citep{miettinen2012nonlinear}. Thereupon, the ultimate goal of a fair PCA reduces to finding a Pareto optimal solution via solving the problem~(\ref{eq:multi-single}).

\subsection{Proof of existence}\label{sec:existence}
The following theorem establishes the conditions under which the set of Pareto optimal solutions exists and is non-empty. We emphasize that compared to methods that optimize Lagrangian function or other scalarization approaches, we aim at finding this Pareto fair frontier completely without any prior information such as weight for each objective.
\begin{theorem}[Existence]\label{thrm:existence}
Consider the vector-valued optimization problem in~(\ref{eq:multi-single}). If the individual functions are convex and bounded, then the set of Pareto optimal solutions is non-empty. 
\end{theorem}
\begin{proof}
The proof is deferred to Appendix~\ref{sec:thrm:existence}.
\end{proof}

To guarantee the existence of a Pareto optimal solution, in Section~\ref{sec:Alg}, we convexify the objectives by properly regularizing them. Thereafter, we propose an efficient gradient-based algorithm to find a subspace that is a Pareto stationary point of the fair PCA problem. 

Although solving the optimization problem in (\ref{eq:multi-single}) results in an efficient trade-off between different objectives, this does not reflect on balanced disparity errors among different groups, which is the ultimate goal of the fair PCA problem. As been asserted by~\cite{samadi2018price}, this issue would be exacerbated in problems with $k > 2$, that having a balanced disparity error among all groups is not always possible due to the fact that all optimal solutions will not assign the same loss to all groups. To alleviate this issue and ensure that the loss of each group remains at most the optimal fairness objective in the original target dimension $r$, we introduce the notion of \textbf{pairwise disparity error}, that would address this issue.
\begin{definition}[Pairwise Disparity Error]
Consider the disparity errors for any projection matrix $\bm{U}$ and sensitive groups of $i$ and $j$ among $k$ different groups, then the pairwise disparity error between these two groups is defined as:
\begin{equation}\label{eq:pde}
    \Delta_{i,j} = \mathcal{E}_i\left(\bm{U}\right) - \mathcal{E}_j\left(\bm{U}\right), \;\; i,j \in \left[k\right], i\neq j.
\end{equation}
\label{def:pde}\vspace{-0.5cm}
\end{definition}
Thus, the optimization in (\ref{eq:multi-single}) becomes:
\begin{equation}
\label{eq:multi-pairwise}
     \arg\min_{\bm{U}}
    \left[\mathcal{L}(\bm{U}), \psi\left(\Delta_{1,2}\left(\bm{U}\right)\right), \ldots, \psi\left(\Delta_{k-1,k}\left(\bm{U}\right)\right)  \right],
\end{equation}
where we have $k\choose 2$ objectives in addition to the main objective. We will show the efficacy of pairwise disparity error over single disparity error in practice in Section~\ref{sec:exp}.

\section{Adaptive Optimization}\label{sec:Alg}
In this section, we will develop a gradient descent (GD) based algorithm to solve the optimization problems in (\ref{eq:multi-single}) or  (\ref{eq:multi-pairwise}). To lay the groundwork for this algorithm, we review how to solve the original PCA problem using gradient descent, and then we propose our proposed algorithm to solve the aforementioned multi-objective problem.

\subsection{Gradient descent for PCA}\label{sec:gd-pca}
To solve the PCA problem using the gradient descent approach, we need to iteratively update the projection matrix $\bm{U}$, based on the gradient of the total reconstruction loss with respect to it. Expanding the total reconstruction loss in (\ref{eq:pca-reconst2})  and removing the constant terms that will not affect the optimization, following~\cite{shalev2014understanding}, we can write the optimization problem:
\begin{equation}\label{eq:pca-trace}
    \underset{\bm{U} \in \mathbb{R}^{d \times r}, \bm{U}^\top\bm{U}=\bm{I} }{\arg \min} -\trace\left( \bm{U}^\top\bm{X}^\top\bm{X}\bm{U}\right),
\end{equation}
Using~(\ref{eq:pca-trace}), we can calculate the gradient of the total reconstruction loss with respect to $\bm{U}$ as follows:
\begin{equation}\label{eq:pca-grad}
    {\mathcal{G}}(\bm{U}) = \frac{\partial \mathcal{L}(\bm{U}) }{\partial \bm{U}} = -2\bm{X}^\top\bm{X}\bm{U}.
\end{equation}
The projection can be learned using the gradient descent by iteratively updating an initial solution by:
\begin{equation}
   \bm{U}_{t+1} =  \Pi_{\mathcal{P}_r}\left(\bm{U}_{t}  - \eta_t \mathcal{G}(\bm{U}_t)\right),
\end{equation}
where $\eta_t$ is the learning rate and $\Pi_{\mathcal{P}_r}\left(.\right)$ is the projection operator onto ${\mathcal{P}_r} = \left\{\bm{U}\in\mathbb{R}^{d\times r} \,\big|\,  \bm{U}^\top\bm{U}=\bm{I}_r \right\}$.

For a single-objective optimization like normal PCA, at each iteration, we take a step toward the negative of the gradient at that point. However, when we are dealing with multiple objectives, the key question is what would be the best direction at each iteration to take, in order to decrease all the objectives. We answer this question in the next section by proposing an optimization problem to find such a descent direction.

\subsection{Pareto fair PCA}\label{sec:MOP}
\begin{algorithm2e}[t]
\DontPrintSemicolon
\caption{Pareto Fair PCA}
\label{alg:fairpca}
\SetNoFillComment
 \SetKwFunction{PFP}{ParetoFairPCA}
 \SetKwProg{Fn}{function}{:}{\KwRet $\bm{U}_{T+1}$}
\textbf{Input:} $r$, $\bm{X} = \bm{X}_1 \cup \bm{X}_1 \cup \ldots \cup \bm{X}_k$, $\bm{U}_0 \in \mathbb{R}^{d \times r}$, $T$\\
Find the optimal $r$-rank subspace for each group, $\mathcal{U}^*=\left\{ \bm{U}_1^*,\ldots,\bm{U}_k^*\right\}$ \tcp*{e.g., using SVD or SGD~\cite{shamir2015stochastic}}
\Fn{\PFP{${r}, {\bm{X}}, {\mathcal{U}^*},  {\bm{U}_0}, {T}$}}{
Form the objective vector $\mathrm{\bm{f}}(\bm{U})$ using $\mathcal{U}^*$ and $\bm{X}$ from~(\ref{eq:multi-single}) or~(\ref{eq:multi-pairwise})\\
\For{$t=1,\ldots,T$}{
    Calculate the gradient of each objective, $\bm{G}_i^{(t)}$\\
    Find the descent direction $\bm{{D}}^{(t)}$ using (\ref{eq:qop})\\
    \If{$\bm{{D}}_t = \mathbf{0}$}{ 
      \textbf{return} $\bm{U}_{t}$ \tcp*{Pareto stationary point}
    }
    Find minimum $p\in\mathbb{N}$ such that for $\eta_t=\frac{1}{2^p}$: \tcp*{Backtracking line search~(\texttt{optional})}
    $$
     f_i\left(\bm{U}_t + \eta_t\bm{{D}}_t\right) \leq f_i(\bm{U}_t) + \beta \eta_t \trace \left( \bm{{D}}_t^\top\bm{{G}}_i^{(t)}\right),\; i \in [m].
     \refstepcounter{equation}
    \label{eq:back}
    \eqno{(\theequation)}\hspace{4cm}
     $$
    $\bm{U}_{t+1} = \Pi_{\mathcal{P}_r}\left(\bm{U}_{t} + \eta_t\bm{{D}}_t\right)$
}
}
\end{algorithm2e}
In order to efficiently solve the multi-objective optimization problem in (\ref{eq:multi-single}) or (\ref{eq:multi-pairwise}), we propose a gradient descent approach, that can guarantee convergence to a Pareto stationary point. For the ease of exposition,  we consider the following general multi-objective problem with $m$ objectives:
$$\mathrm{\bm{f}}(\bm{U}) = \left[ f_1(\bm{U}), \ldots, f_m(\bm{U})\right]$$
In a single-objective problem with gradient descent method, we always choose the opposite direction of the gradient on that point as the descent direction to decrease the objective function for the next iteration point. However, this notion in multi-objective programming is more complicated, as we have to find the direction that is a descent direction for all objectives based on their gradients on that point. In order to find a descent direction,  let   $\left[ \bm{G}_1^{(t)}, \ldots, \bm{G}_m^{(t)}\right]$ denote the gradient of individual objectives at point  $\bm{U}_t$. To  find a descent direction with respect to all of the objectives at point $\bm{U}_t$, we solve  the following minmax optimization problem~\citep{fliege2000steepest}:
\begin{equation}\label{eq:qop}
    \bm{D}_t = \arg\underset{\bm{D}\in \mathbb{R}^{d\times r}}{\min} \left\{ \underset{i=1,\ldots,m}{\max} \trace\left(\bm{D}^\top\bm{G}_i^{(t)}\right) + \frac{1}{2} \big\|\bm{D} \big\|_{\text{F}}^2\right\}.
\end{equation}
We note that for a single objective case, that is $m=1$, the solution of above minimax is the opposite of the gradient, i.e. $\bm{D}_t = -\bm{G}_1^{(t)}$. Using the KKT optimally conditions, it is easy to show that the dual problem becomes a quadratic programming and can be efficiently solved to identify a descent direction $\bm{D}_t$, for which all the objectives are non-increasing. The following lemma states this characteristic of the descent direction:
\begin{lemma}[Descent Direction]
The solution found in the optimization problem (\ref{eq:qop}) has one of the following two conditions. Either $\bm{D}_t = \bm{0}$, which means the point $\bm{U}_t$ is a Pareto stationary point, or $\bm{D}_t$ is a descent direction to all objectives, that is:
\begin{equation}\label{eq:descent}
    \trace\left(\bm{D}_t^\top \bm{G}_i^{(t)} \right) \leq 0, \;\; \forall \; 1 \leq i \leq m
\end{equation}
Then, the obtained descent direction is in the form of $\bm{D}_t = -\sum_{i=1}^m \lambda_i^{(t)} \bm{G}_i^{(t)} $, where $\sum_{i=1}^m \lambda_i^{(t)} = 1$ and  $\lambda_i^{(t)} \geq 0$ for $1 \leq i \leq m$.
\label{lemma:descent}
\end{lemma}
\begin{proof}
The proof is provided in Appendix~\ref{app:lemma}.
\end{proof}
As elaborated in the proof in Appendix~\ref{app:lemma}, the theorem implies that the descent direction is the minimum norm matrix in the convex hull of the gradients of all objectives and is the non-increasing direction with respect to each objective. Understanding this, the following corollary is palpable:
\begin{corollary}
The first order Pareto stationary point holds for a solution $\bm{U}$ when the mentioned minimum norm is zero, i.e., there is no descent direction that is non-increasing for all objectives. In other words,  there exists a  ${\boldsymbol{\lambda}}\in\Delta_m$ such that $\bm{D} = -\sum_{i=1}^m {\lambda}_i \bm{G}_i = \bm{0}$ where $\bm{G}_i = \nabla f_i(\bm{U})$. 
\end{corollary}
Having a descent direction at hand, we can use it to decrease all the objectives in every iteration, similar to the procedure defined in Algorithm~\ref{alg:fairpca}. Based on the first-order optimality condition of this problem, we know that at a Pareto optimal solution, the direction found in~(\ref{eq:qop}) should be $\bm{0}$, meaning, that it cannot further improve any objective without hurting others. Equipped with this descent direction and first-order optimality condition, we can iteratively update the initial solution in the direction of the descent direction, until it converges to a Pareto stationary point. 
\begin{remark}\label{rem:norm}
One crucial step before finding the descent direction is to balance out the scale of different gradients. Since they are calculated based on very different and possibly contradictory objective functions, their Frobenius norm would vary a lot; hence, by a normalization step, we can avoid the dominance of the descent direction by some gradients with high Frobenius norm. 
\end{remark}
 Since the disparity errors, as well as the main PCA objective, are weakly convex functions, following Theorem~\ref{thrm:existence}, to guarantee the existence of Pareto optimal subspace, we add a regularization term to each objective to make them convex functions-- with which we also stabilize the solutions and guarantee convergence. As a result,  the optimization in (\ref{eq:multi-pairwise}) becomes:
\begin{equation}\label{eq:multi-multi-opt}
    \arg\min_{\bm{U}}
    \begin{bmatrix}
    \mathcal{L}(\bm{U}) +  \alpha \|\bm{U}\|_{\text{F}}^2\\  \psi\left(\Delta_{1,2}\left(\bm{U}\right)\right)  + \alpha \|\bm{U}\|_{\text{F}}^2 \\ 
    \vdots\\
    \psi\left(\Delta_{k-1,k}\left(\bm{U}\right)\right) + \alpha \|\bm{U}\|_{\text{F}}^2
    \end{bmatrix},
\end{equation}
where $\alpha$ is the regularization parameter to make the Hessian matrices of objectives positive semi-definite and needs to  be decided based on the maximum eigen-gap between covariance matrices of each pair of sensitive groups. Having $k$ different groups, each with data matrix of $\bm{X}_i$, $i \in [k]$, we set $\gamma=\underset{i,j \in [k]}{\max} \gamma_d\left(\bm{X}_i^\top\bm{X}_i \right) - \gamma_1\left(\bm{X}_j^\top\bm{X}_j\right)$ to denote the maximum eigen-gap. Then, we should have $\alpha \geq \gamma$.
We now turn to prove the convergence rate of Algorithm~\ref{alg:fairpca} for convex objectives, as stated in the following theorem.
\begin{theorem}[Convex Convergence]  Let $\mathrm{\bm{f}} = \left[f_1(\bm{U}),\ldots,f_m(\bm{U}) \right]$ be convex component-wise Lipchitz continuous with constants $L_1, L_2, \ldots, L_m$. Then, for the sequence of the solutions $\bm{U}_1,\ldots,\bm{U}_T$ generated iteratively by Algorithm~\ref{alg:fairpca}, and the sequence of $\hat{\boldsymbol{\lambda}}^{(1)},\ldots,\hat{\boldsymbol{\lambda}}^{(T)}$ generated by~(\ref{eq:qop2}) during $T$ iterations, by setting $\eta=\frac{R}{L\sqrt{T}}$ and $\beta = \sqrt{T}/R$, we have:
\begin{equation}\label{eq:convex-bound}
    \sum_{i=1}^{m} \bar{\lambda}_i \left(f_i(\bm{U}_T) - f_i(\bm{U}^*)\right) \leq \frac{R L}{2\sqrt{T}}, 
\end{equation}
where $R^2 = \lVert \bm{U}_1 - \bm{U}^* \rVert^2_{\mathrm{F}}$,  $L = \max_{i=1, \ldots, m} L_i$,  $\bar{\lambda}_i = \frac{1}{T}\sum_{t=1}^T \hat{\lambda}_i^{(t)}$, and $\bm{U}^*$ is a Pareto efficient solution.
\label{theorm:convex}
\end{theorem}
\begin{proof}
The detailed proof is deferred to Appendix~\ref{app:convex}.
\end{proof}
Theorem~\ref{theorm:convex} indicates that, using the Pareto descent direction, we can achieve an $\epsilon$-accurate Pareto efficient solution with taking $\order\left(\frac{1}{\epsilon^2}\right)$ gradient descent steps. Using~(\ref{eq:convex-bound}), we can bound the average deviation of each objective from its respective value in the Pareto efficient solution of $\bm{U}^*$.

\begin{remark}
We note that Algorithm~\ref{alg:fairpca} is guaranteed to converges to a single Pareto fair subspace, starting from a fixed initial solution $\bm{U}_0$. Using different random starting points, we can find different Pareto fair subspaces and form the Pareto fair frontier of the problem. From an algorithmic point of view, we can not distinguish between different Pareto optimal subspaces, but as discussed by~\cite{kearns2019ethical}, based on the preference of different objectives, we can choose a desirable Pareto fair subspace from the frontier set.
\end{remark}
We note that when the regularization is not added to convexify the main objective,
we  have to deal with non-convex objectives in the optimization problem. In the following theorem, we investigate the convergence of Algorithm~\ref{alg:fairpca} for non-convex objectives that guarantees  the gradient vanishes over iterations.
\begin{theorem}[Nonconvex Convergence] Let $\mathrm{\bm{f}}(\bm{U}) = \left[f_1(\bm{U}),\ldots,f_m(\bm{U}) \right]$ be the multi-objective function to be minimized to find  a fair subspace with respect to $k$ sensitive groups.  Let $\bm{U}_1, \bm{U}_2, \ldots , \bm{U}_T$ be the sequence of solutions generated by Algorithm~\ref{alg:fairpca} updated using descent directions $\bm{D}_1, \bm{D}_2, \ldots , \bm{D}_T$. Then, if we choose the regularization parameter as $\alpha\geq\gamma$, we have the following:
\begin{equation}
    \min_{t=1,2, \ldots, T} \|\bm{D}_t  \|_{\mathrm{F}} \leq  \sqrt{\frac{\mathsf{M}_u - \mathsf{M}_l}{ C T}}, \end{equation}
where $\mathsf{M}_l$ is a lower bound for the values of all objective functions, $\mathsf{M}_u$ is the maximum of the values of all functions at initial point, and $C$ is a constant depending on the smoothness of objectives.
\label{thm:converg}
\end{theorem}
\begin{proof}
The proof can be found in Appendix~\ref{app:nonconvex}.
\end{proof}
An immediate consequence of the above theorem is that the gradient of Pareto descent directions vanishes and convergences to zero and thereby the solutions generated by the algorithm convergence to a  stationary fair subspace. In particular, only $\order\left(\frac{1}{\epsilon^2}\right)$ iterations are required to obtain an $\epsilon$-close fair subspace.  The analysis of Theorem~\ref{thm:converg} follows the standard analysis of gradient descent for non-convex smooth optimization where the obtained bound matches the known achievable convergence rate for the norm of the gradients.  We want to sketch another alternative method that results in the same rate with careful analysis. Specifically, observe that the descent direction can be considered as an inexact gradient from the viewpoint of individual functions with perturbation $\mathrm{\mathbf{D}}_t - \mathrm{\mathbf{G}}_i^{(t)}$. Noting that $\trace\left(\mathrm{\mathbf{D}}_t^\top \mathrm{\mathbf{G}}_j^{(t)}\right) \leq -\left\| \mathrm{\mathbf{D}}_t \right\|_{\text{F}}^2$ as shown in the proof of Lemma~\ref{lemma:descent} and following the standard analysis of convergence of non-convex functions,  we can show that norm of descent directions vanishes as algorithm proceeds, thereby the proposed algorithm can find a stationary point. However, the obtained solution is not guaranteed to be an optimal Pareto due to the non-convexity of the objectives and might be a saddle point.

\subsection{Comparison with other approaches}\label{sec:comp}
As it was discussed, one approach to solve a multi-objective optimization is to make it constrained optimization, in which we keep the main objective and change all other objectives to inequality constraints with parameters $\boldsymbol{\epsilon}$. Hence, constrained optimization is a relaxation of multi-objective optimization, where finding the best constraint parameter ($\boldsymbol{\epsilon}$) for each constraint could be very challenging as discussed in~\cite{donini2018empirical}. It also lacks theoretical guarantees due to the non-convex nature of constraints. Lagrangian method of multipliers is equivalent to constraint optimization problems, but not exactly to multi-objective counterpart. To see this, we note that by applying GD to Lagrangian function, the contribution of the gradient of each individual function, $\bm{G}_i^{(t)}$, is weighted by its Lagrangian multiplier, while in our case the weights are adaptively learned by finding a Pareto decent direction. We note that while~\cite{morgenstern2019fair} improves the requirement of extra dimensions over~\cite{samadi2018price}, it still needs $\lfloor \sqrt{2k+\frac{1}{4}} - \frac{3}{2}\rfloor$ extra dimensions for a $k$-group sensitive feature and has to solve an SDP which has the time complexity of $\order\left(d^{6.5}\log\left(\frac{1}{\epsilon}\right)\right)$ or $\order\left(d^3/\epsilon^2\right)$ with multiplicative weight update. On the other hand, our method enjoys the efficiency of GD with an overhead to solve the quadratic problem over the simplex for finding the descent direction. Also, at each step, we need to project the solution to find the orthonormal bases for the updated solution, which could be done using SVD with an overhead of $\order\left(r^2d\right)$ or more efficiently using variance-reduced SGD~\citep{shamir2015stochastic}. Using vanilla SVD for per-iteration projection brings the overall time complexity of the proposed algorithm to $\order\left(r^2d/\epsilon^2\right)$. We note that the convex formulation in~\cite{olfat2018convex} also requires solving an SDP programming  (e.g.,  ellipsoid method to interior point method), which suffers from high computational cost as well.
 
This is for the first time that we are solving the exact multi-objective problem, rather than its min-max relaxation using SDP in a fairness problem. \cite{morgenstern2019fair} is suggesting that for $k=2$ their approximation is exact, meaning their algorithm will find the fair representation in the exact $r$ dimension they aim to reach. However, in practice, even for $k=2$, we can show that our algorithm can achieve a smaller disparity error, as shown in Figure~\ref{fig:compare}, which indicates that pairwise disparity error can achieve a better subspace in terms of fairness. For $k>2$, they are still solving an inefficient SDP problem to exact same problem we are proposing. Hence, the novelty of our approach lies in solving this problem using gradient descent and ensuring to reach a Pareto stationary point, which even does not require extra dimensions to satisfy fairness. This setting and its proposed gradient descent algorithm to solve it can be applied to other unsupervised and supervised fairness problems. Thus, it could open up new perspectives on all other fairness problems in learning tasks, by advocating optimal trade-offs between main learning objectives and fairness criteria using Pareto efficiency.

\section{Experiment}\label{sec:exp}
In this section, we empirically examine the introduced algorithm for fair PCA with the Adult dataset\footnote{\url{https://archive.ics.uci.edu/ml/datasets/Adult}} and the Credit dataset\footnote{\url{https://archive.ics.uci.edu/ml/datasets/default+of+credit+card+clients}}. The Adult dataset consists of census data to predict whether the income of a person exceeds $50$K per year or not. The Credit dataset contains clients' credit history information to predict whether they would default in the future or not. For PCA, we will omit the label data and work with the rest of it, which contains $14$ features for the Adult dataset, including gender and race, which we consider as sensitive features in this dataset. In the Credit dataset, we will have $23$ features, including sensitive features of sex and marriage. The gender feature in the Adult dataset and sex in the Credit dataset are binary features with two values, namely, Male and Female. Race from the Adult dataset, on the other side, is a \textit{multiple group feature}, with $5$ different groups, including White, Asian-Pac-Islander, Amer-Indian-Eskimo, Black, and Other. Marriage in the Credit dataset is also a multiple group feature with $3$ groups of Single, Married, and Other. 

For the Adult dataset, we use the training dataset, which has $32,561$ number of samples, among which $10,548$ belongs to the Female group and $22,013$ to the Male group. The distribution of samples among race groups are as follows: Black $30,47$, White $27,994$, Asian-Pac-Islander $312$, Amer-Indian-Eskimo $962$, and Other $246$. In the Credit dataset, we have $30,000$ training samples, out of which there are $18,112$ Female and $11,888$ Male samples. The distribution of the Marriage feature is $13,659$ married, $15,964$ single, and $323$ other samples. We first apply the fair PCA method to binary sensitive feature, in which we set the learning rate to $1/\sqrt{t}$, where $t$ is the iteration number. This condition on the learning rate satisfies the maximum decrease condition by backtracking line search in (\ref{eq:back}).

\subsection{Binary sensitive feature}
In the Adult dataset, we observed that the Female group is benefiting from normal PCA on the whole dataset, while the Male group is sacrificing its reconstruction error. Hence, by applying the Fair PCA algorithm, we can perfectly decrease these trade-offs, while suffering an insignificant loss to the total reconstruction error, compared to normal PCA. The results are depicted in Figure~\ref{fig:binary_re}, where the trade-offs and how Fair PCA is addressing them is noticeable.
\begin{figure*}[t!]
    \centering
    \begin{subfigure}[b]{0.31\textwidth}
		\centering
		\includegraphics[width=\textwidth]{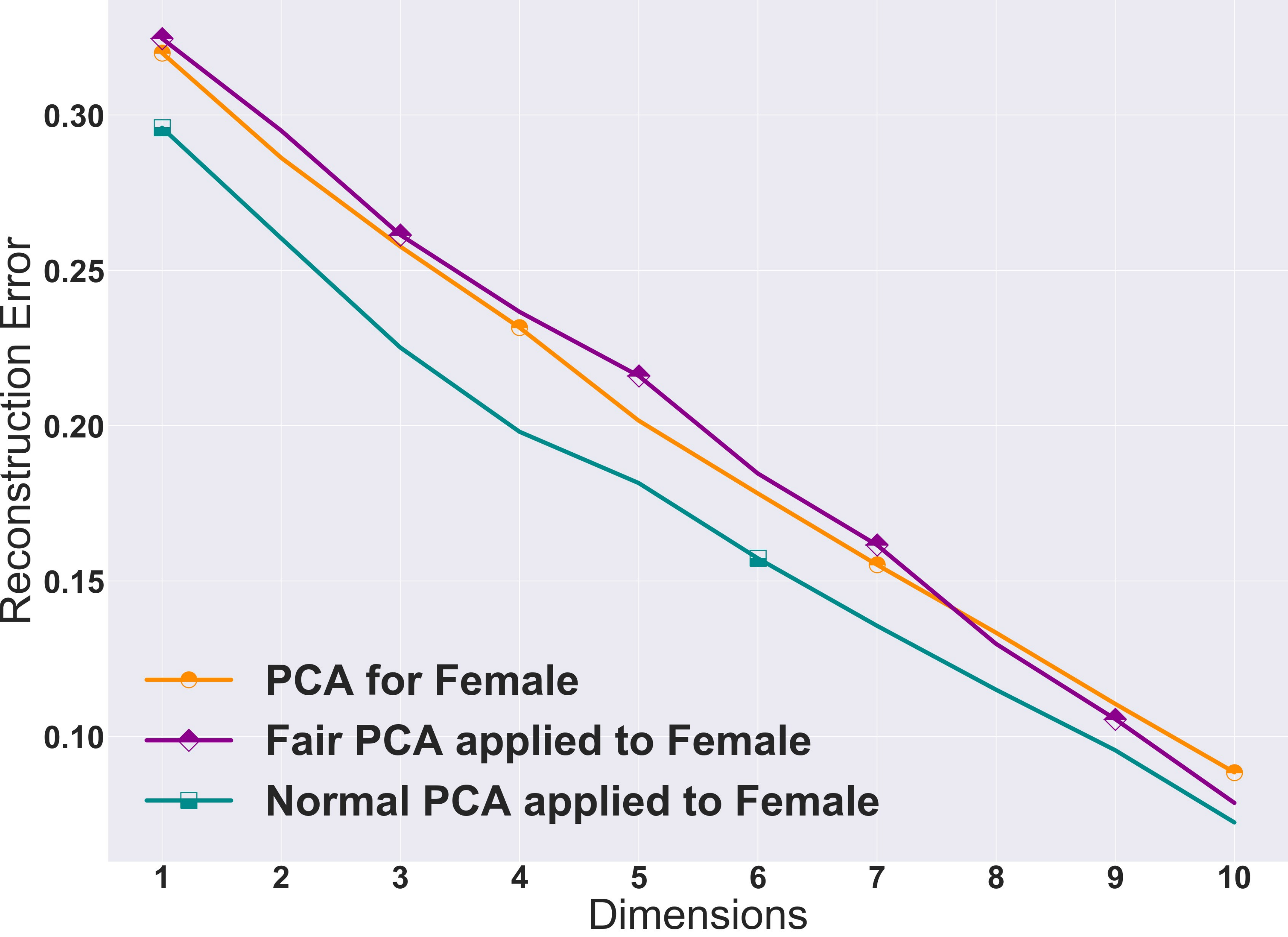}
		\label{fig:demale}
	\end{subfigure}
	\hfill
	\begin{subfigure}[b]{0.31\textwidth}
		\centering
		\includegraphics[width=\textwidth]{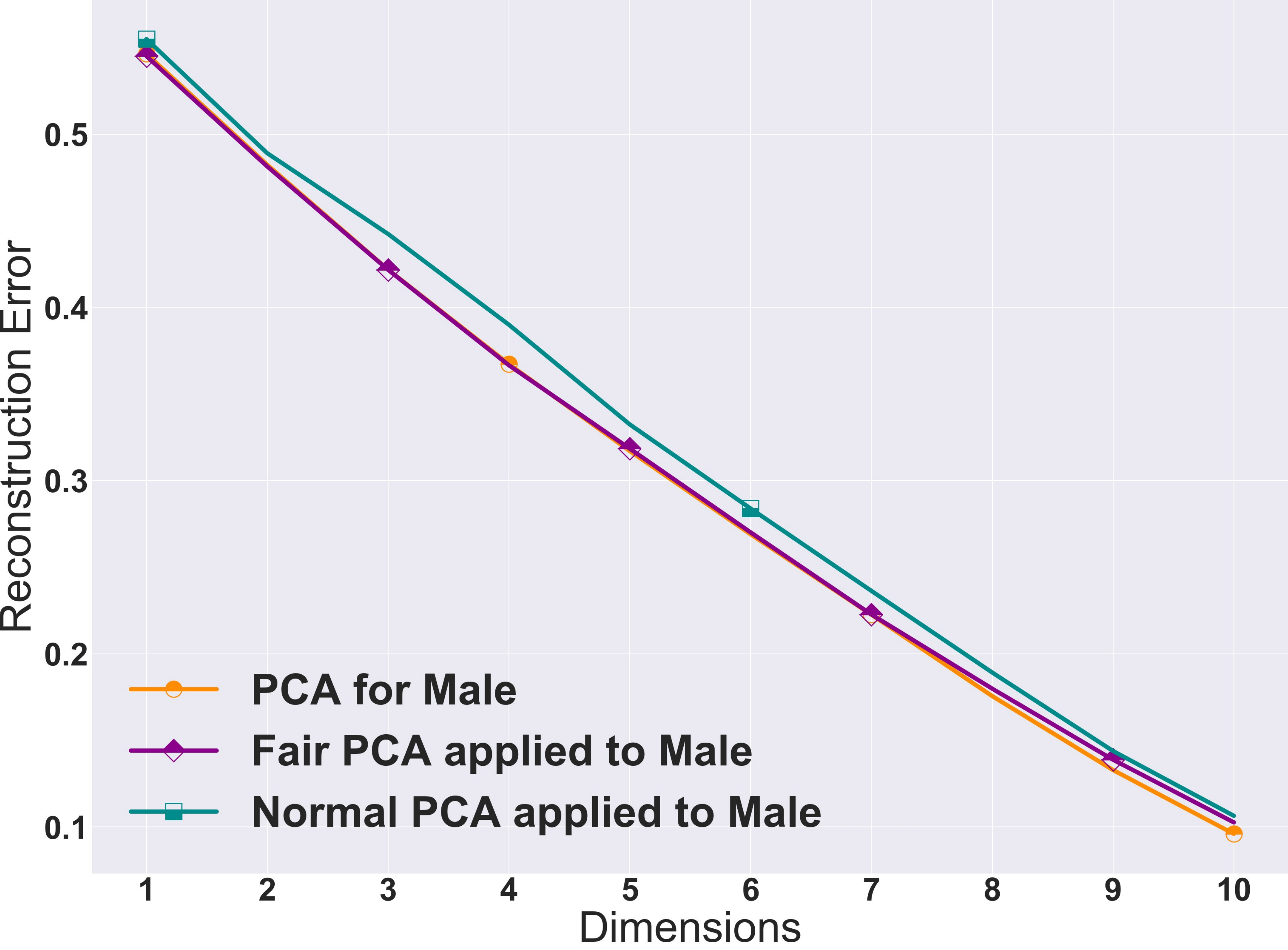}
		\label{fig:male}
	\end{subfigure}
	\hfill
	\begin{subfigure}[b]{0.31\textwidth}
		\centering
		\includegraphics[width=\textwidth]{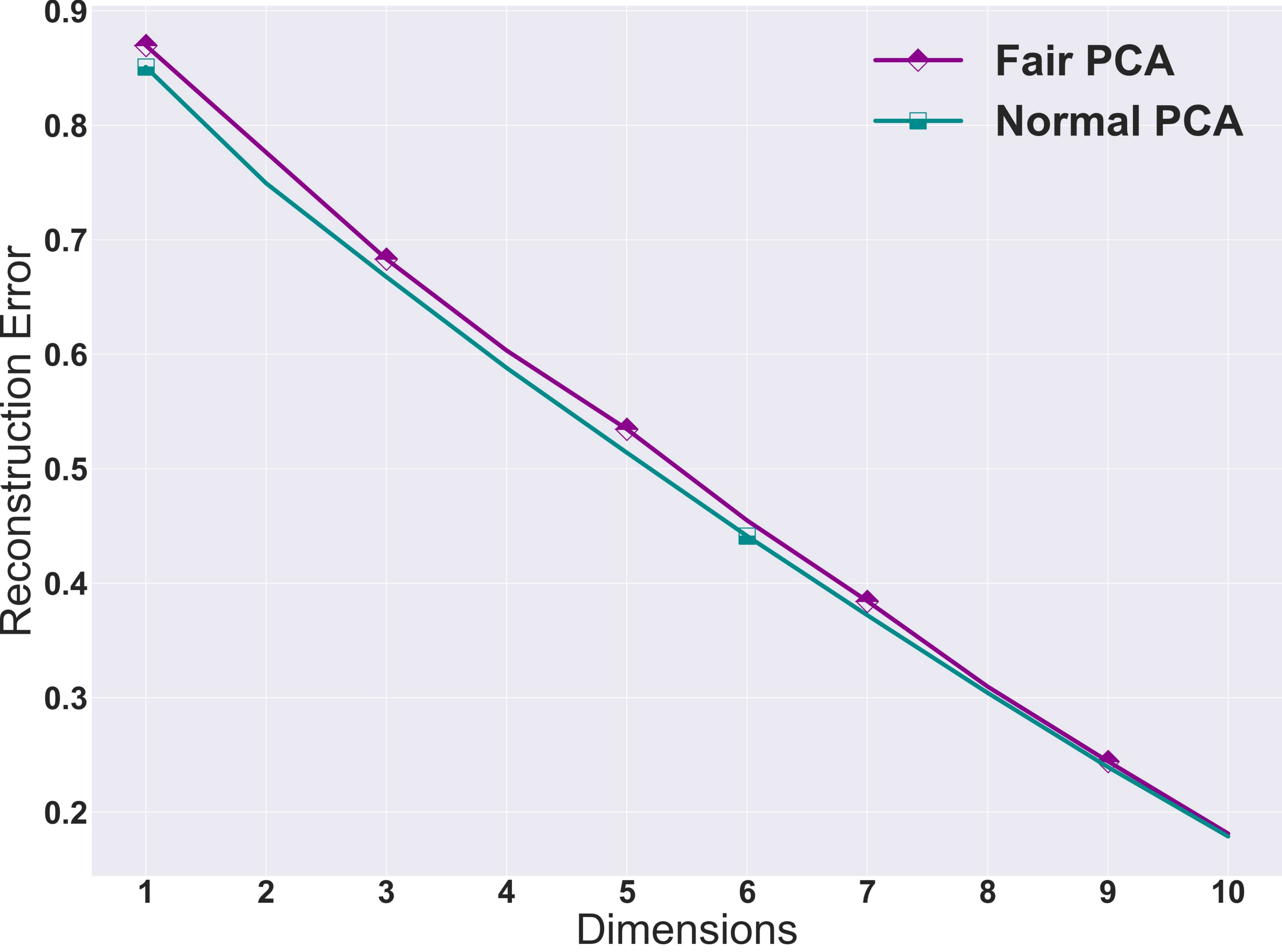}
		\label{fig:binarytotal}
	\end{subfigure}
	\hfill
    \caption{Applying normal PCA and fair PCA to the Adult dataset with gender as its sensitive feature. The first two figures show the reconstruction error of normal PCA (trained on all data) applied to each group, fair PCA (trained on all data) applied to each group, and normal PCA trained on the data of each group individually. The last figure reveals the difference between normal and fair PCA reconstruction loss on all data, which is very tiny and negligible.}
    \label{fig:binary_re}
\end{figure*}
To compare the introduced Pareto fair PCA with algorithms using SDP in~\cite{samadi2018price} and ~\cite{morgenstern2019fair}, we will use the average disparity errors across sensitive groups in both Adult and Credit datasets.
Figure~\ref{fig:compare} shows the average disparity errors of Pareto fair PCA with single and pairwise disparity error objectives, SDP fair PCA, and normal PCA on binary features (gender and sex) of Adult and Credit datasets. First, it reveals that there is a huge gap between normal PCA and fair PCA algorithms in terms of disparity errors, which is indicating that these algorithms are decreasing this disparity error. Second, it shows the superiority of Pareto fair PCA over SDP relaxation methods in both datasets (especially with pairwise objectives), where Pareto fair PCA has a smaller average disparity error close to zero. Also, to show that what is the exact price of fairness that each algorithm pays, we show the total reconstruction loss of our algorithm and other fair PCA methods with respect to normal PCA in Figure~\ref{fig:compare_TRL}. It can be noted that the Pareto fair PCA incurs a slight degradation in total reconstruction loss in exchange for fairness, while this price is much higher in the other state-of-the-art algorithms.
\begin{figure}[t!]
    \centering
	\begin{subfigure}[b]{0.45\textwidth}
		\centering
		\includegraphics[width=\textwidth]{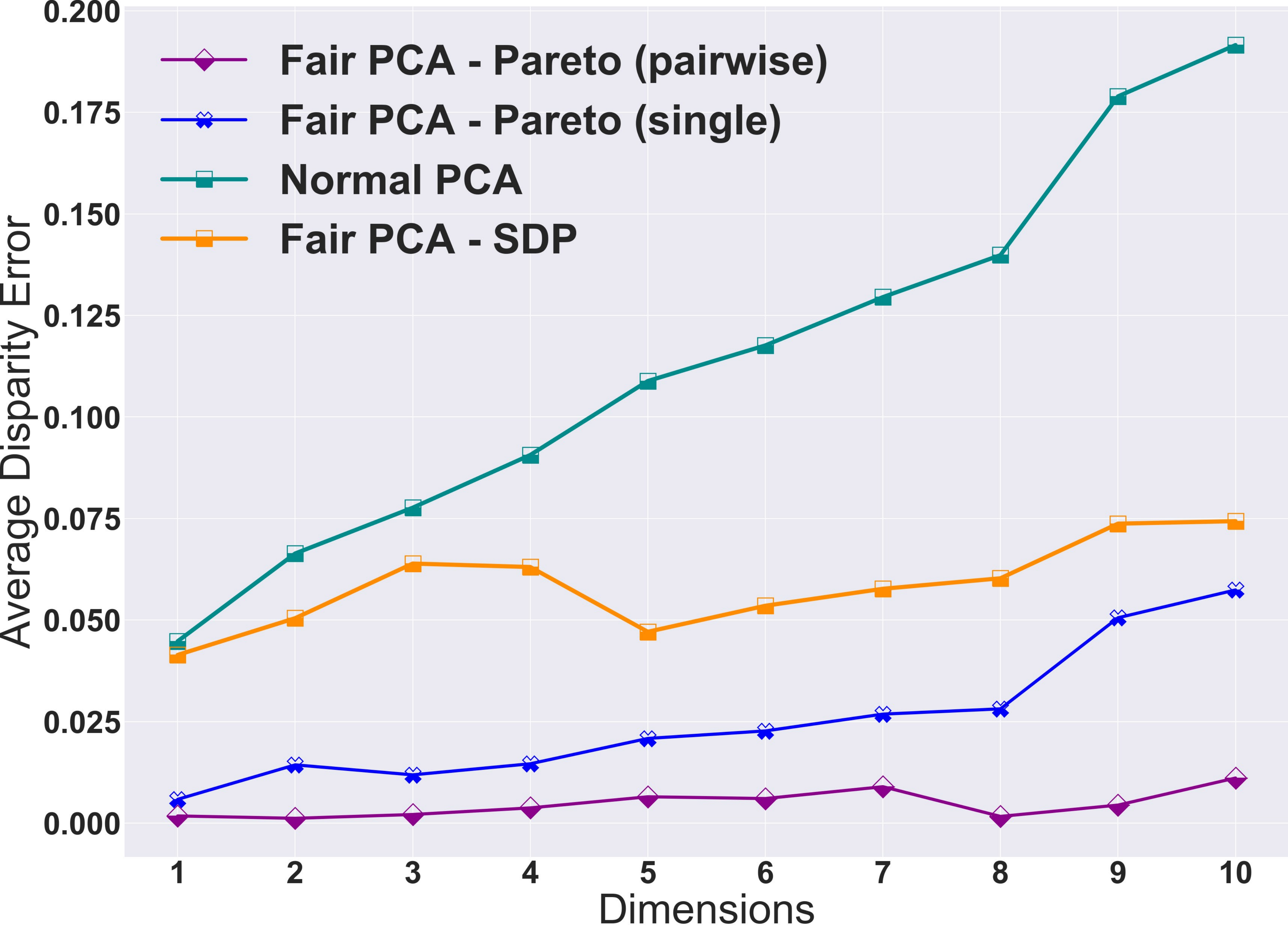}
		\caption{Adult dataset}
		\label{fig:compare_adult}
	\end{subfigure}
    \hfill
	\begin{subfigure}[b]{0.45\textwidth}
		\centering
		\includegraphics[width=\textwidth]{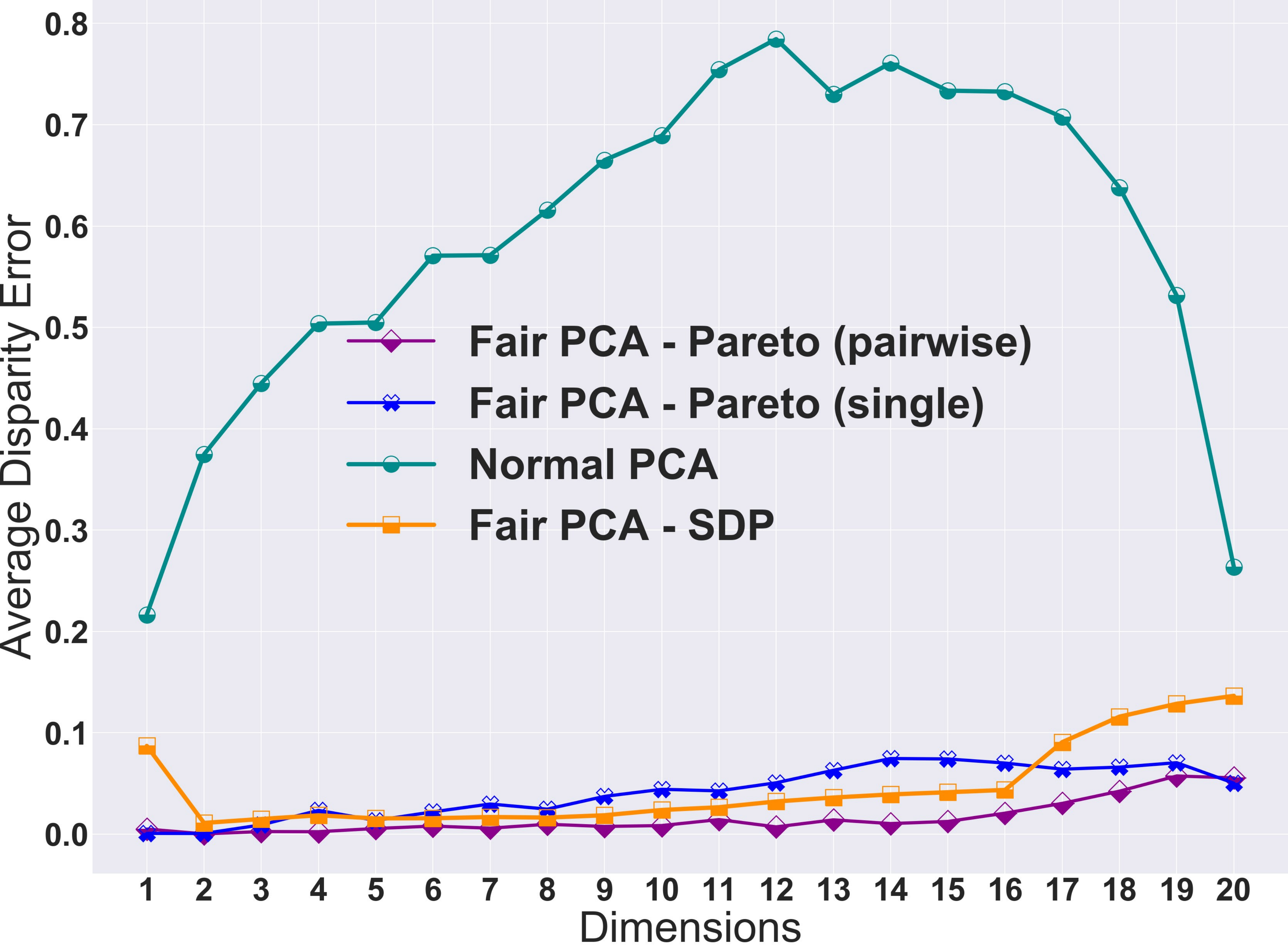}
		\caption{Credit dataset}
		\label{fig:compare_credit}
	\end{subfigure}
    \caption{Comparing Pareto fair PCA algorithm introduced in this paper (pairwise and single disparity error) with fair PCA algorithms using SDP relaxation introduced in~\cite{samadi2018price} and \cite{morgenstern2019fair}. The experiment is on binary features of Adult and  Credit datasets (gender and sex). The average disparity error of algorithms on the Adult dataset clearly shows the superiority of the Pareto fair PCA with pairwise disparity error objectives and then the single disparity error objectives. In the Credit dataset, Pareto fair PCA with pairwise objectives has a slightly better performance with respect to two other methods.}
    \label{fig:compare}
\end{figure}

\begin{figure}[t!]
    \centering
	\begin{subfigure}[b]{0.45\textwidth}
		\centering
		\includegraphics[width=\textwidth]{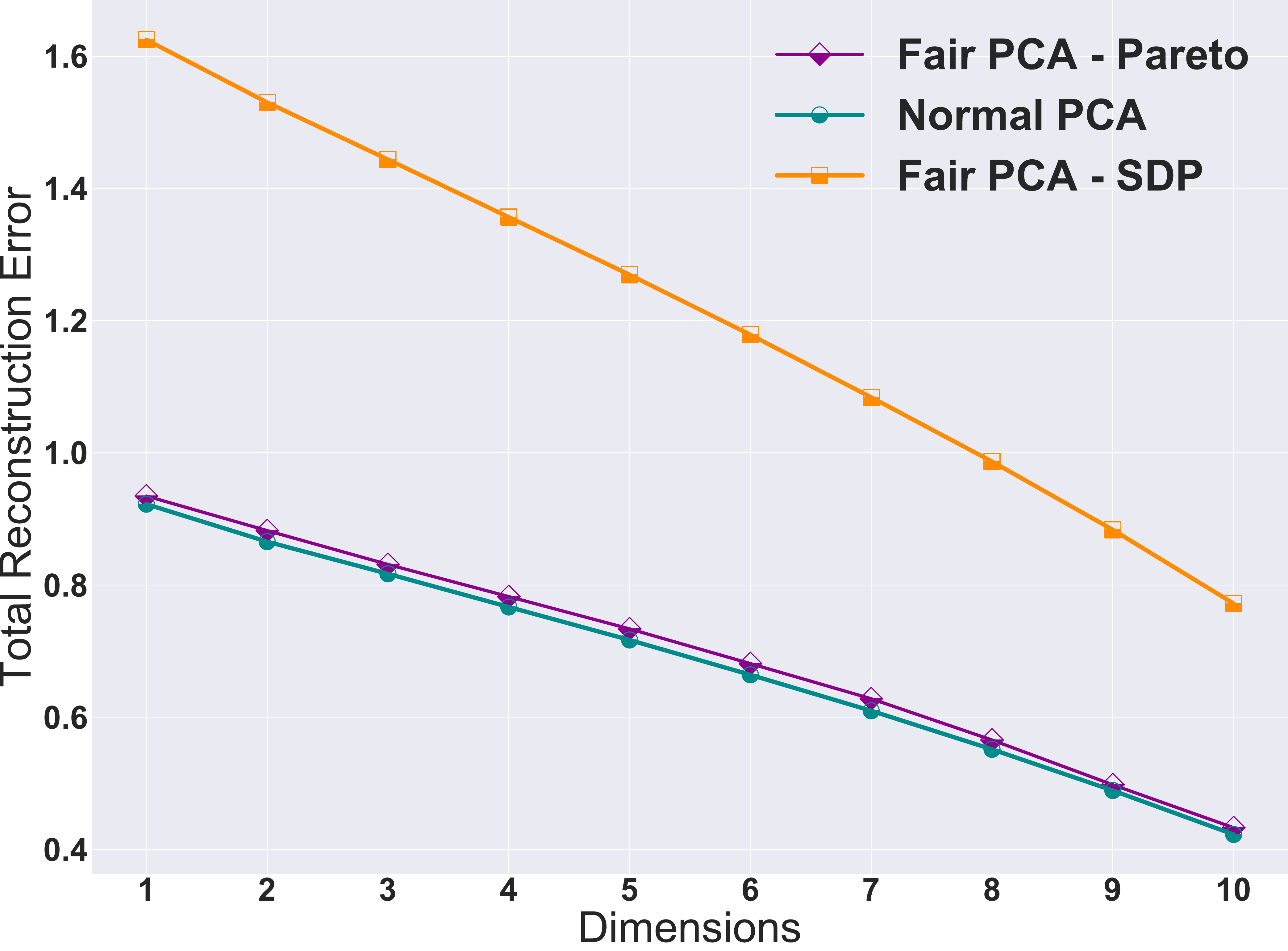}
		\caption{Adult dataset}
		\label{fig:compare_trl_adult}
	\end{subfigure}
    \hfill
	\begin{subfigure}[b]{0.45\textwidth}
		\centering
		\includegraphics[width=\textwidth]{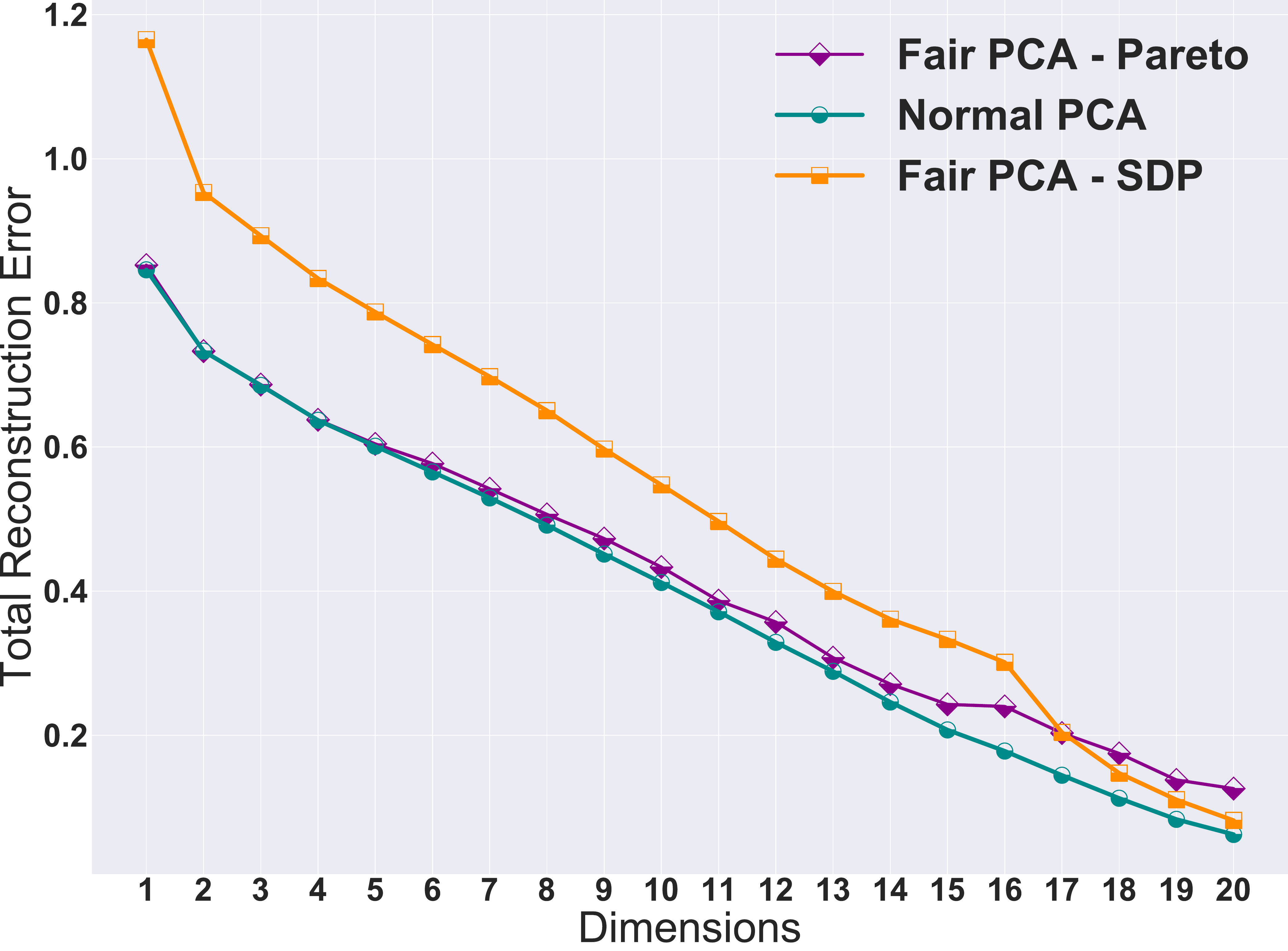}
		\caption{Credit dataset}
		\label{fig:compare_trl_credit}
	\end{subfigure}
    \caption{Comparing the total reconstruction loss of Pareto fair PCA algorithm introduced in this paper (with pairwise objectives), with fair PCA algorithms using SDP relaxation introduced in~\cite{samadi2018price} and \cite{morgenstern2019fair}. The experiment is on binary features of Adult and  Credit datasets (gender and sex). From both figures, it can be inferred that Pareto fair PCA satisfies the fairness objective better, while it incurs a tiny increase in the total reconstruction loss with respect to the normal PCA. On the other hand, Fair PCA with SDP suffers from a huge increase in the total reconstruction loss to satisfy the fairness objective.}
    \label{fig:compare_TRL}
\end{figure}

\subsection{Multiple group sensitive feature}
The proposed Algorithm~\ref{alg:fairpca}, can efficiently generalize to the multiple group sensitive features, by adding pairwise disparity errors of each pair of groups to the objective vector and minimize the overall vector to reach a Pareto optimal or stationary point. However, adding more objectives, introduces more trade-offs, makes the optimization over all objectives more difficult. 

First, we start with the Adult dataset with race as the sensitive feature. In this dataset, race has $5$ categories, makes it a multiple group sensitive feature. The reconstruction error of the Pareto fair PCA and normal PCA is shown in Figure~\ref{fig:multi_re}, where the trade-offs between benefits and sacrifices of different groups are clearly noticeable. The fair PCA  algorithm can superbly decrease these trade-offs for all but one group, with a negligible increase in overall reconstruction loss. Following the same step as in the binary case, we show the disparity error of different groups in Figure~\ref{fig:de-multi_re}, which reveals that fair PCA clearly outperforms normal PCA in most of the groups. Figure~\ref{fig:multi_both} depicts the reconstruction loss and average disparity error of fair and normal PCA. The results indicate that even in a dataset with a multiple group sensitive feature, the increase in the reconstruction loss of fair PCA compared to the normal PCA is slim, while the gap between their disparity errors is huge. This means that normal PCA is unfairly treating different groups in its learned representation subspace.

\begin{figure}[th!]
    \centering
    \begin{subfigure}[b]{0.32\textwidth}
		\centering
		\includegraphics[width=\textwidth]{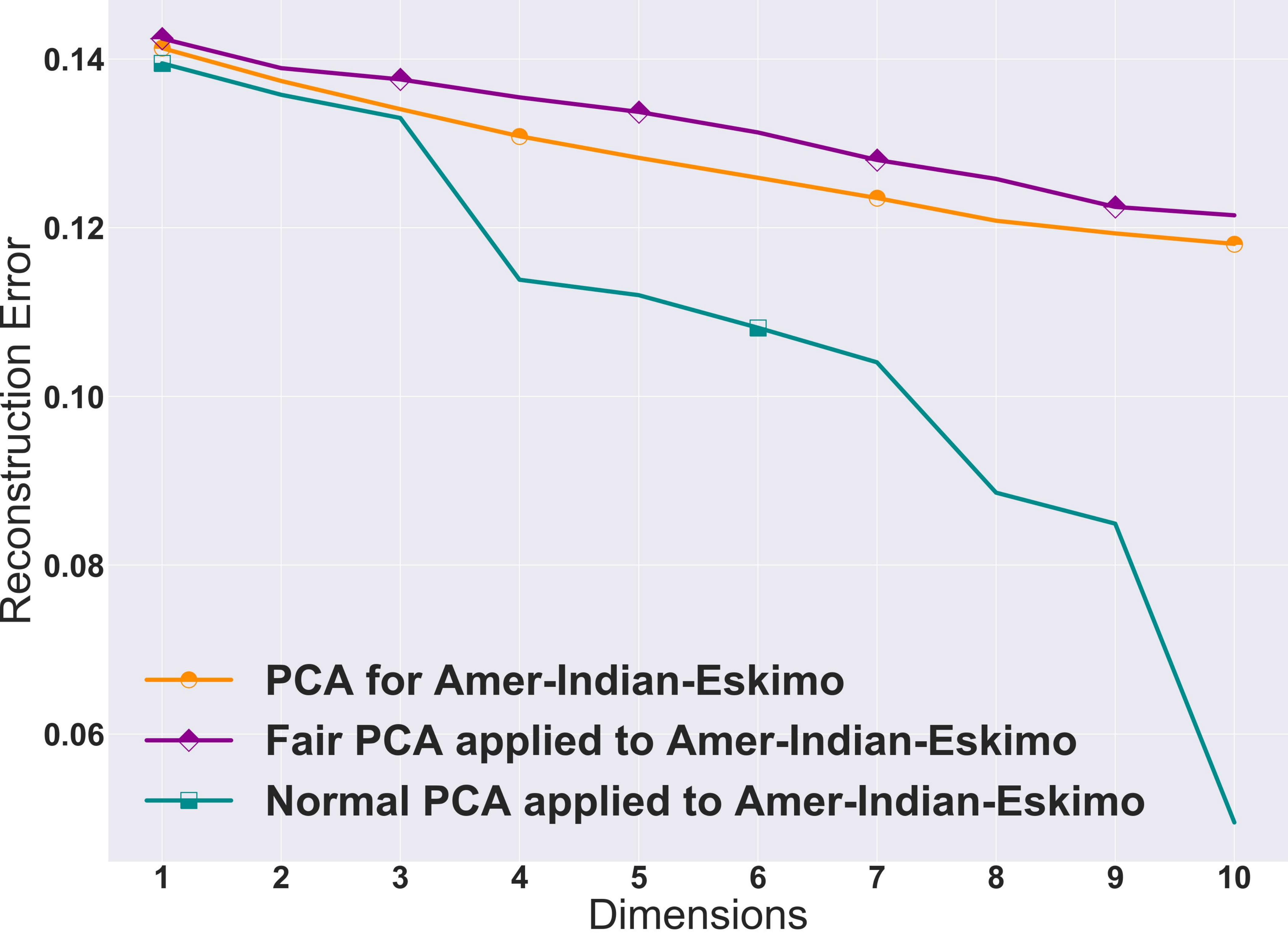}  
		\label{fig:multi-AIE}
	\end{subfigure}
	\hfill
	\begin{subfigure}[b]{0.32\textwidth}
		\centering
		\includegraphics[width=\textwidth]{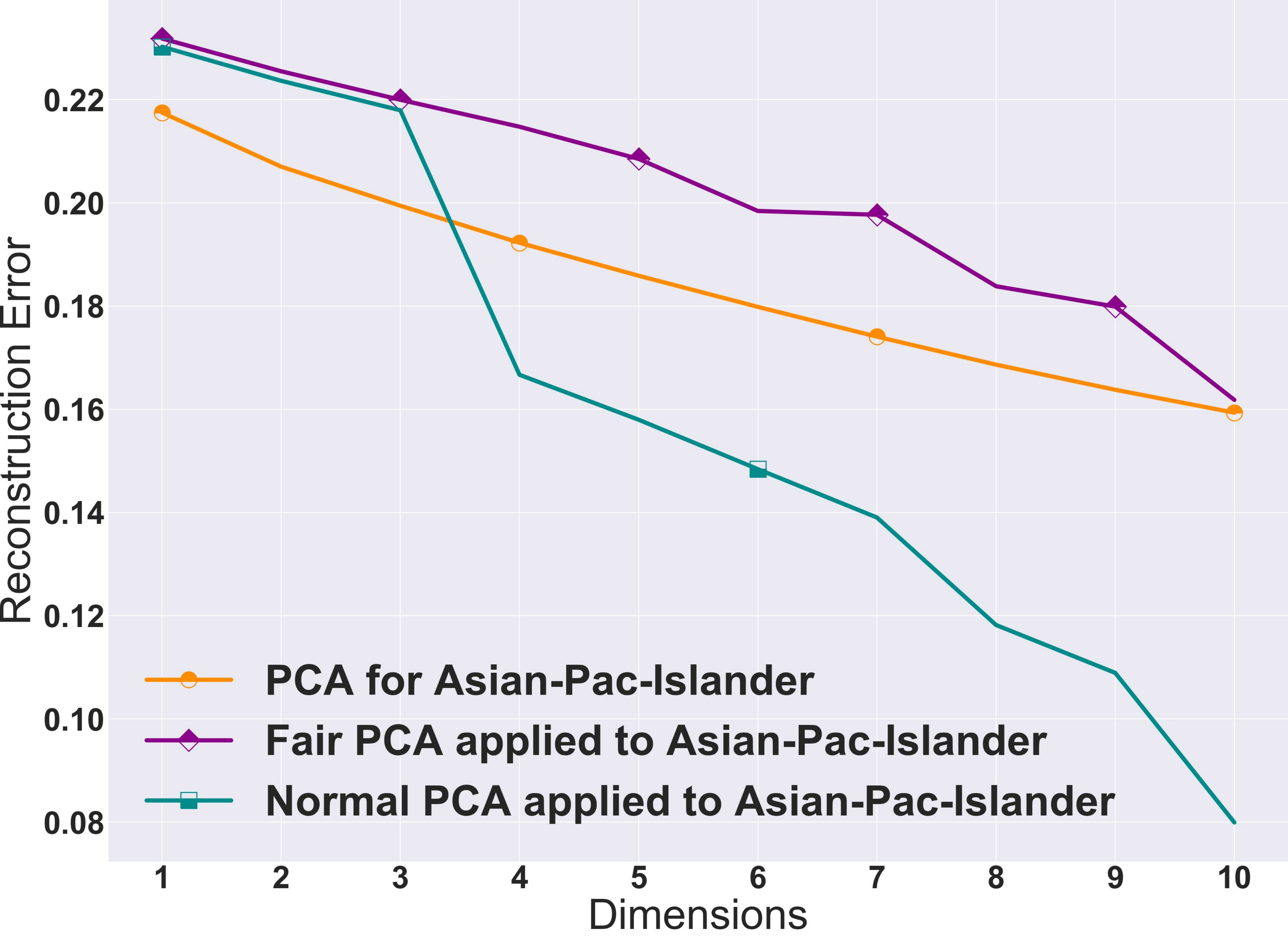}
		\label{fig:multi-API}
	\end{subfigure}
	\hfill
	\begin{subfigure}[b]{0.32\textwidth}
		\centering
		\includegraphics[width=\textwidth]{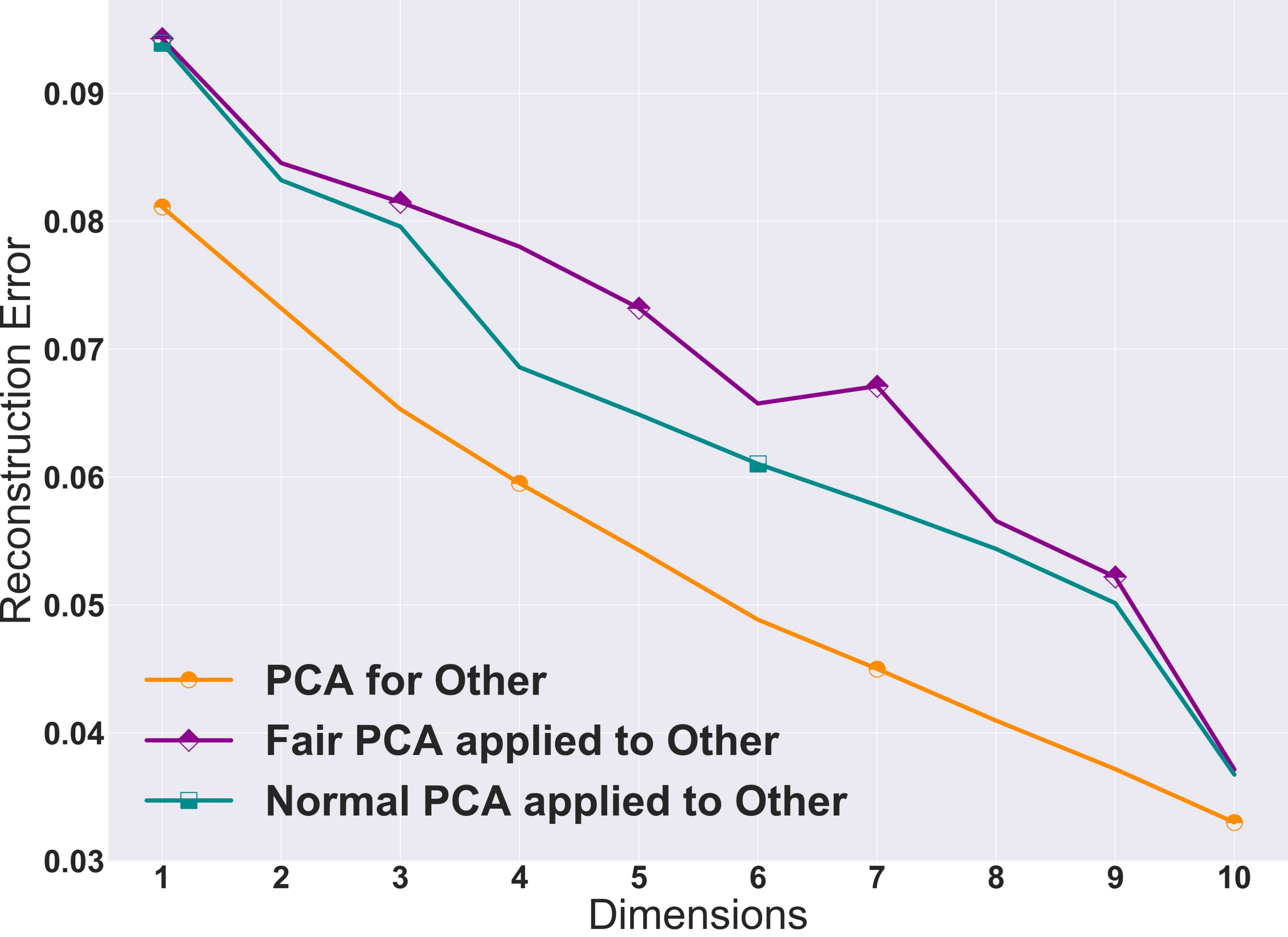}
		\label{fig:multi-other}
	\end{subfigure}
	
	\centering
    \begin{subfigure}[b]{0.32\textwidth}
		\centering
		\includegraphics[width=\textwidth]{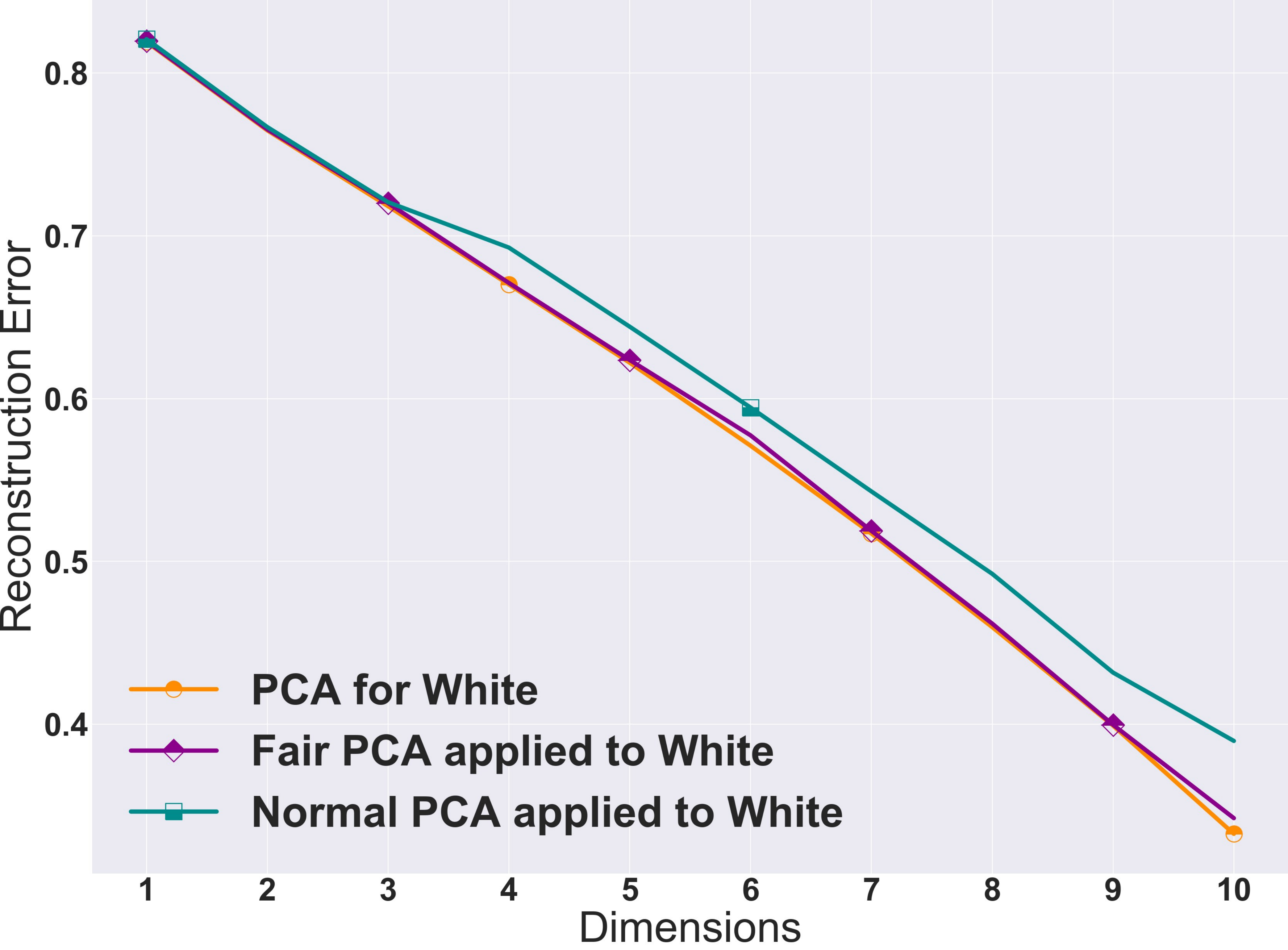}  
		\label{fig:multi-white}
	\end{subfigure}
	\begin{subfigure}[b]{0.32\textwidth}
		\centering
		\includegraphics[width=\textwidth]{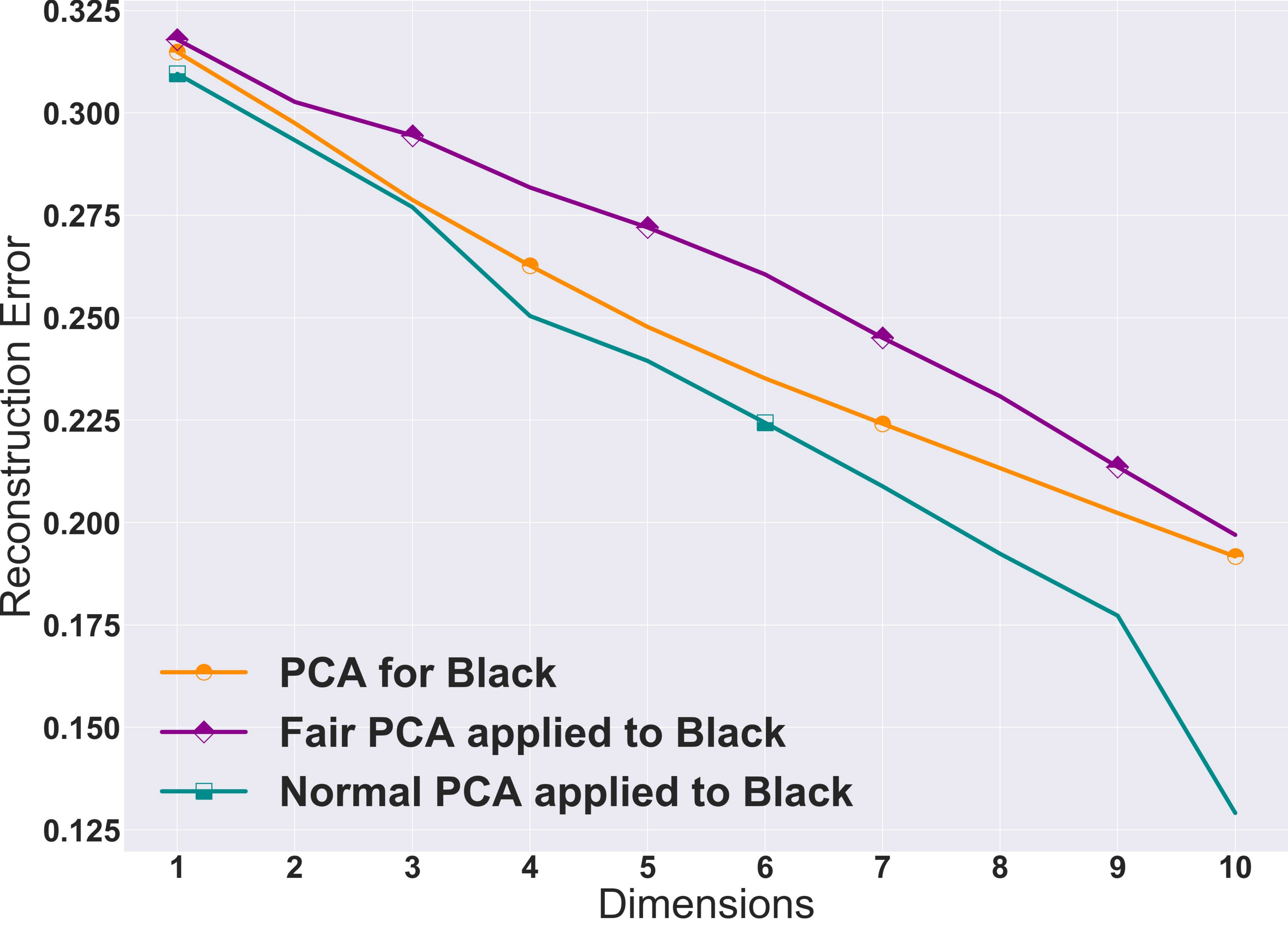}
		\label{fig:multi-black}
	\end{subfigure}
    \caption{Applying normal and fair PCA on the Adult dataset with ``race'' as its sensitive feature. Each plot shows the reconstruction error of the normal PCA (trained on the whole data) applied to each group's data, fair PCA (trained on the whole data) applied to each group's data, and normal PCA trained on each group's data individually. }
    \label{fig:multi_re}
\end{figure}

\begin{figure}[th!]
    \centering
    \begin{subfigure}[b]{0.32\textwidth}
		\centering
		\includegraphics[width=\textwidth]{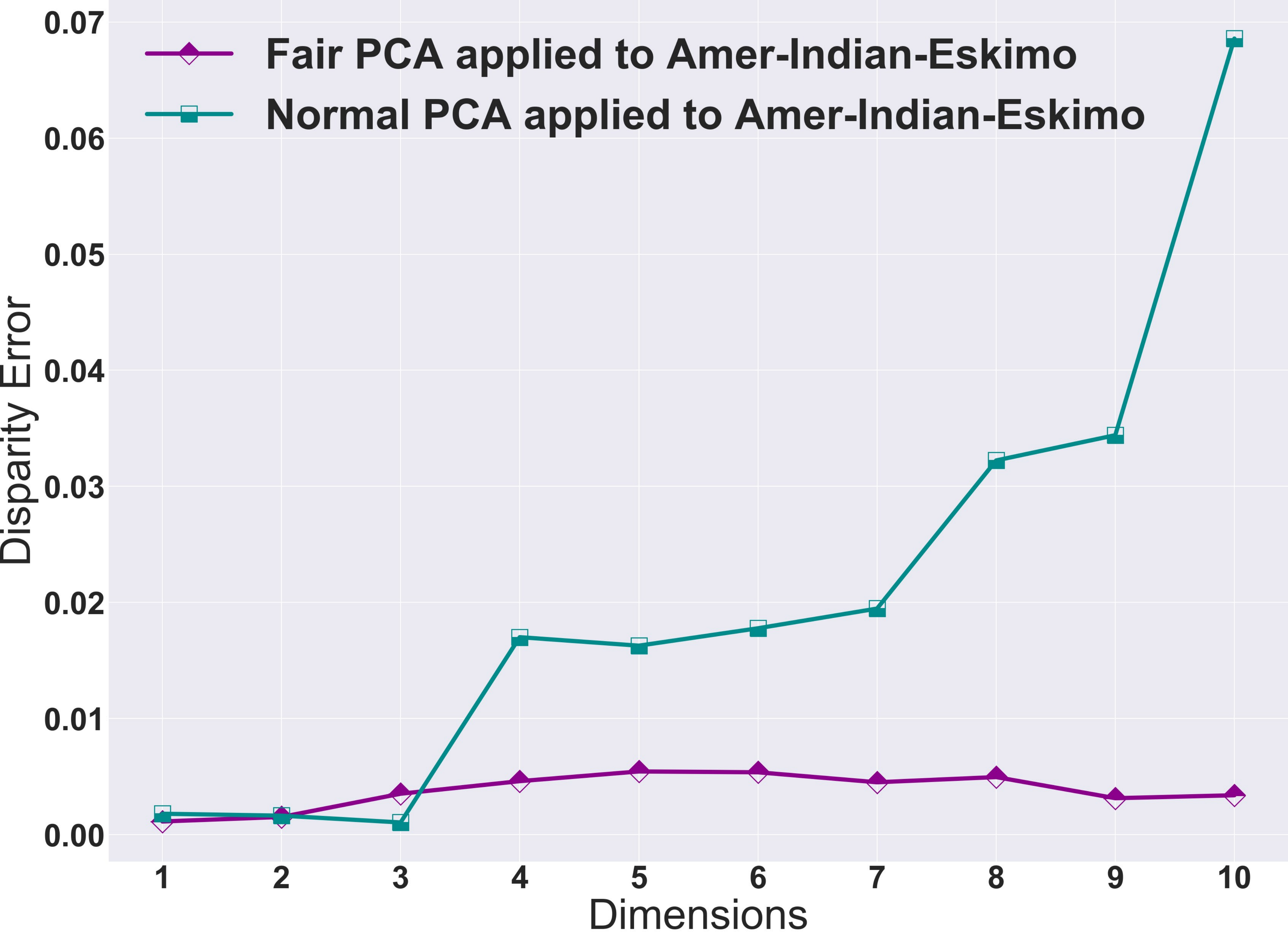}  
		\label{fig:de-multi-AIE}
	\end{subfigure}
	\hfill
	\begin{subfigure}[b]{0.32\textwidth}
		\centering
		\includegraphics[width=\textwidth]{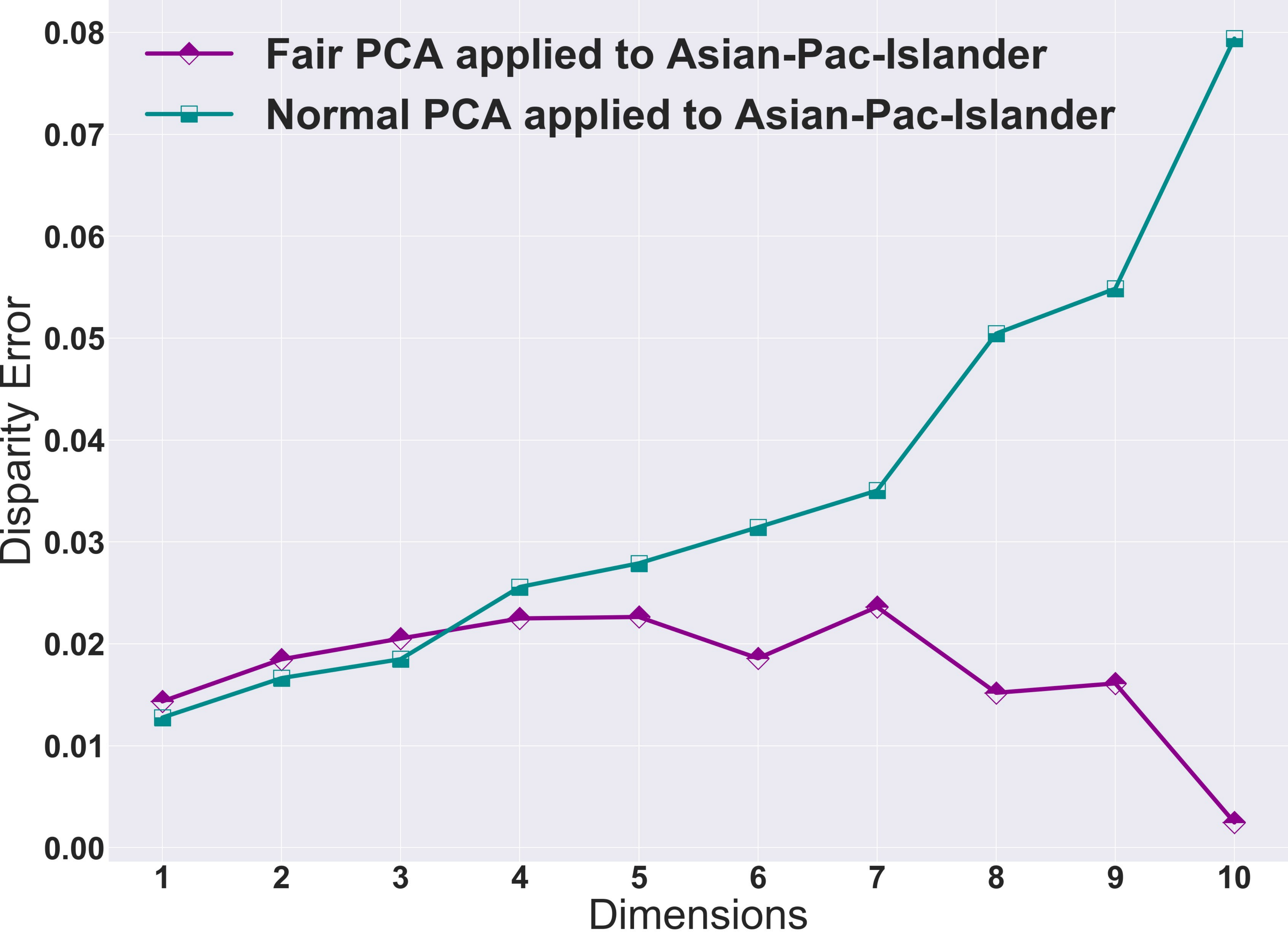}
		\label{fig:de-multi-API}
	\end{subfigure}
	\hfill
	\begin{subfigure}[b]{0.32\textwidth}
		\centering
		\includegraphics[width=\textwidth]{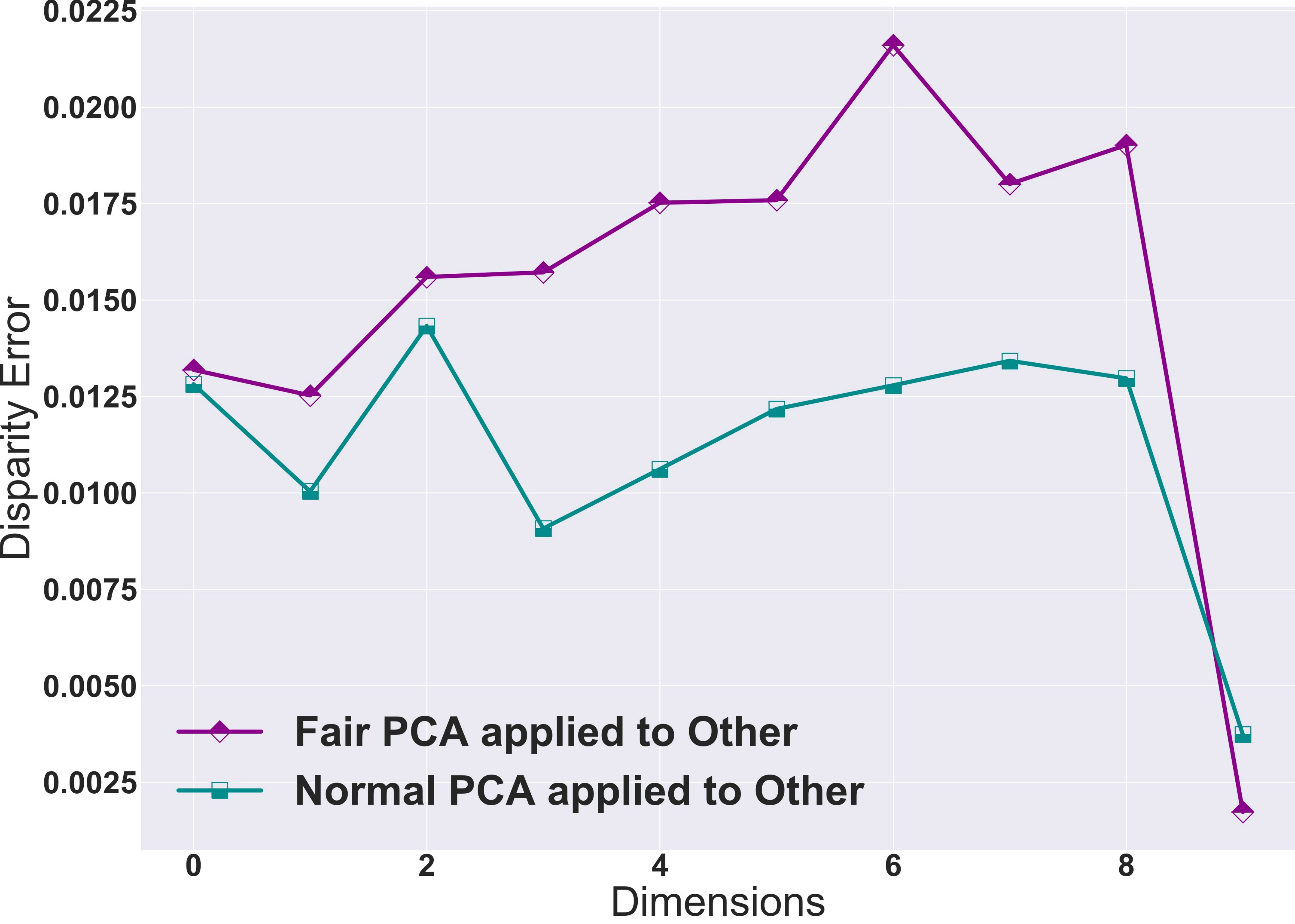}
		\label{fig:de-multi-other}
	\end{subfigure}
	
	\centering
    \begin{subfigure}[b]{0.32\textwidth}
		\centering
		\includegraphics[width=\textwidth]{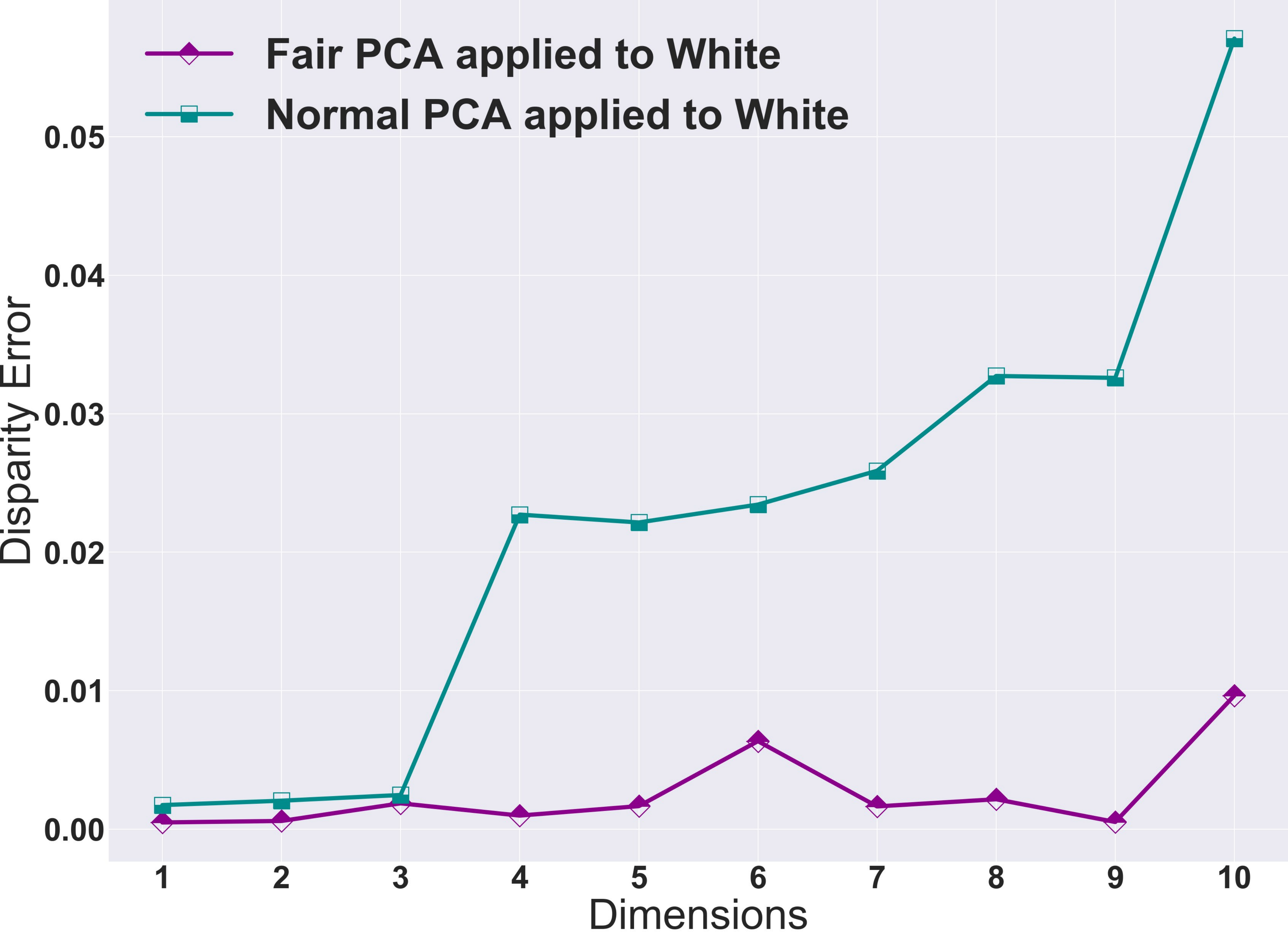}  
		\label{fig:de-multi-white}
	\end{subfigure}
	\begin{subfigure}[b]{0.32\textwidth}
		\centering
		\includegraphics[width=\textwidth]{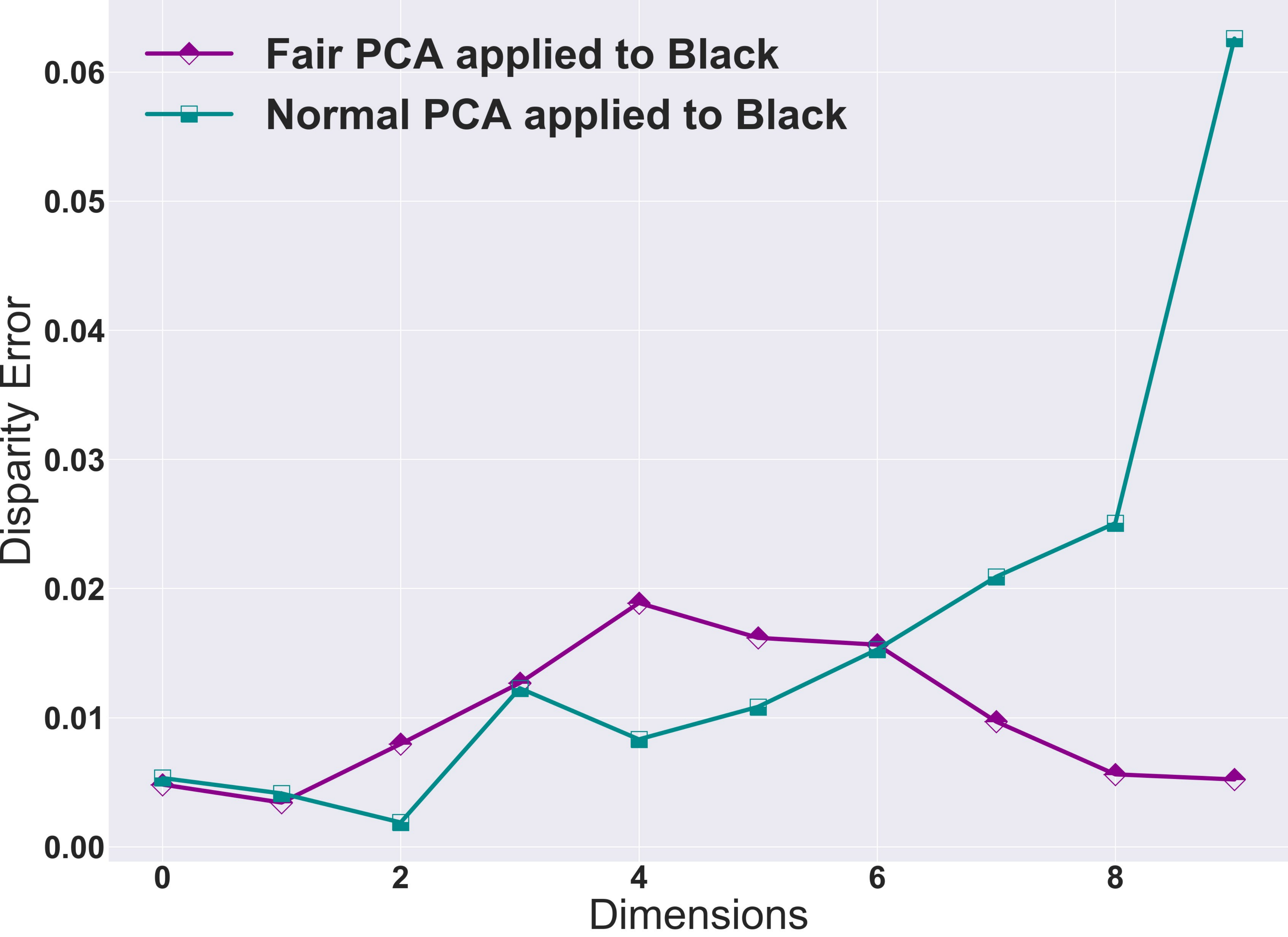}
		\label{fig:de-multi-black}
	\end{subfigure}
    \caption{Disparity error of normal and fair PCA trained on the Adult dataset with ``race'' as its sensitive feature. Each plot depicts the disparity errors of different groups with normal and fair PCA.}
    \label{fig:de-multi_re}
\end{figure}

\begin{figure}[t!]
    \centering
	\begin{subfigure}[b]{0.45\textwidth}
		\centering
		\includegraphics[width=\textwidth]{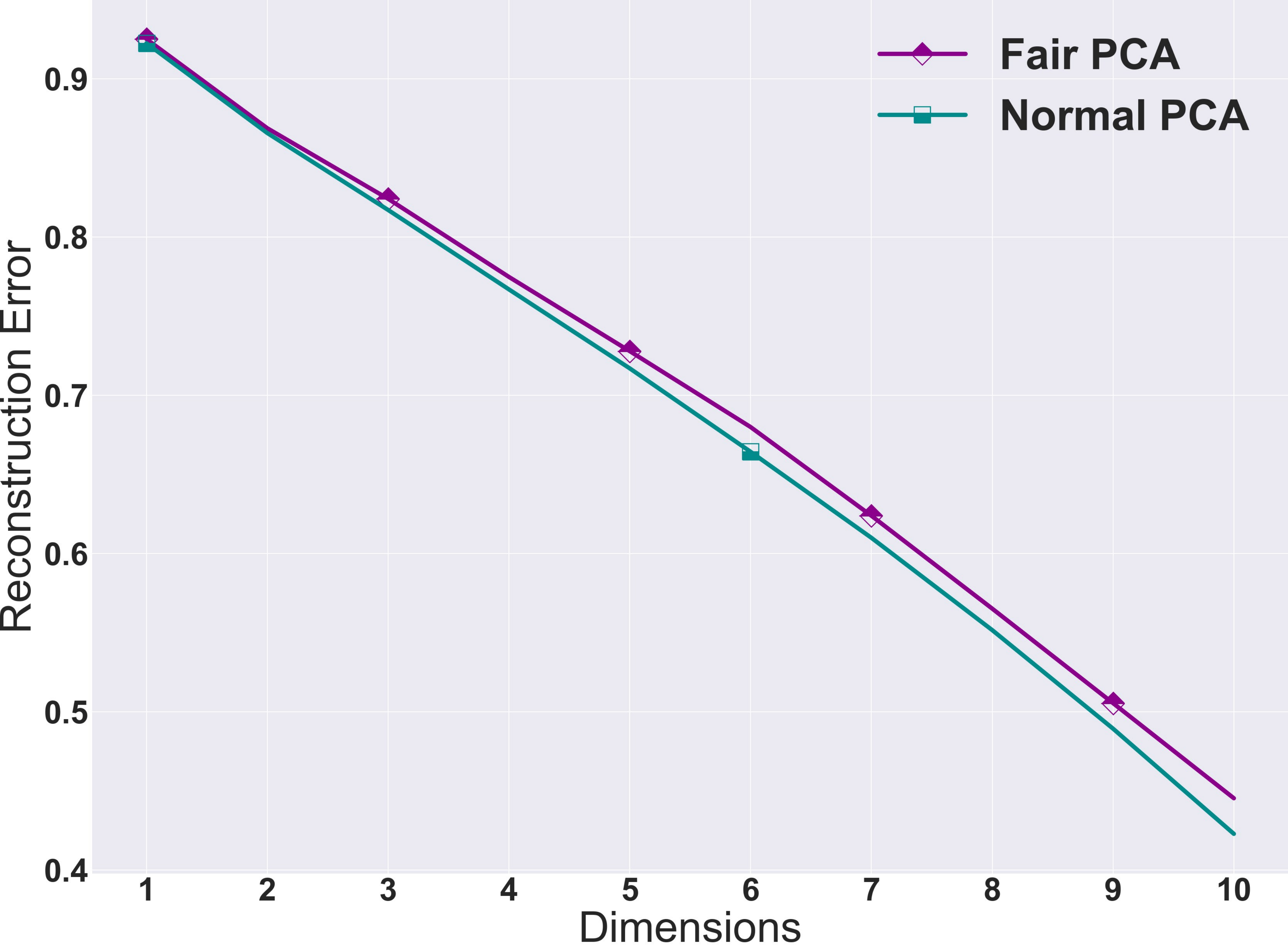}
		\caption{Total Reconstruction Loss}
		\label{fig:multi-total}
	\end{subfigure}
    \hfill
	\begin{subfigure}[b]{0.45\textwidth}
		\centering
		\includegraphics[width=\textwidth]{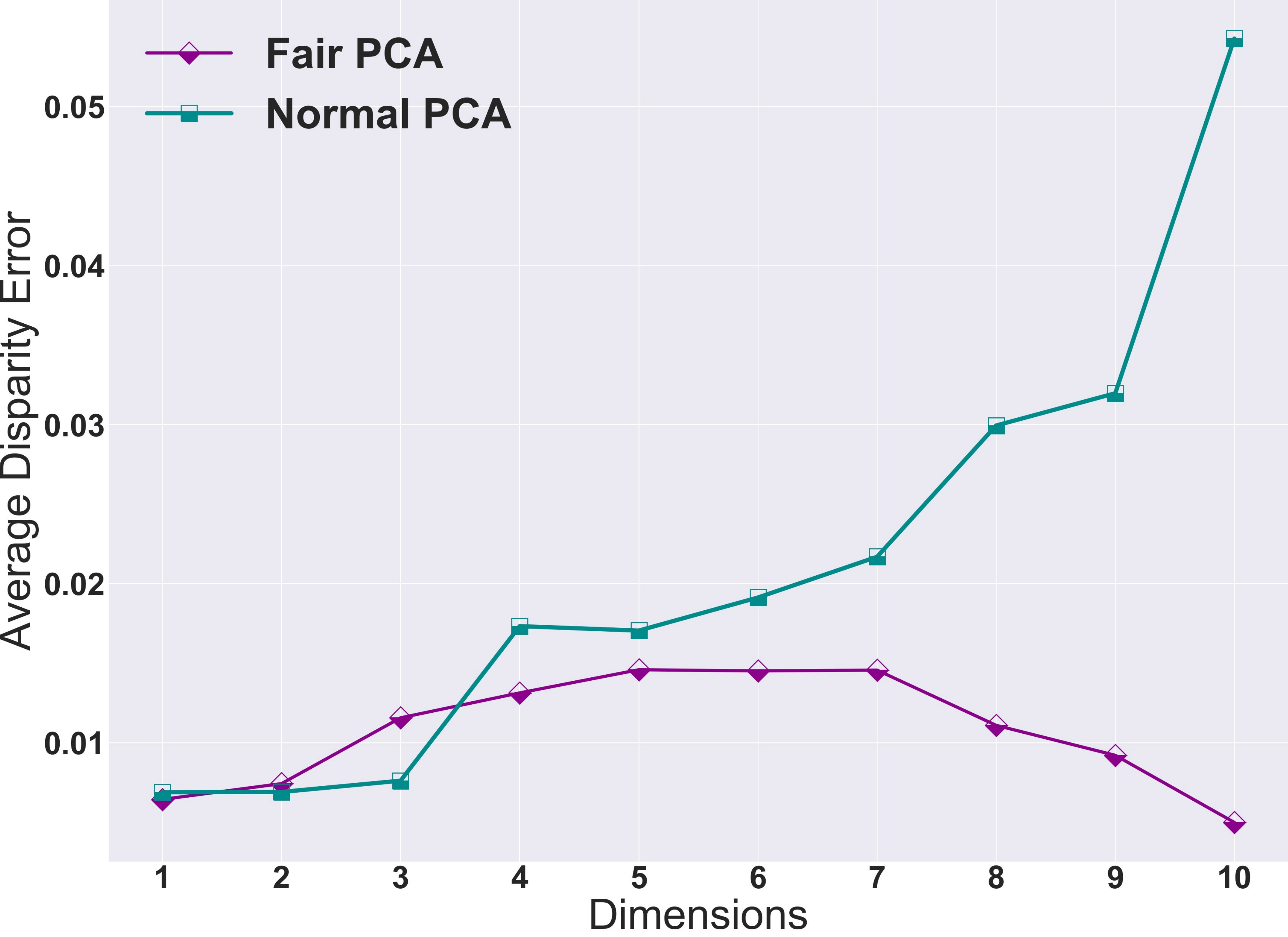}
		\caption{Average Disparity Error}	\label{fig:demulti-average}
	\end{subfigure}
    \caption{Applying normal and fair PCA on the Adult dataset with ``race'' as its multiple group sensitive feature.  Left shows the difference between reconstruction loss of fair and normal PCA, which is infinitesimal. On the other hand, the right plot shows their difference in terms of average disparity error, which is huge and demonstrating the efficacy of fair PCA in addressing fairness even in multiple group sensitive feature cases. }
    \label{fig:multi_both}
\end{figure}

As for the Credit dataset, we also test it on its multiple group sensitive feature, marriage, which has $3$ different groups. The result of reconstruction error on Pareto fair PCA, normal PCA and normal PCA on each group's data individually is depicted in Figure~\ref{fig:multi_re-credit}, where it is clear that fair PCA is very close to each group's PCA (except for the ``other'' group, because the number of samples in that group is too low), while its reconstruction error is very close to that of normal PCA. Also, the disparity error and average disparity error of fair PCA versus normal PCA is shown in Figure~\ref{fig:de-multi_re-credit}, where the superiority of fair PCA is noticeable.

\begin{figure}[th!]
    \centering
    \begin{subfigure}[b]{0.31\textwidth}
		\centering
		\includegraphics[width=\textwidth]{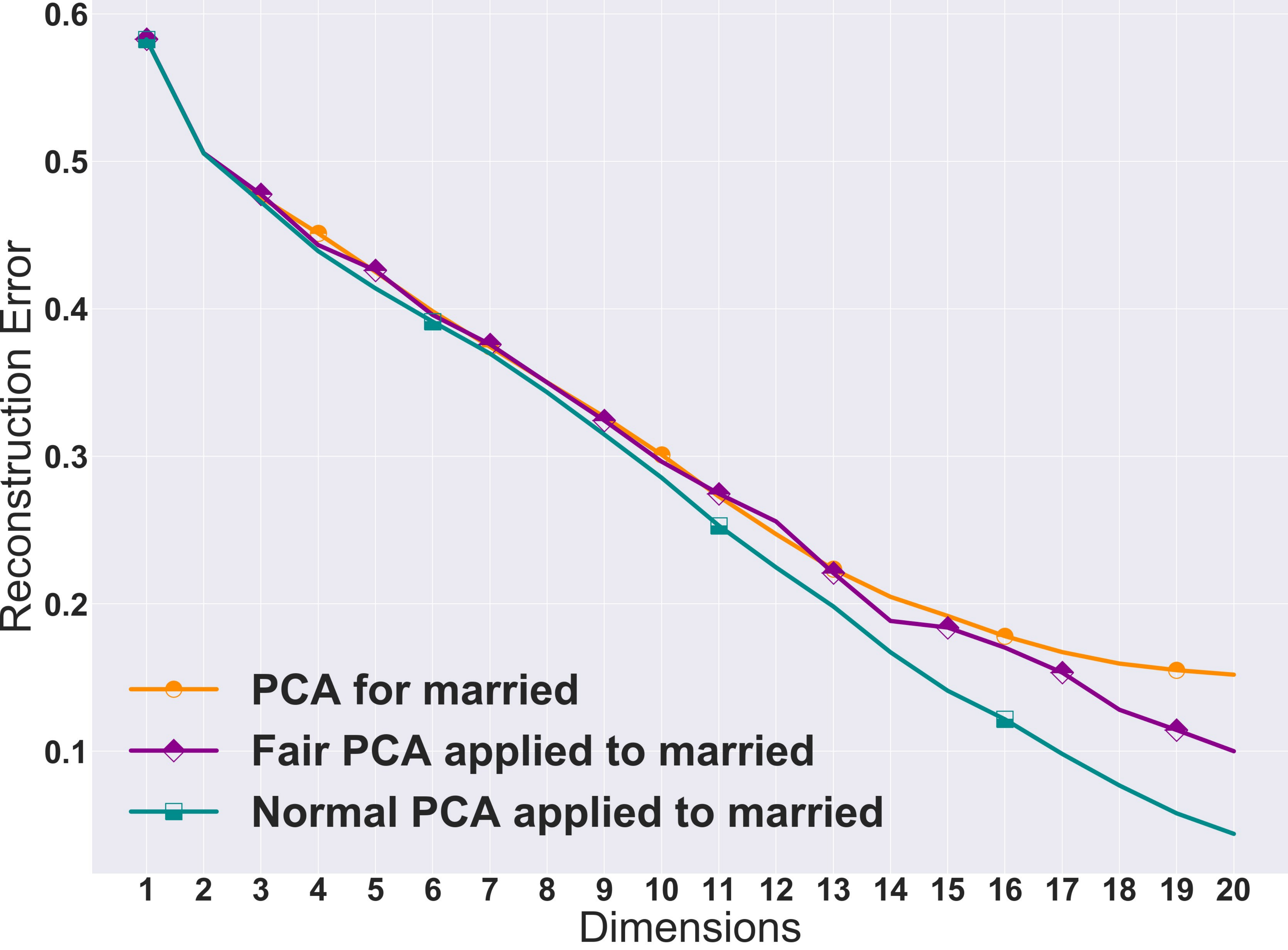}  
		\label{fig:multi-credit-mar}
	\end{subfigure}
	\hfill
	\begin{subfigure}[b]{0.31\textwidth}
		\centering
		\includegraphics[width=\textwidth]{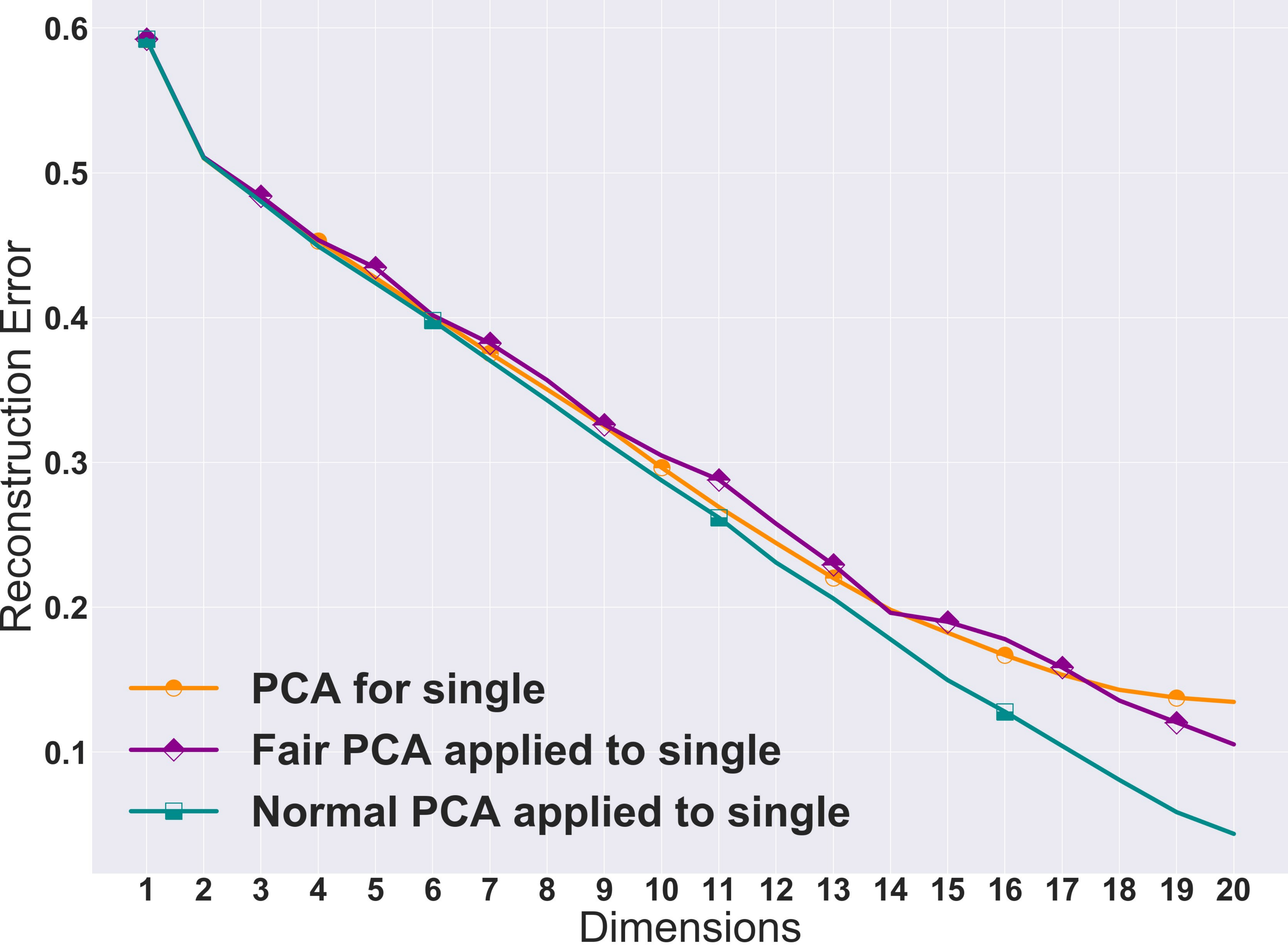}
		\label{fig:multi-credit-sing}
	\end{subfigure}
	\hfill
	\begin{subfigure}[b]{0.31\textwidth}
		\centering
		\includegraphics[width=\textwidth]{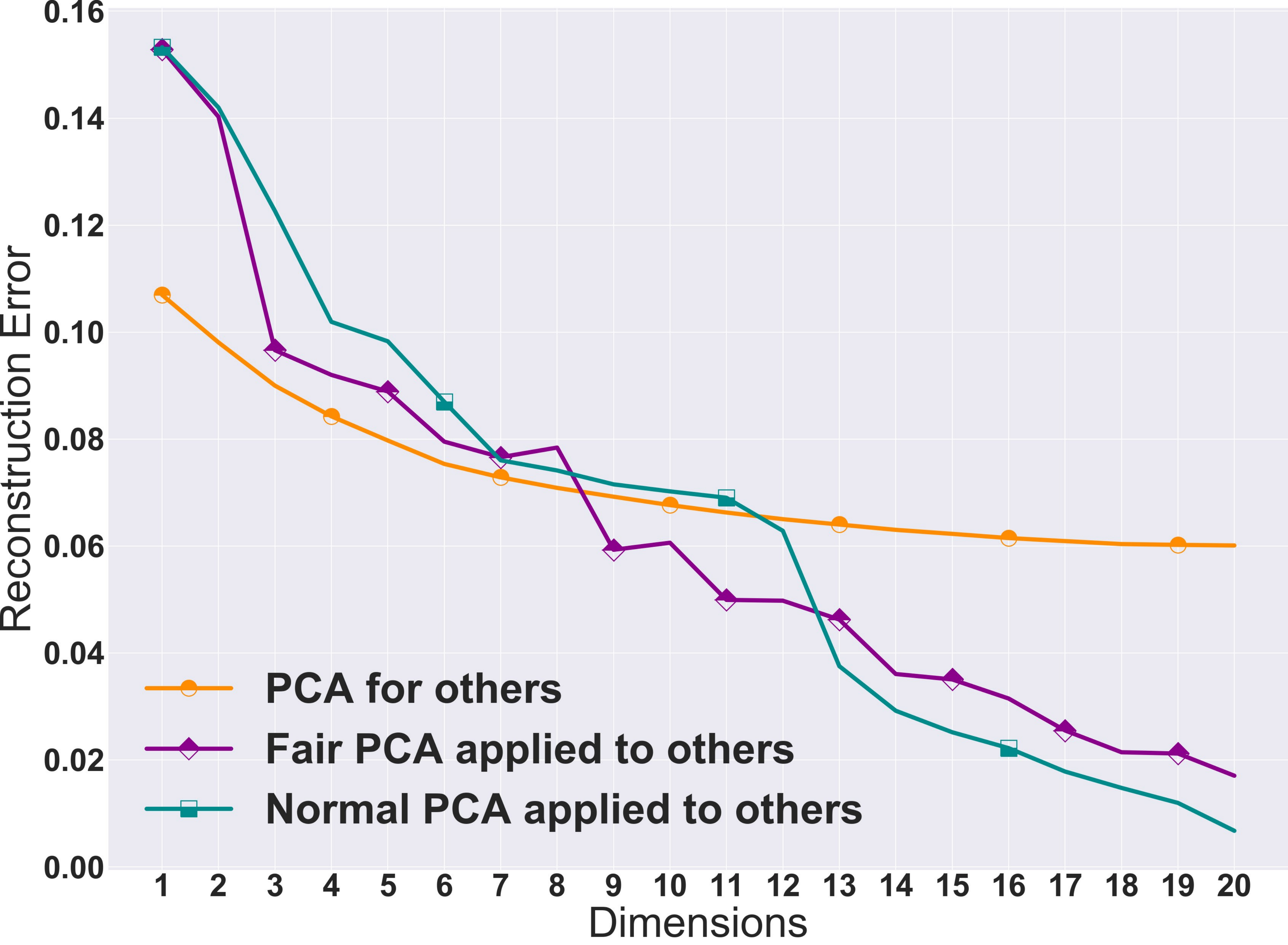}
		\label{fig:multi-credit-other}
	\end{subfigure}
	
	\centering
	\begin{subfigure}[b]{0.31\textwidth}
		\centering
		\includegraphics[width=\textwidth]{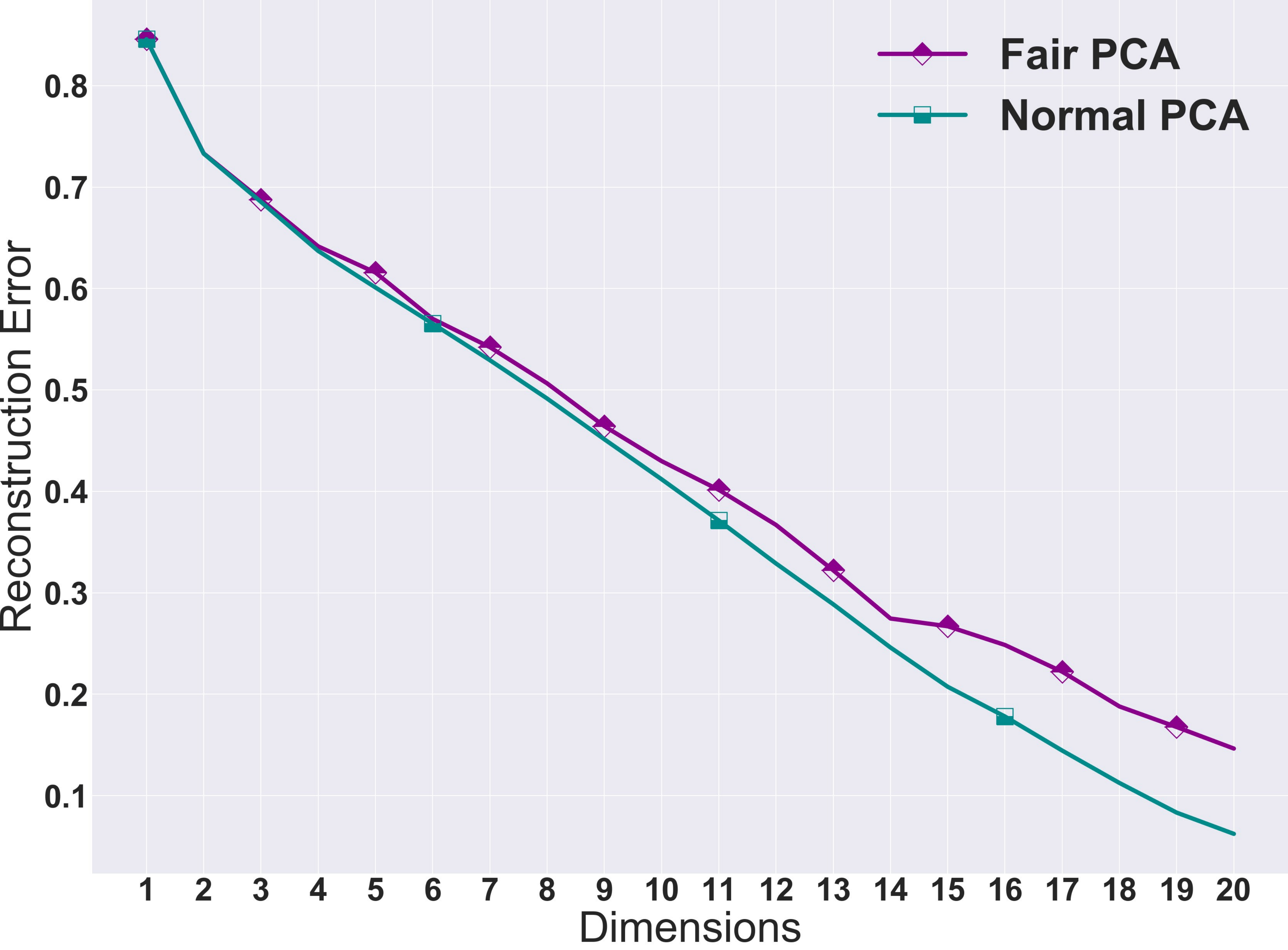}
		\label{fig:multi-credit-total}
	\end{subfigure}
    \caption{Applying normal and fair PCA on the Credit dataset with ``Marriage'' as its sensitive feature. Each plot shows the reconstruction error of the normal PCA (trained on the whole data) applied to each group's data, fair PCA (trained on the whole data) applied to each group's data, and normal PCA trained on each group's data individually. The last figure shows the reconstruction error of normal PCA and fair PCA on this dataset with multiple group sensitive feature of marriage. }
    \label{fig:multi_re-credit}
    \vspace{-0.1cm}
\end{figure}

\begin{figure}[th!]
    \centering
    \begin{subfigure}[b]{0.31\textwidth}
		\centering
		\includegraphics[width=\textwidth]{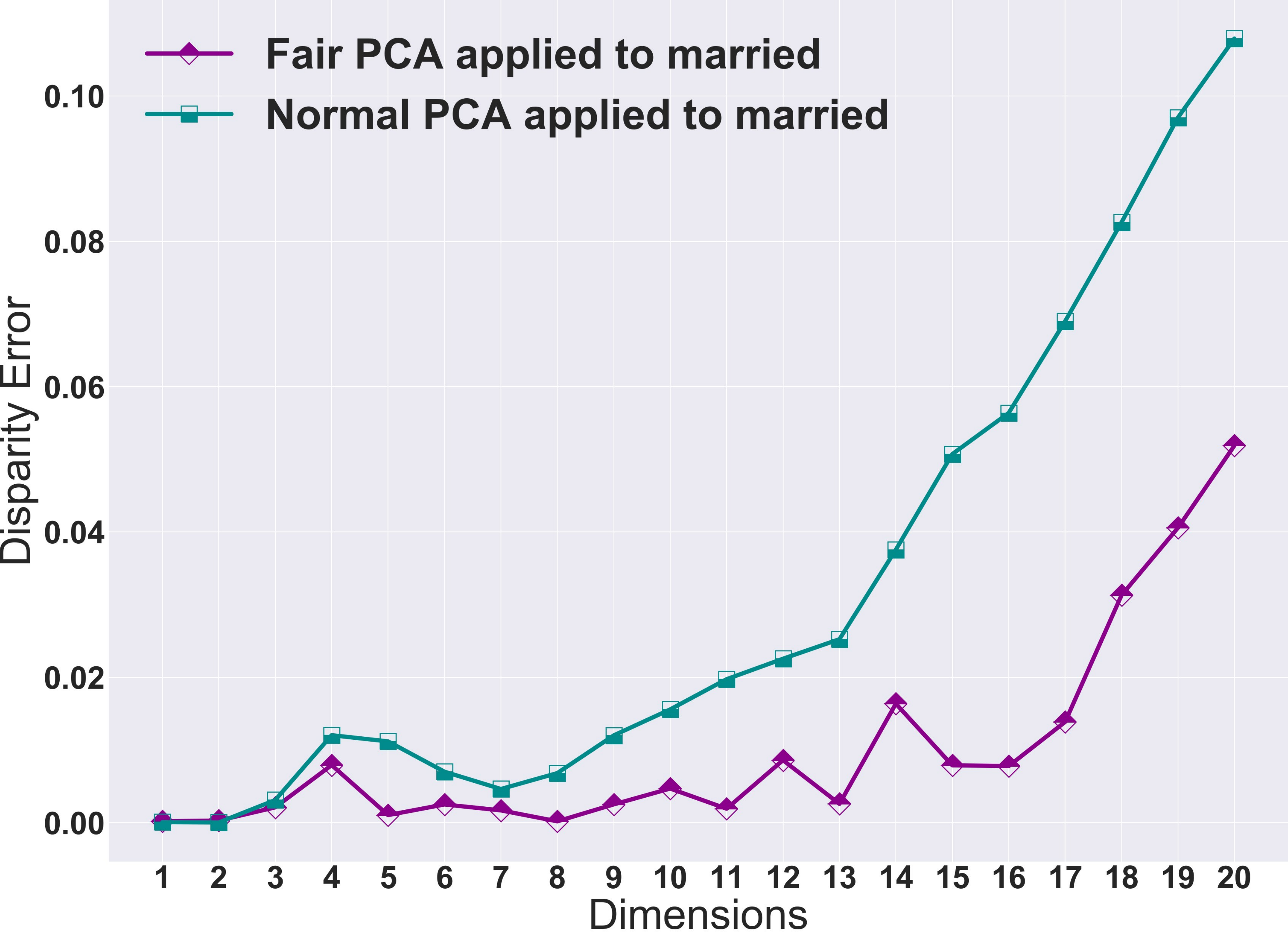}  
		\label{fig:de-multi-credit-mar}
	\end{subfigure}
	\hfill
	\begin{subfigure}[b]{0.31\textwidth}
		\centering
		\includegraphics[width=\textwidth]{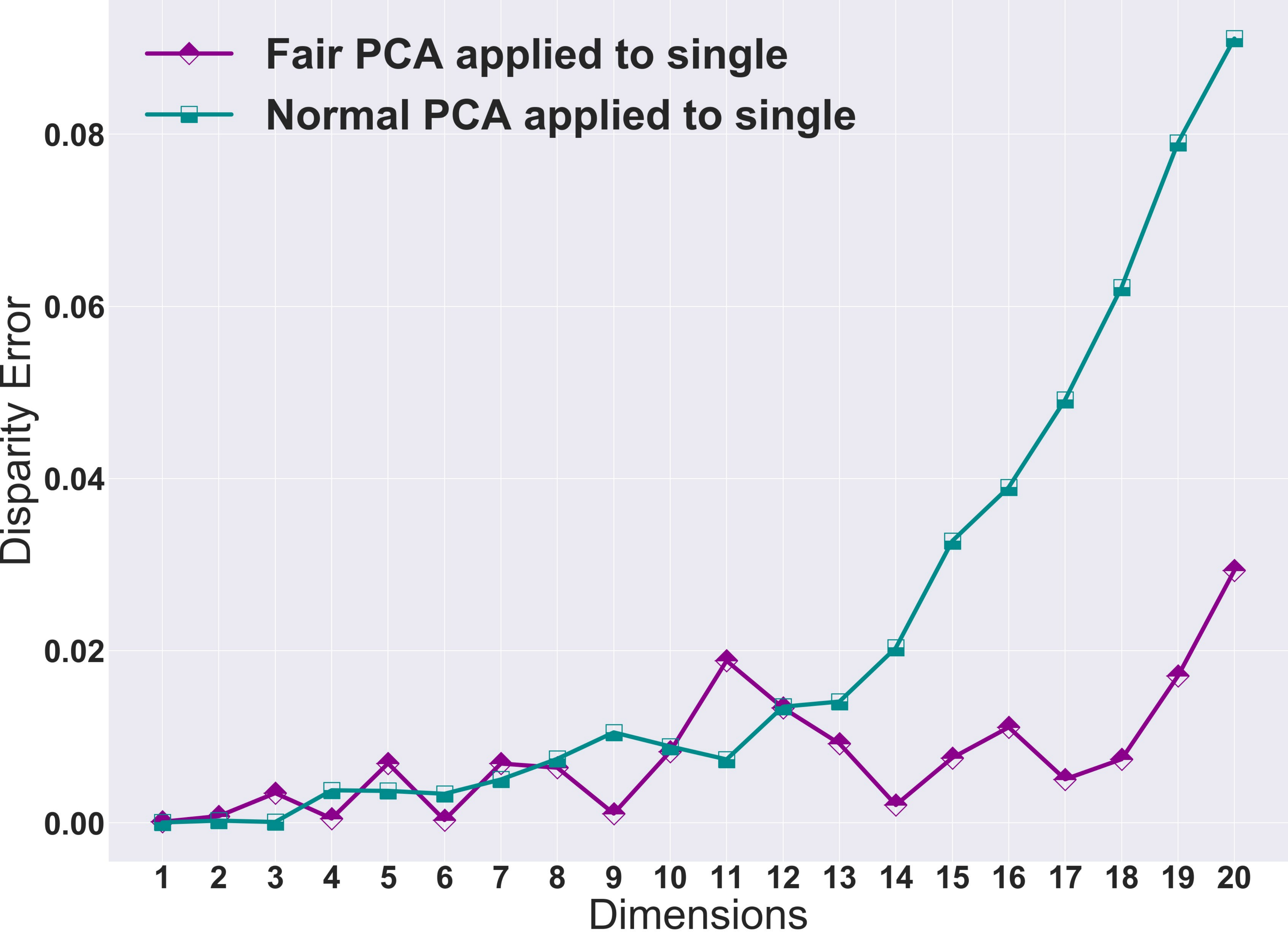}
		\label{fig:de-multi-credit-sing}
	\end{subfigure}
	\hfill
	\begin{subfigure}[b]{0.31\textwidth}
		\centering
		\includegraphics[width=\textwidth]{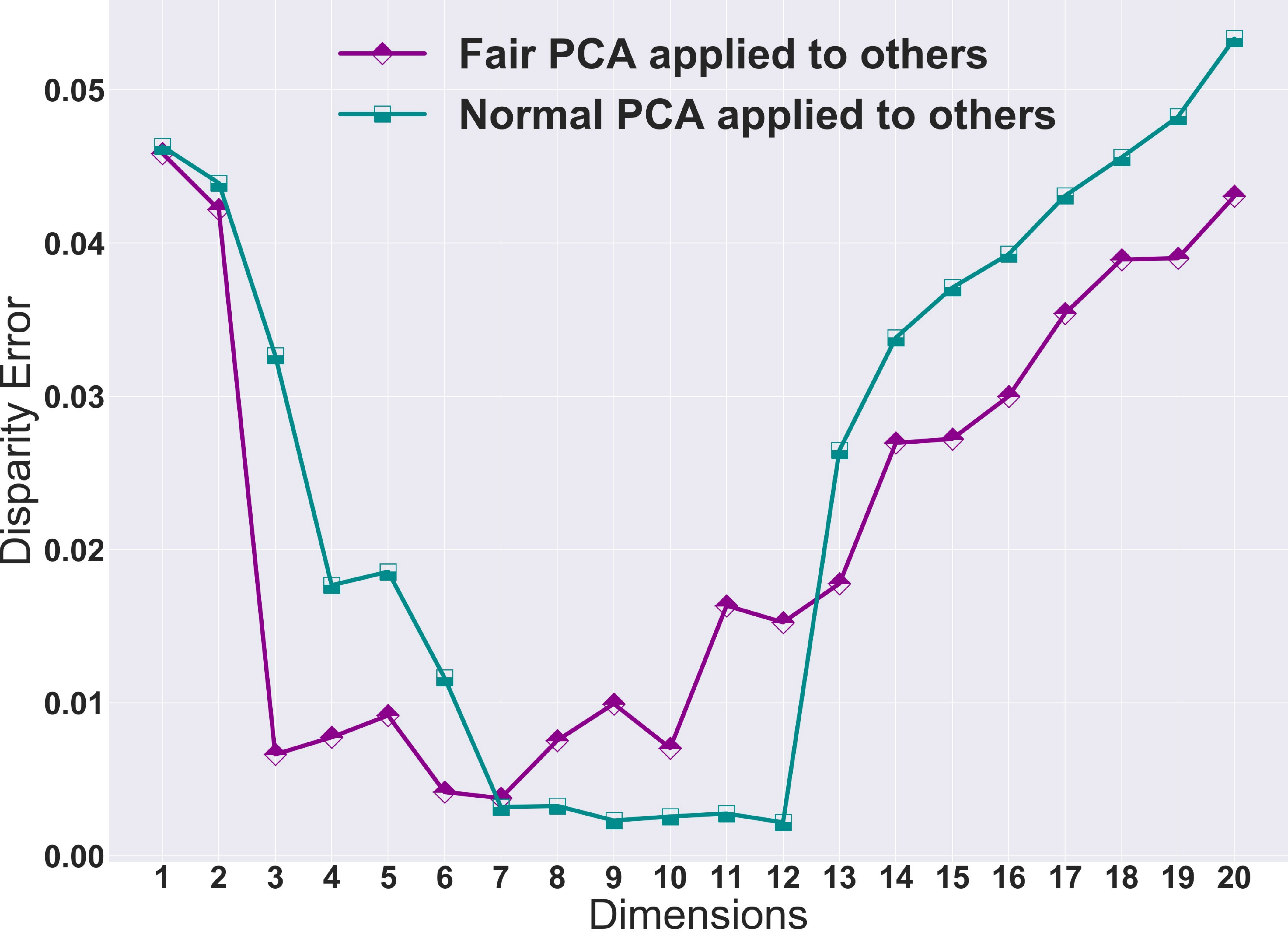}
		\label{fig:de-multi-credit-other}
	\end{subfigure}
	
	\centering
	\begin{subfigure}[b]{0.31\textwidth}
		\centering
		\includegraphics[width=\textwidth]{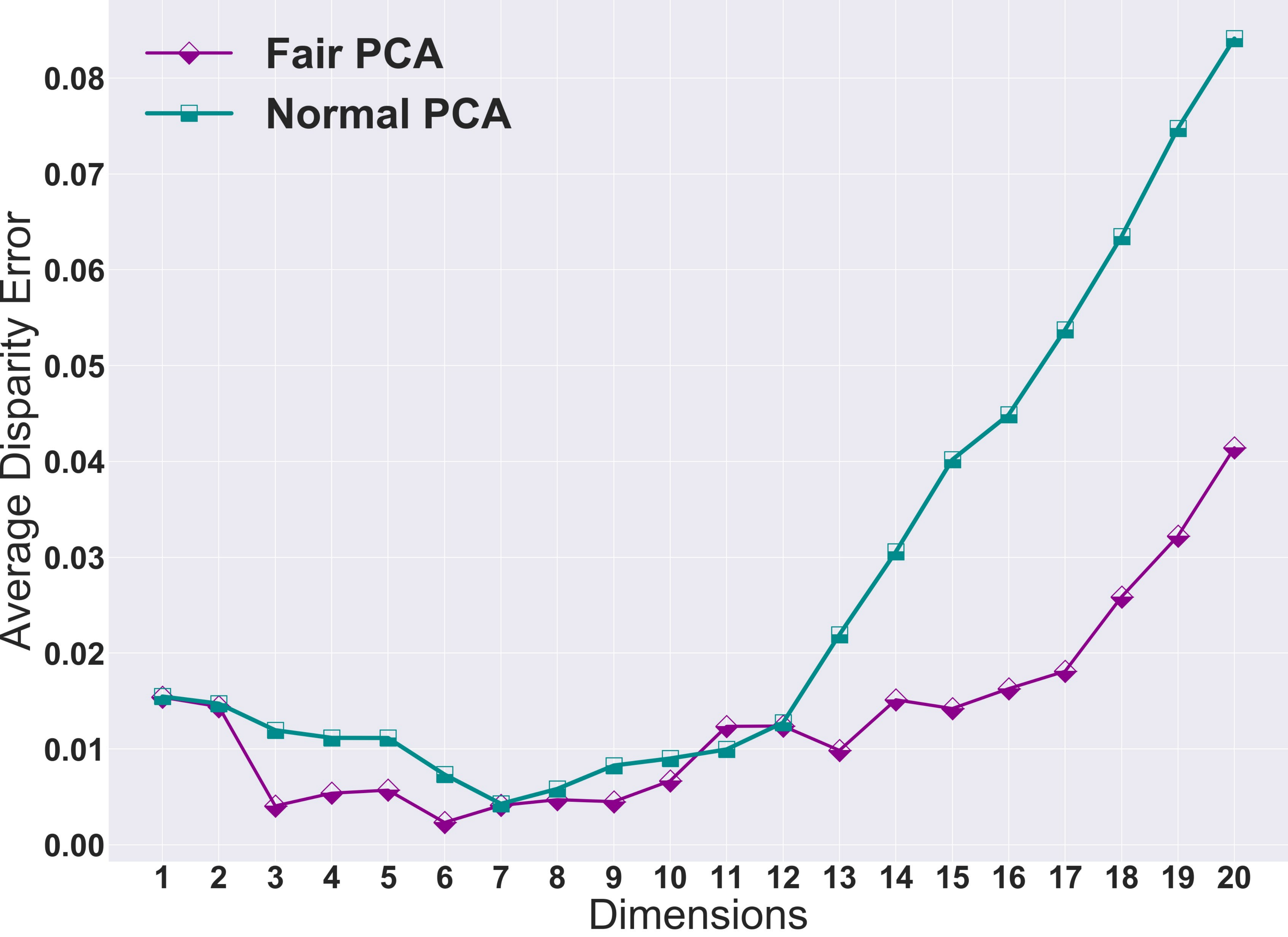}
		\label{fig:demulti-credit-average}
	\end{subfigure}
    \caption{Disparity error of normal and fair PCA trained on the Credit dataset with ``Marriage'' as its sensitive feature. Each plot depicts the disparity errors of different groups with normal and fair PCA. The last figure shows the average of disparity errors across groups.}
    \label{fig:de-multi_re-credit}
\end{figure}

\subsection{Fairness in composition}
Most of the time, when we use a dimension reduction algorithm, it is accompanied by some downstream tasks such as classifiers. Hence, it is important to investigate the effects of our fairness dimension reduction algorithm on those downstream tasks. Here, we empirically examine this effect on a simple classifier. To that end, we use both the Adult and Credit datasets and first reduce the dimension of their feature space to $10$, and then use the new projection to learn a standard linear SVM model. One standard fairness measure in the supervised domain is called Equality of Opportunity~\citep{hardt2016equality}, where the goal is to ensure that the true positive rate among different sensitive features does not differ significantly. For a binary sensitive group~\cite{donini2018empirical} introduced a measure called difference of equality of opportunity, which is $\text{DEO} = | \texttt{TP}_1 - \texttt{TP}_2|$ with $\texttt{TP}_i$ representing true positive rate of the $i$th group in a sensitive feature. This measure shows the gap between the two groups' true positive rates. As can be inferred from Figure~\ref{fig:composition}, applying Pareto fair PCA can boost fairness of the downstream model and dramatically drop the gap between two groups' true positive rates (DEO) with respect to the normal PCA.

\begin{figure}[t!]
    \centering
	\begin{subfigure}[b]{0.32\textwidth}
		\centering
		\includegraphics[width=\textwidth]{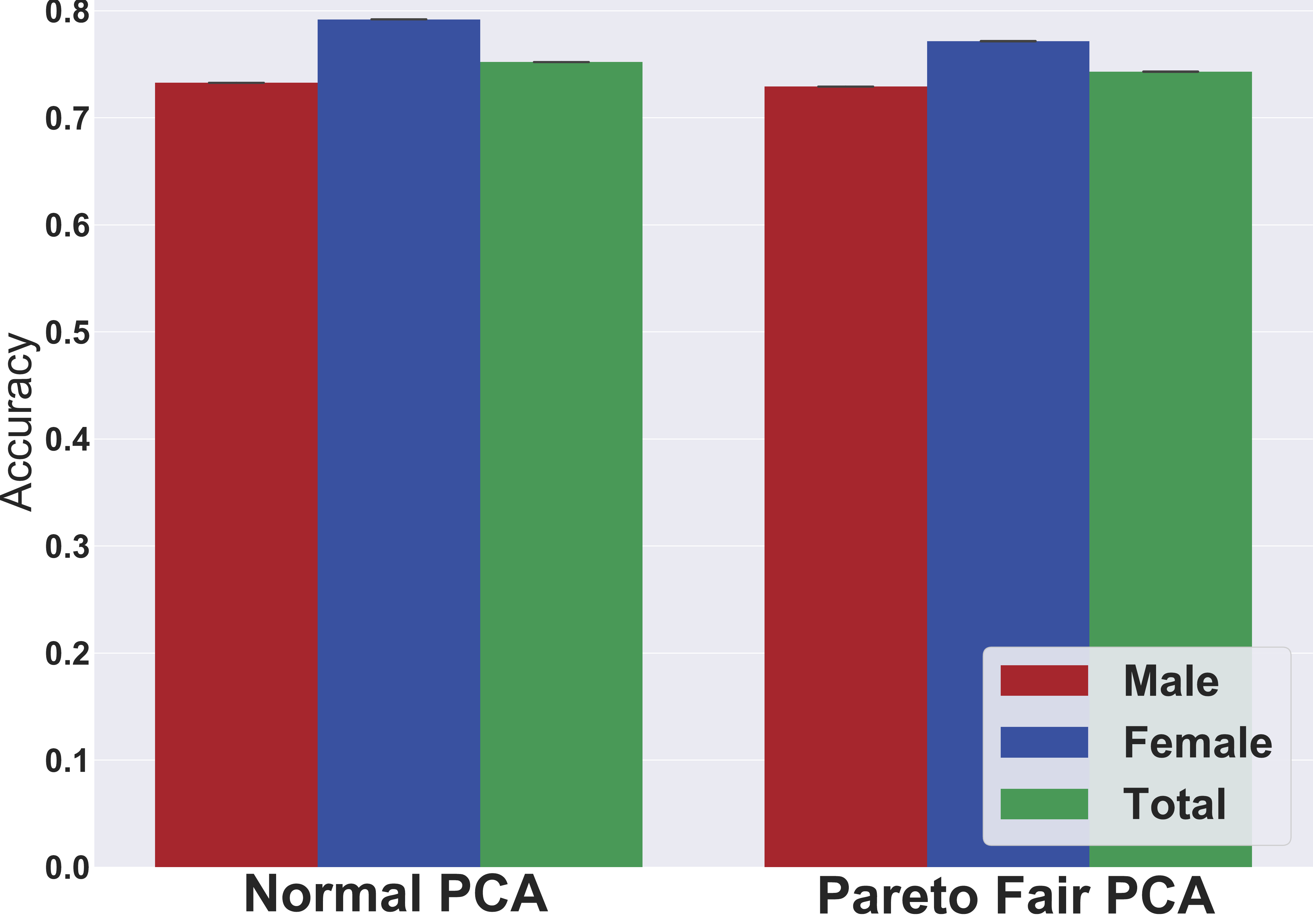}
		\label{fig:compose_adult_acc}
	\end{subfigure}
    \hfill
	\begin{subfigure}[b]{0.32\textwidth}
		\centering
		\includegraphics[width=\textwidth]{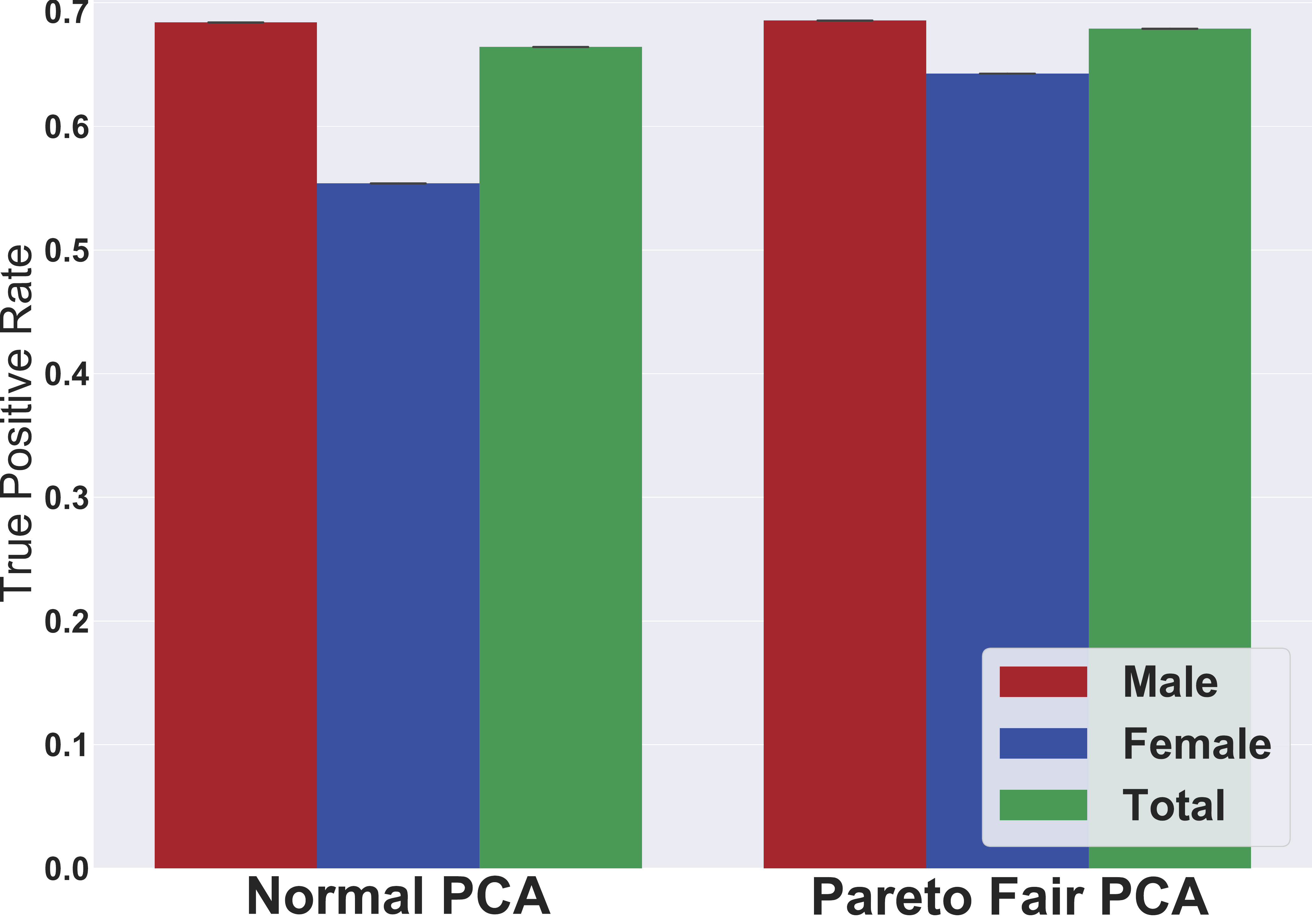}
			\label{fig:compose_adult_tpr}
	\end{subfigure}
	\hfill
	\begin{subfigure}[b]{0.32\textwidth}
		\centering
		\includegraphics[width=\textwidth]{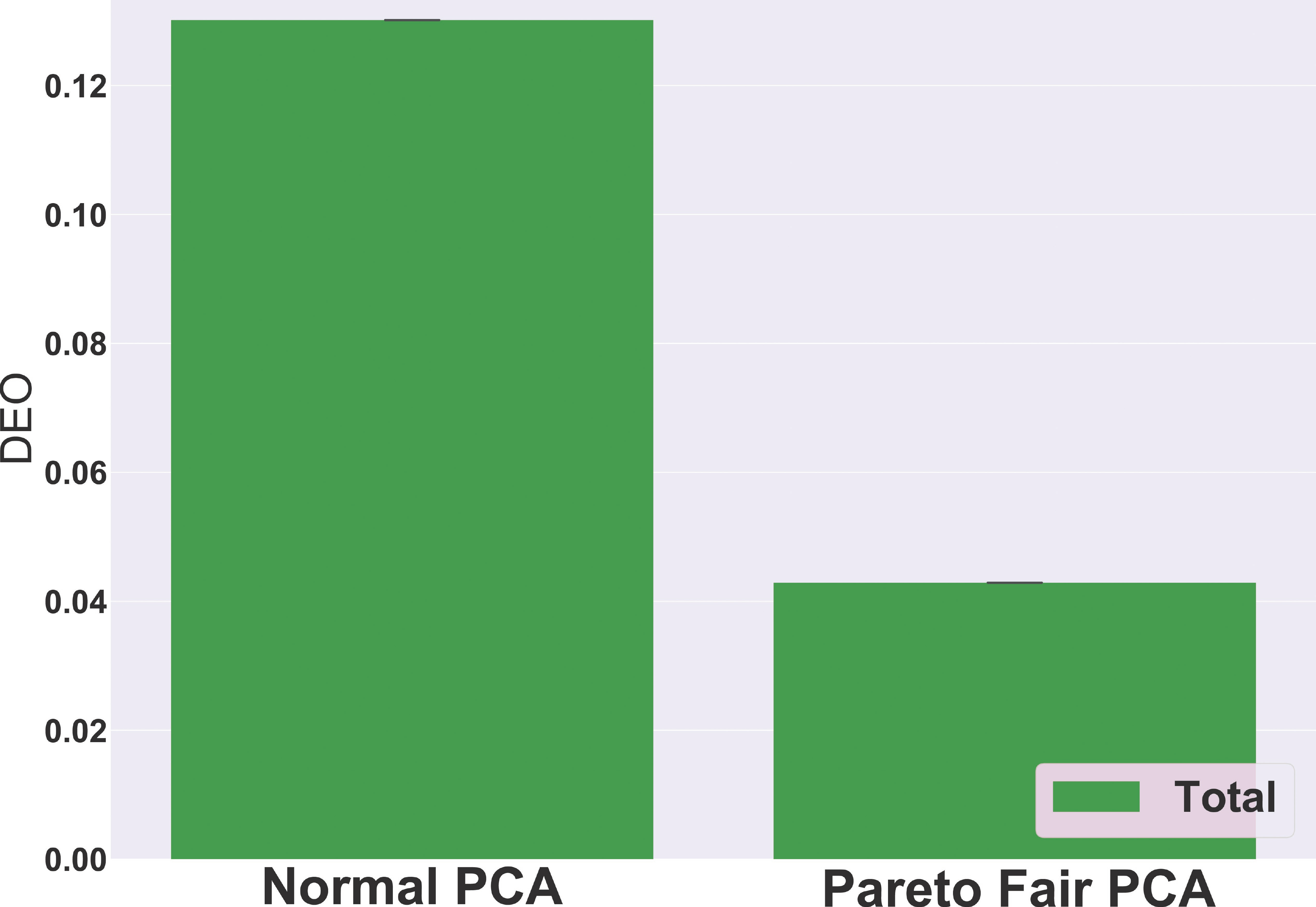}
			\label{fig:compose_adult_deo}
	\end{subfigure}
	
	\setcounter{subfigure}{0}
	\begin{subfigure}[b]{0.32\textwidth}
		\centering
		\includegraphics[width=\textwidth]{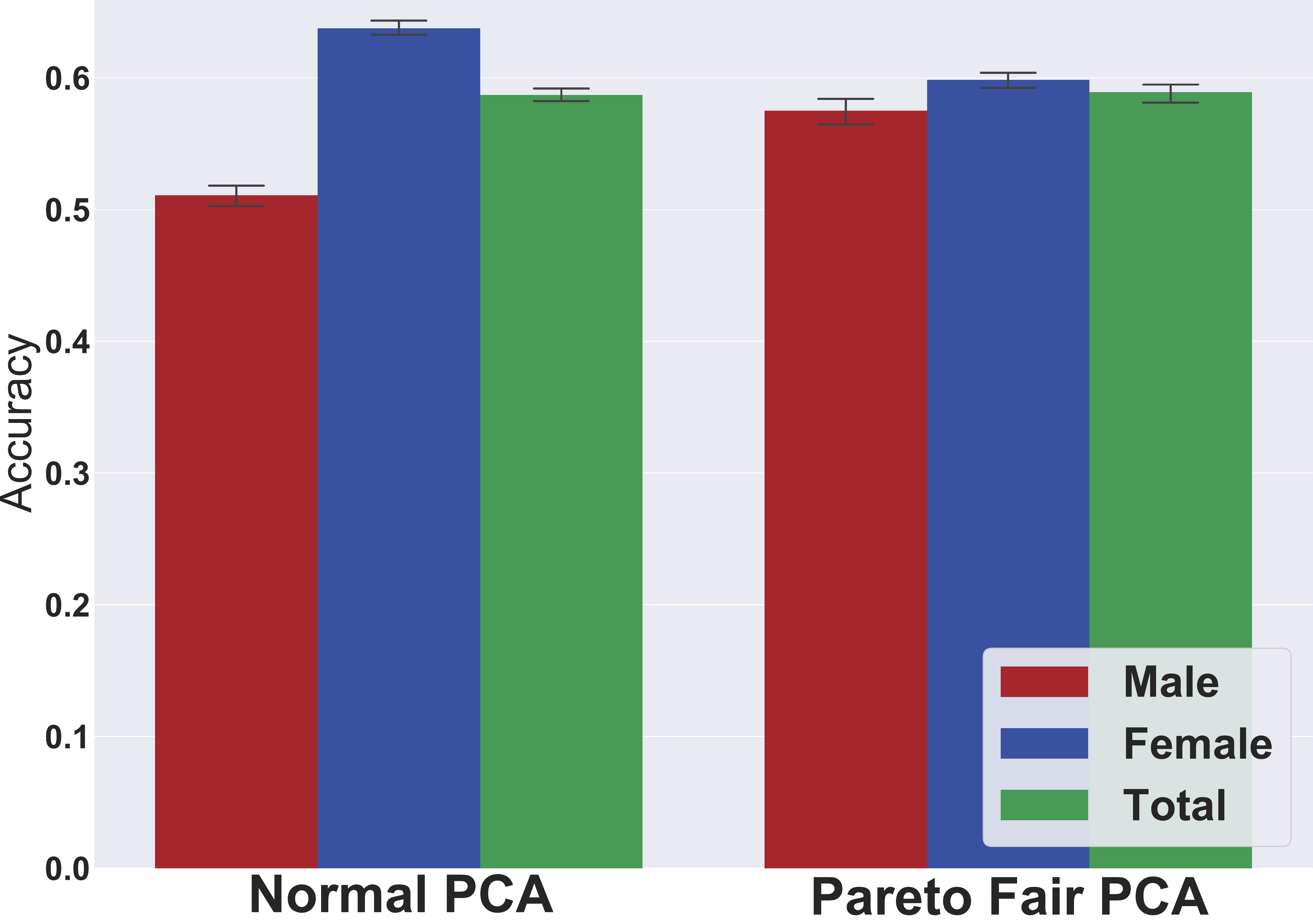}
		\caption{Test Accuracy}
		\label{fig:compose_credit_acc}
	\end{subfigure}
    \hfill
	\begin{subfigure}[b]{0.32\textwidth}
		\centering
		\includegraphics[width=\textwidth]{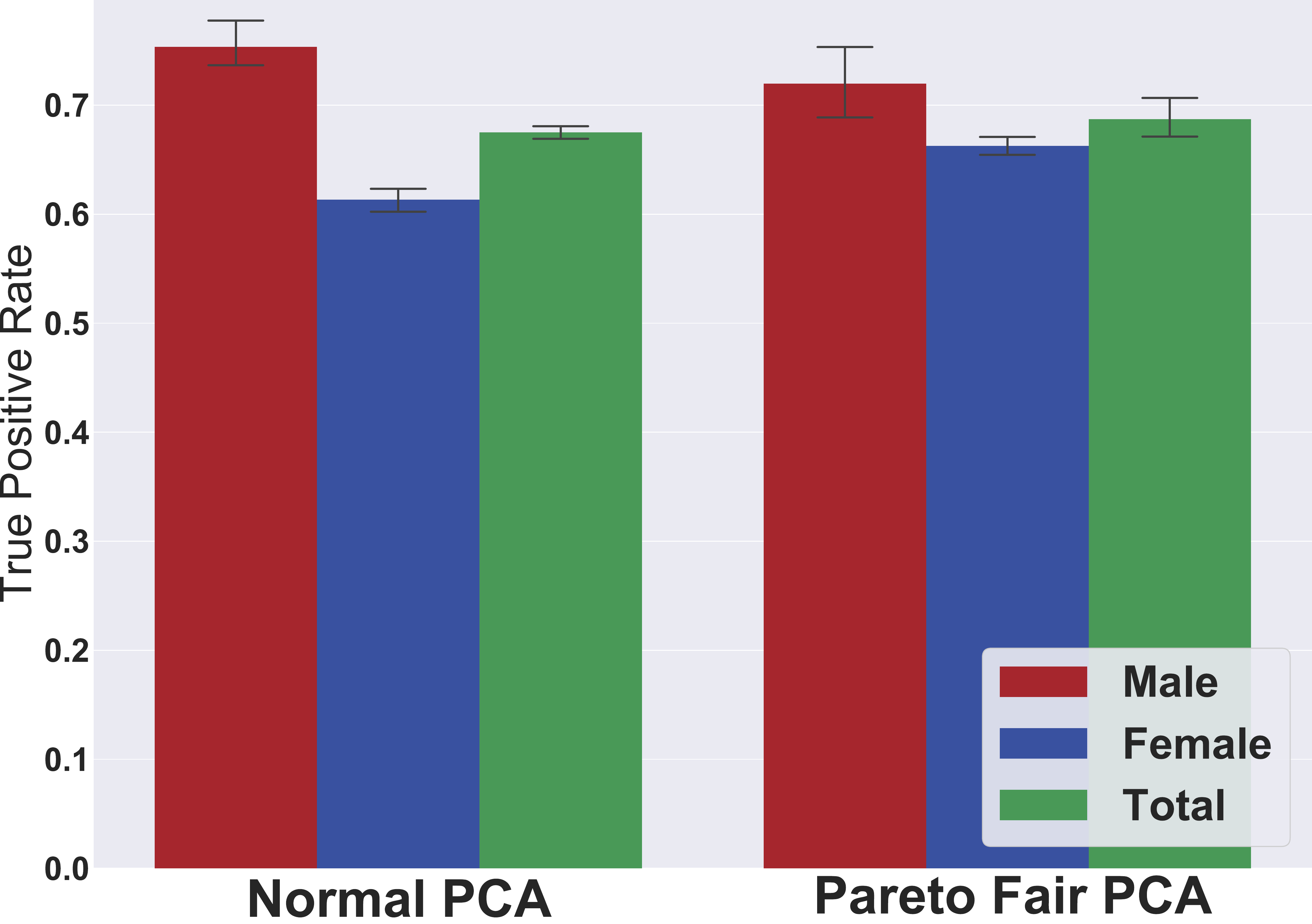}
		\caption{True Positive Rate}	\label{fig:compose_credit_tpr}
	\end{subfigure}
	\hfill
	\begin{subfigure}[b]{0.32\textwidth}
		\centering
		\includegraphics[width=\textwidth]{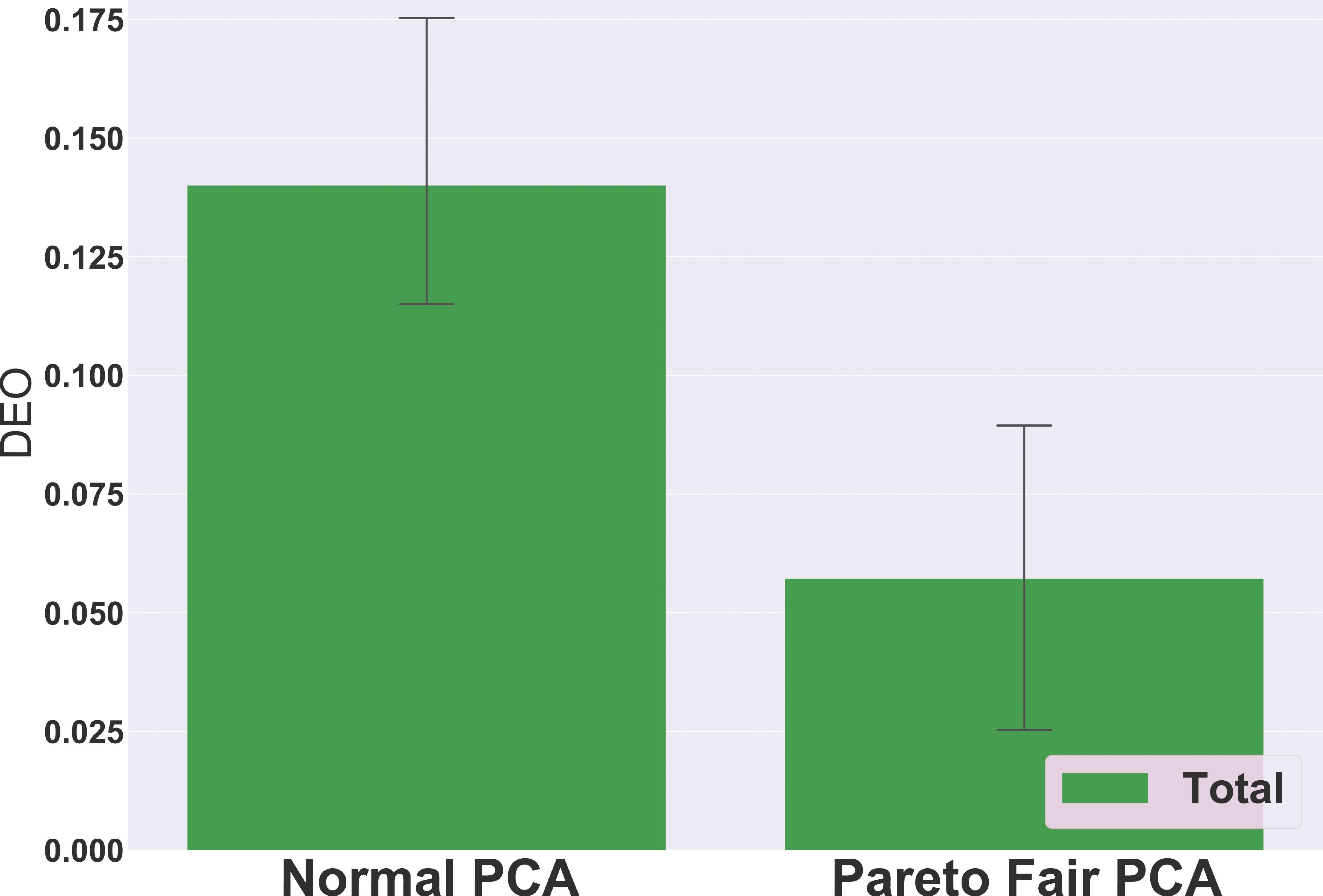}
		\caption{DEO}	\label{fig:compose_credit_deo}
	\end{subfigure}
    \caption{The effect of Pareto fair PCA on a downstream SVM classification task. The first row is using the Adult dataset with gender as its sensitive feature, and the second row is using the Credit dataset with sex as its sensitive feature. In both datasets, we reduce their feature space dimension to $10$ once using normal PCA and once using Pareto fair PCA. Then use the new representation to learn a linear SVM. Column (a) is the accuracy among different groups, (b) is the true positive rate, and (c) is the DEO introduced by~\cite{donini2018empirical} (the lower is better), all on the test dataset. It clearly can be noted that applying Pareto fair PCA can reduce the gap between true positive rates of different groups and enhance the fairness of downstream models. }
    \label{fig:composition}
\end{figure}
\section{Conclusion}\label{sec:con}
In this paper, we cast the fairness problem in dimension reduction algorithms such as PCA as a multi-objective programming. Unlike supervised learning, there is not a clear definition of fairness in unsupervised learning tasks. Thus, we use the notion of balancing between sacrifices and benefits each sensitive group makes or enjoys to define a fairness metric for this problem. These sacrifices or benefits are the consequence of finding the optimal subspace over the whole data rather than using only each protected group's data. Hence, the notion of fairness is to have an equal contribution from each group to the overall reconstruction loss with respect to the reconstruction loss they have on the subspace learned by their own data. This introduces a trade-off between these contributions and overall reconstruction loss.
We propose an efficient multi-objective optimization procedure that can guarantee the convergence to a Pareto stationary point, which has an efficient trade-off between these objectives.
This paper also introduces some interesting problems worthy of future investigations. First, the generalization of the proposed disparity error and pairwise disparity error as fairness metrics in other dimension reduction algorithms and, also, other unsupervised learning tasks. Moreover, it is interesting to investigate the stochastic version of the proposed algorithm and its convergence analysis since finding a descent direction where gradients are noisy might be a challenging task. Also, as noted before, the existing methods, including the one proposed in the present work, require learning a local optimal projection subspace for each group before learning the global fair subspace. One interesting direction is to extend these works to efficiently learn all subspaces together while preserving the fairness of the global subspace. Finally, a thorough theoretical investigation of the composition effects of the proposed fairness measure on downstream tasks such as classification is an interesting open problem.

\newpage

\bibliographystyle{plainnat}
\bibliography{references}

\begin{thebibliography}{38}
\providecommand{\natexlab}[1]{#1}
\providecommand{\url}[1]{\texttt{#1}}
\expandafter\ifx\csname urlstyle\endcsname\relax
  \providecommand{\doi}[1]{doi: #1}\else
  \providecommand{\doi}{doi: \begingroup \urlstyle{rm}\Url}\fi

\bibitem[Agarwal et~al.(2018)Agarwal, Beygelzimer, Dud{\'\i}k, Langford, and
  Wallach]{agarwal2018reductions}
Alekh Agarwal, Alina Beygelzimer, Miroslav Dud{\'\i}k, John Langford, and Hanna
  Wallach.
\newblock A reductions approach to fair classification.
\newblock \emph{arXiv preprint arXiv:1803.02453}, 2018.

\bibitem[Barocas et~al.(2017)Barocas, Hardt, and
  Narayanan]{barocas2017fairness}
Solon Barocas, Moritz Hardt, and Arvind Narayanan.
\newblock Fairness in machine learning.
\newblock \emph{NIPS Tutorial}, 2017.

\bibitem[Calders and Verwer(2010)]{calders2010three}
Toon Calders and Sicco Verwer.
\newblock Three naive bayes approaches for discrimination-free classification.
\newblock \emph{Data Mining and Knowledge Discovery}, 21\penalty0 (2):\penalty0
  277--292, 2010.

\bibitem[Calders et~al.(2009)Calders, Kamiran, and
  Pechenizkiy]{calders2009building}
Toon Calders, Faisal Kamiran, and Mykola Pechenizkiy.
\newblock Building classifiers with independency constraints.
\newblock In \emph{2009 IEEE International Conference on Data Mining
  Workshops}, pages 13--18. IEEE, 2009.

\bibitem[Creager et~al.(2019)Creager, Madras, Jacobsen, Weis, Swersky, Pitassi,
  and Zemel]{creager2019flexibly}
Elliot Creager, David Madras, J{\"o}rn-Henrik Jacobsen, Marissa~A Weis, Kevin
  Swersky, Toniann Pitassi, and Richard Zemel.
\newblock Flexibly fair representation learning by disentanglement.
\newblock \emph{arXiv preprint arXiv:1906.02589}, 2019.

\bibitem[Donini et~al.(2018)Donini, Oneto, Ben-David, Shawe-Taylor, and
  Pontil]{donini2018empirical}
Michele Donini, Luca Oneto, Shai Ben-David, John~S Shawe-Taylor, and
  Massimiliano Pontil.
\newblock Empirical risk minimization under fairness constraints.
\newblock In \emph{Advances in Neural Information Processing Systems}, pages
  2791--2801, 2018.

\bibitem[Dwork and Ilvento(2018)]{dwork2018fairness}
Cynthia Dwork and Christina Ilvento.
\newblock Fairness under composition.
\newblock \emph{arXiv preprint arXiv:1806.06122}, 2018.

\bibitem[Dwork et~al.(2012)Dwork, Hardt, Pitassi, Reingold, and
  Zemel]{dwork2012fairness}
Cynthia Dwork, Moritz Hardt, Toniann Pitassi, Omer Reingold, and Richard Zemel.
\newblock Fairness through awareness.
\newblock In \emph{Proceedings of the 3rd innovations in theoretical computer
  science conference}, pages 214--226. ACM, 2012.

\bibitem[Dwork et~al.(2018)Dwork, Immorlica, Kalai, and
  Leiserson]{dwork2018decoupled}
Cynthia Dwork, Nicole Immorlica, Adam~Tauman Kalai, and Max Leiserson.
\newblock Decoupled classifiers for group-fair and efficient machine learning.
\newblock In \emph{Conference on Fairness, Accountability and Transparency},
  pages 119--133, 2018.

\bibitem[Ehrgott(2006)]{ehrgott2006discussion}
Matthias Ehrgott.
\newblock A discussion of scalarization techniques for multiple objective
  integer programming.
\newblock \emph{Annals of Operations Research}, 147\penalty0 (1):\penalty0
  343--360, 2006.

\bibitem[Feldman et~al.(2015)Feldman, Friedler, Moeller, Scheidegger, and
  Venkatasubramanian]{feldman2015certifying}
Michael Feldman, Sorelle~A Friedler, John Moeller, Carlos Scheidegger, and
  Suresh Venkatasubramanian.
\newblock Certifying and removing disparate impact.
\newblock In \emph{Proceedings of the 21th ACM SIGKDD International Conference
  on Knowledge Discovery and Data Mining}, pages 259--268. ACM, 2015.

\bibitem[Fliege and Svaiter(2000)]{fliege2000steepest}
J{\"o}rg Fliege and Benar~Fux Svaiter.
\newblock Steepest descent methods for multicriteria optimization.
\newblock \emph{Mathematical Methods of Operations Research}, 51\penalty0
  (3):\penalty0 479--494, 2000.

\bibitem[Ghili et~al.(2019)Ghili, Kazemi, and Karbasi]{ghili2019eliminating}
Soheil Ghili, Ehsan Kazemi, and Amin Karbasi.
\newblock Eliminating latent discrimination: Train then mask.
\newblock In \emph{Proceedings of the AAAI Conference on Artificial
  Intelligence}, volume~33, pages 3672--3680, 2019.

\bibitem[Goh et~al.(2016)Goh, Cotter, Gupta, and
  Friedlander]{goh2016satisfying}
Gabriel Goh, Andrew Cotter, Maya Gupta, and Michael~P Friedlander.
\newblock Satisfying real-world goals with dataset constraints.
\newblock In \emph{Advances in Neural Information Processing Systems}, pages
  2415--2423, 2016.

\bibitem[Hardt et~al.(2016)Hardt, Price, Srebro, et~al.]{hardt2016equality}
Moritz Hardt, Eric Price, Nati Srebro, et~al.
\newblock Equality of opportunity in supervised learning.
\newblock In \emph{Advances in neural information processing systems}, pages
  3315--3323, 2016.

\bibitem[Kamani et~al.(2016)Kamani, Farhat, Wistar, and Wang]{kamani2016shape}
Mohammad~Mahdi Kamani, Farshid Farhat, Stephen Wistar, and James~Z Wang.
\newblock Shape matching using skeleton context for automated bow echo
  detection.
\newblock In \emph{IEEE International Conference on Big Data}, pages 901--908,
  2016.

\bibitem[Kamani et~al.(2018)Kamani, Farhat, Wistar, and
  Wang]{kamani2018skeleton}
Mohammad~Mahdi Kamani, Farshid Farhat, Stephen Wistar, and James~Z Wang.
\newblock Skeleton matching with applications in severe weather detection.
\newblock \emph{Applied Soft Computing}, 70:\penalty0 1154--1166, 2018.

\bibitem[Kamani et~al.(2019)Kamani, Farhang, Mahdavi, and
  Wang]{kamani2019targeted}
Mohammad~Mahdi Kamani, Sadegh Farhang, Mehrdad Mahdavi, and James~Z Wang.
\newblock Targeted meta-learning for critical incident detection in weather
  data.
\newblock \emph{International Conference on Machine Learning, Workshop on
  "Climate Change: How Can AI Help?"}, 2019.

\bibitem[Kamiran and Calders(2009)]{kamiran2009classifying}
Faisal Kamiran and Toon Calders.
\newblock Classifying without discriminating.
\newblock In \emph{2009 2nd International Conference on Computer, Control and
  Communication}, pages 1--6. IEEE, 2009.

\bibitem[Kamishima et~al.(2011)Kamishima, Akaho, and
  Sakuma]{kamishima2011fairness}
Toshihiro Kamishima, Shotaro Akaho, and Jun Sakuma.
\newblock Fairness-aware learning through regularization approach.
\newblock In \emph{2011 IEEE 11th International Conference on Data Mining
  Workshops}, pages 643--650. IEEE, 2011.

\bibitem[Kannan et~al.(2019)Kannan, Roth, and Ziani]{kannan2019downstream}
Sampath Kannan, Aaron Roth, and Juba Ziani.
\newblock Downstream effects of affirmative action.
\newblock In \emph{Proceedings of the Conference on Fairness, Accountability,
  and Transparency}, pages 240--248. ACM, 2019.

\bibitem[Kearns and Roth(2019)]{kearns2019ethical}
Michael Kearns and Aaron Roth.
\newblock \emph{The Ethical Algorithm: The Science of Socially Aware Algorithm
  Design}.
\newblock Oxford University Press, 2019.

\bibitem[Lipton et~al.(2017)Lipton, Chouldechova, and McAuley]{lipton2017does}
Zachary~C Lipton, Alexandra Chouldechova, and Julian McAuley.
\newblock Does mitigating ml’s disparate impact require disparate treatment?
\newblock \emph{stat}, 1050:\penalty0 19, 2017.

\bibitem[Mahdavi et~al.(2013)Mahdavi, Yang, and Jin]{mahdavi2013stochastic}
Mehrdad Mahdavi, Tianbao Yang, and Rong Jin.
\newblock Stochastic convex optimization with multiple objectives.
\newblock In \emph{Advances in Neural Information Processing Systems}, pages
  1115--1123, 2013.

\bibitem[Miettinen(2012)]{miettinen2012nonlinear}
Kaisa Miettinen.
\newblock \emph{Nonlinear multiobjective optimization}, volume~12.
\newblock Springer Science \& Business Media, 2012.

\bibitem[Morgenstern et~al.(2019)Morgenstern, Samadi, Singh, Tantipongpipat,
  and Vempala]{morgenstern2019fair}
Jamie Morgenstern, Samira Samadi, Mohit Singh, Uthaipon Tantipongpipat, and
  Santosh Vempala.
\newblock Fair dimensionality reduction and iterative rounding for sdps.
\newblock \emph{arXiv preprint arXiv:1902.11281}, 2019.

\bibitem[Munoz et~al.(2016)Munoz, of~the President, , Director, of~Science,
  Policy)), for Data~Policy, of~Science, and Policy))]{executive2016big}
Cecilia Munoz, Executive~Office of~the President, , Domestic Policy~Council
  Director, Megan (US Chief Technology Officer Smith~(Office of~Science,
  Technology Policy)), DJ~(Deputy Chief Technology~Officer for Data~Policy,
  Chief Data Scientist Patil~(Office of~Science, and Technology Policy)).
\newblock \emph{Big data: A report on algorithmic systems, opportunity, and
  civil rights}.
\newblock Executive Office of the President, 2016.

\bibitem[Nikbakht and Papakonstantinou(2019)]{nikbakht2019direct}
Hamed Nikbakht and Konstantinos~G Papakonstantinou.
\newblock A direct hamiltonian mcmc approach for reliability estimation.
\newblock \emph{arXiv preprint arXiv:1909.03575}, 2019.

\bibitem[Olfat and Aswani(2018)]{olfat2018convex}
Matt Olfat and Anil Aswani.
\newblock Convex formulations for fair principal component analysis.
\newblock \emph{arXiv preprint arXiv:1802.03765}, 2018.

\bibitem[Pleiss et~al.(2017)Pleiss, Raghavan, Wu, Kleinberg, and
  Weinberger]{pleiss2017fairness}
Geoff Pleiss, Manish Raghavan, Felix Wu, Jon Kleinberg, and Kilian~Q
  Weinberger.
\newblock On fairness and calibration.
\newblock In \emph{Advances in Neural Information Processing Systems}, pages
  5680--5689, 2017.

\bibitem[Samadi et~al.(2018)Samadi, Tantipongpipat, Morgenstern, Singh, and
  Vempala]{samadi2018price}
Samira Samadi, Uthaipon Tantipongpipat, Jamie~H Morgenstern, Mohit Singh, and
  Santosh Vempala.
\newblock The price of fair pca: One extra dimension.
\newblock In \emph{Advances in Neural Information Processing Systems}, pages
  10976--10987, 2018.

\bibitem[Shalev-Shwartz and Ben-David(2014)]{shalev2014understanding}
Shai Shalev-Shwartz and Shai Ben-David.
\newblock \emph{Understanding machine learning: From theory to algorithms}.
\newblock Cambridge university press, 2014.

\bibitem[Shamir(2015)]{shamir2015stochastic}
Ohad Shamir.
\newblock A stochastic pca and svd algorithm with an exponential convergence
  rate.
\newblock In \emph{International Conference on Machine Learning}, pages
  144--152, 2015.

\bibitem[Ustun et~al.(2019)Ustun, Liu, and Parkes]{ustun2019fairness}
Berk Ustun, Yang Liu, and David Parkes.
\newblock Fairness without harm: Decoupled classifiers with preference
  guarantees.
\newblock In \emph{International Conference on Machine Learning}, pages
  6373--6382, 2019.

\bibitem[Yao and Huang(2017)]{yao2017beyond}
Sirui Yao and Bert Huang.
\newblock Beyond parity: Fairness objectives for collaborative filtering.
\newblock In \emph{Advances in Neural Information Processing Systems}, pages
  2921--2930, 2017.

\bibitem[Zafar et~al.(2015)Zafar, Valera, Rodriguez, and
  Gummadi]{zafar2015fairness}
Muhammad~Bilal Zafar, Isabel Valera, Manuel~Gomez Rodriguez, and Krishna~P
  Gummadi.
\newblock Fairness constraints: Mechanisms for fair classification.
\newblock \emph{arXiv preprint arXiv:1507.05259}, 2015.

\bibitem[Zafar et~al.(2017{\natexlab{a}})Zafar, Valera, Gomez~Rodriguez, and
  Gummadi]{zafar2017fairness}
Muhammad~Bilal Zafar, Isabel Valera, Manuel Gomez~Rodriguez, and Krishna~P
  Gummadi.
\newblock Fairness beyond disparate treatment \& disparate impact: Learning
  classification without disparate mistreatment.
\newblock In \emph{Proceedings of the 26th International Conference on World
  Wide Web}, pages 1171--1180. International World Wide Web Conferences
  Steering Committee, 2017{\natexlab{a}}.

\bibitem[Zafar et~al.(2017{\natexlab{b}})Zafar, Valera, Rodriguez, Gummadi, and
  Weller]{zafar2017parity}
Muhammad~Bilal Zafar, Isabel Valera, Manuel Rodriguez, Krishna Gummadi, and
  Adrian Weller.
\newblock From parity to preference-based notions of fairness in
  classification.
\newblock In \emph{Advances in Neural Information Processing Systems}, pages
  229--239, 2017{\natexlab{b}}.

\end{thebibliography}

\clearpage

\appendix
\section{Proof of Theorem~\ref{thrm:existence}}
\label{sec:thrm:existence}

\begin{proof}
Consider the following constrained optimization problem:
\begin{equation}
\begin{aligned}
& \sup
& & \sum_{i\in \left[m\right]} \epsilon_i \\
& \text{subject to}
& & f_i\left(\bm{U}\right) +  \epsilon_i = f_i\left(\widetilde{\bm{U}}\right), \; i \in \left[m\right], \\
& & & \epsilon_i \geq 0, \; i \in \left[m\right], \\
\label{eq:const-opt-exist}
\end{aligned}
\vspace{-0.5cm}
\end{equation}
where $\widetilde{\bm{U}}$ is any feasible subspace. By assuming that $f_i(.)$ is convex for $i \in \left[m\right]$, if there is no finite maximum value for this optimization, then the set of proper Pareto optimal solutions is empty.
The main immediate implication of this theorem is that if the objectives are bounded, then a Pareto optimal solution exists for this optimization problem. More specifically, if the solution of this optimization is the objective value of zero, then the $\widetilde{\bm{U}}$ is a Pareto optimal solution.
To prove this theorem, we consider $\bm{U}^*$ to be a proper Pareto optimal solution to the problem (\ref{eq:const-opt-exist}), then there exists a vector $\boldsymbol{\lambda} \in \mathbb{R}^m_{+}$, such that the point $\bm{U}^*$ is a Pareto optimal solution to the problem:
\begin{equation}
    \underset{\bm{U}}{\arg\min}\sum_{i\in \left[m\right]} \lambda_if_i\left(\bm{U}\right)
    \label{eq:ex0}
\end{equation}
Then, from the Pareto optimality we have for every feasible $\bm{U}$:
\begin{equation}
    \sum_{i\in \left[m\right]} \lambda_i\left[f_i\left(\bm{U}\right) - f_i\left(\bm{U}^*\right)\right] \geq 0
    \label{eq:ex1}
\end{equation}
By setting $\bm{U} = \widetilde{\bm{U}}$, we can write:
\begin{equation}
    \sum_{i\in \left[m\right]} \lambda_i\left[f_i\left(\widetilde{\bm{U}}\right) - f_i\left(\bm{U}^*\right)\right] = M^\dagger \geq 0
    \label{eq:ex2}
\end{equation}
Also, from the optimization problem (\ref{eq:const-opt-exist}) since there is not a finite maximum objective value available, for every $\widehat{M} \geq 0$ we can find a $\widehat{\bm{U}}$ such that:
\begin{equation}
    \sum_{i\in \left[m\right]} \left[f_i\left(\widetilde{\bm{U}}\right) - f_i\left(\widehat{\bm{U}}\right)\right] \geq \widehat{M}
    \label{eq:ex3}
\end{equation}
Then, if we set $\lambda_{\text{min}} = \min\left\{\lambda_1,\ldots,\lambda_m\right\}$, we have:
\begin{align}
    \lambda_{\text{min}}\widehat{M} &\leq  \lambda_{\text{min}}  \sum_{i\in \left[m\right]} \left[f_i\left(\widetilde{\bm{U}}\right) - f_i\left(\widehat{\bm{U}}\right)\right] \nonumber\\
    & =    \sum_{i\in \left[m\right]} \lambda_{\text{min}} \left[f_i\left(\widetilde{\bm{U}}\right) - f_i\left(\widehat{\bm{U}}\right)\right] \nonumber\\
    & \leq \sum_{i\in \left[m\right]} \lambda_i \left[f_i\left(\widetilde{\bm{U}}\right) - f_i\left(\widehat{\bm{U}}\right)\right]
\end{align}
If the $\widehat{\bm{U}}$ is chosen to satisfy $ \lambda_{\text{min}}\widehat{M} = M^\dagger$, then we have:
\begin{align}
    \sum_{i\in \left[m\right]} \lambda_i\left[f_i\left(\widetilde{\bm{U}}\right) - f_i\left(\bm{U}^*\right)\right] &\leq \sum_{i\in \left[m\right]} \lambda_i \left[f_i\left(\widetilde{\bm{U}}\right) - f_i\left(\widehat{\bm{U}}\right)\right] \nonumber\\
    \sum_{i\in \left[m\right]} \lambda_i f_i\left(\widehat{\bm{U}}\right) &\leq \sum_{i\in \left[m\right]} \lambda_i f_i\left(\bm{U}^*\right) ,
\end{align}
which contradicts the assumption of Pareto optimality of $\bm{U}^*$, and hence, Pareto optimal set is empty.
\end{proof}
\section{Proof of Lemma~\ref{lemma:descent}}\label{app:lemma}
\begin{proof}
The proof is straightforward and directly follows from   KKT optimally conditions for  problem~(\ref{eq:qop_constrained}), however, we show the derivation  here for completeness.  

First, we note that the minmax optimization problem introduced in (\ref{eq:descent}) to find the descent direction $\mathrm{\mathbf{D}}_t$, can be rewritten as the following equivalent  constrained optimization problem:
\begin{equation}
\begin{aligned}
  (\mathrm{\mathbf{D}}_t, \epsilon_t) =  \arg&\underset{\mathrm{\mathbf{D}} \in \mathbb{R}^{d\times r},\epsilon \in \mathbb{R}_{+}}{\min} \;\; \epsilon + \frac{1}{2}\left\| \mathrm{\mathbf{D}}\right\|_{\text{F}}^2, \\
    &\textrm{s.t.} \; \trace\left(\mathrm{\mathbf{D}}^\top\mathrm{\mathbf{G}}_i^{(t)}\right) \leq \epsilon, \;\; \forall \; 1 \leq i \leq m.
\end{aligned}\label{eq:qop_constrained}
\end{equation}
Forming  the  Lagrangian  of the constrained problem as follows
 \begin{equation}
     \mathcal{L}\left( \mathrm{\mathbf{D}}, \epsilon; \lambda_i \right) = \frac{1}{2}\left\| \mathrm{\mathbf{D}}\right\|_{\text{F}}^2 + \epsilon + \sum_{i=1}^m \lambda_i \left(\trace\left(\mathrm{\mathbf{D}}^\top\mathrm{\mathbf{G}}_i^{(t)}\right) - \epsilon  \right),
 \end{equation}
 and writing  the KKT conditions gives:
\begin{align}\label{eq:kkt1}
        \frac{\partial \mathcal{L}}{\partial \mathrm{\mathbf{D}} } =& \mathrm{\mathbf{D}} + \sum_{i=1}^m \lambda_i^{(t)} \mathrm{\mathbf{G}}_i^{(t)} = 0& \\\label{eq:kkt2}
        \frac{\partial \mathcal{L}}{\partial \epsilon } =& 1 - \sum_{i=1}^m \lambda_i = 0& \\\label{eq:kkt3}
        \frac{\partial \mathcal{L}}{\partial \lambda_i } =& \trace\left(\mathrm{\mathbf{D}}^\top\mathrm{\mathbf{G}}_i^{(t)}\right) - \epsilon = 0&  \forall\; 1 \leq i \leq m \\\label{eq:kkt4}
        &\lambda_i \left( \trace\left(\mathrm{\mathbf{D}}^\top\mathrm{\mathbf{G}}_i^{(t)}\right) - \epsilon  \right) = 0&  \forall\; 1 \leq i \leq m \\\label{eq:kkt5}
        &\lambda_i \geq 0&  \forall\; 1 \leq i \leq m
\end{align}
From~(\ref{eq:kkt1}) we have the following that holds for the descent direction:
\begin{equation}\label{eq:descent-def}
    \mathrm{\mathbf{D}}_t = -\sum_{i=1}^m \lambda_i^{(t)} \mathrm{\mathbf{G}}_i^{(t)},
\end{equation}
where $\boldsymbol{\lambda}^{(t)}=\left[\lambda_1^{(t)},\ldots,\lambda_m^{(t)}\right]^\top$ belongs to $\Delta_m$--  the $m$-dimensional simplex. By plugging  these conditions back to the main problem,  the dual problem can be simplified as:
\begin{equation}\label{eq:qop2}
    \hat{\boldsymbol{\lambda}}^{(t)} = \arg \underset{\boldsymbol{\lambda} \in \Delta_m}{\min} \frac{1}{2} \left\| \sum_{i=1}^m \lambda_i\mathrm{\mathbf{G}}_i^{(t)} \right\|_{\text{F}}^2.
\end{equation}
By solving the dual problem, which is a quadratic programming and using~(\ref{eq:descent-def}) we can find the descent direction from optimal dual variables.

Next, we need to show that the obtained direction  is either $\bm{0}$ or a descent direction to all objectives. If the point $\bm{U}_t$ is a Pareto stationary point, then it means that we cannot find a direction that can decrease all the objectives, without increasing one. Hence, there is no such a $\mathrm{\mathbf{D}}$ that $\trace\left( \mathrm{\mathbf{D}}^\top\mathrm{\mathbf{G}}_i^{(t)}\right) \leq 0$ for all $1 \leq i \leq m$, unless $\mathrm{\mathbf{D}} = \bm{0}$. For points that are not Pareto stationary, consider the following quadratic optimization for every $1 \leq j \leq m$:
\begin{equation}\label{eq:qop3}
    \arg\underset{\beta\in[0,1]}{\min} \frac{1}{2}\left\| (1-\beta)\mathrm{\mathbf{G}}_j^{(t)} - \beta \mathrm{\mathbf{D}}_t\right\|_{\text{F}}^2.
\end{equation}
We can see that this optimization problem is equivalent to the optimization problem in (\ref{eq:qop2}), with $\lambda_i = \beta\hat{\lambda}_i$ for $1\leq i \leq m, \; i\neq j$, and $\lambda_j = 1 - \beta(1-\hat{\lambda}_j)$. This means that the optimum of the quadratic optimization in (\ref{eq:qop3}) happens at $\beta=1$. Then by using the first order optimally condition at optimum point we get:
\begin{equation}\label{eq:bound_trace}
\begin{aligned}
    2\left\| \mathrm{\mathbf{G}}_j^{(t)} +  \mathrm{\mathbf{D}}_t \right\|_{\text{F}}^2 - 2 \trace\left( \left( \mathrm{\mathbf{G}}_j^{(t)} \right)^\top \left(\mathrm{\mathbf{G}}_j^{(t)} +  \mathrm{\mathbf{D}}_t\right)\right) \leq 0\\
    2\left\| \mathrm{\mathbf{G}}_j^{(t)} +  \mathrm{\mathbf{D}}_t \right\|_{\text{F}}^2 - 2 \trace\left( \left( \mathrm{\mathbf{G}}_j^{(t)} 
    +\mathrm{\mathbf{D}}_t - \mathrm{\mathbf{D}}_t\right)^\top \left(\mathrm{\mathbf{G}}_j^{(t)} +  \mathrm{\mathbf{D}}_t\right)\right) \leq 0\\
    \trace\left(\mathrm{\mathbf{D}}_t^\top \left(  \mathrm{\mathbf{G}}_j^{(t)} +  \mathrm{\mathbf{D}}_t\right) \right) \leq 0\\
     \trace\left(\mathrm{\mathbf{D}}_t^\top \mathrm{\mathbf{G}}_j^{(t)}\right) \leq -\left\| \mathrm{\mathbf{D}}_t \right\|_{\text{F}}^2
     \end{aligned}
\end{equation}
which clearly shows that $\mathrm{\mathbf{D}}_t$ is a descent direction for all objectives.
\end{proof}

\section{Proof of Theorem~\ref{theorm:convex}}\label{app:convex}
To prove the Theorems~\ref{theorm:convex} and~\ref{thm:converg}, we first need to show that by properly choosing the regularization parameter $\alpha$ our objectives are smooth. Recall that, our goal is to solve the following multi-objective optimization problem with non-convex components:
\begin{equation}
\label{eqn-g}
    \bm{\mathrm{f}}(\bm{U}) = \left[f_1(\bm{U}), \ldots, f_m(\bm{U})\right]
\end{equation} 
where $ m = 1 + {k \choose 2}$ with $k$ being the number of groups in the sensitive feature.  Also, recall that in the case of fair PCA, we have $f_1(\bm{U}) =   -\frac{1}{2}\trace\left( \bm{U}^\top\bm{X}^\top\bm{X}\bm{U}\right)$ is the overall reconstruction loss, and $f_i(\bm{U}), i = 2, 3, \ldots, m$ are disparity errors for  pair of groups. In what follows we use $\|\cdot\|$ and $\|\cdot\|_{\text{F}}$ to denote the spectral and Frobenius norms of a matrix, respectively.

To prove the theorem,  we first show that all the the individual objective functions are smooth with bounded gradient (i.e., $\bm{\mathrm{f}}(\cdot)$ is component-wise smooth), conditioned that the regularization parameter $\alpha$ satisfies $\alpha \geq \underset{i,j \in [k]}{\max} \gamma_d\left(\bm{X}_i^\top\bm{X}_i \right) - \gamma_1\left(\bm{X}_j^\top\bm{X}_j\right)$ (recall that $\gamma_d(\cdot)$ is the smallest eigenvalue of input PSD matrix). To this end, we follow   the definition of the smooth functions, i.e., $\|\nabla f(\bm{U}) - \nabla f(\bm{V})\|_{\text{F}} \leq L \|\bm{U} - \bm{V}\|_{\text{F}}^2$. 

In particular, for $f_1(\bm{U})$ we have:
\begin{equation*}
    \begin{aligned}
    \|\nabla f_1(\bm{U}) - \nabla f_1(\bm{V}) \|_{\text{F}} &= \| \bm{X}^{\top}\bm{X}\bm{U} - \bm{X}^{\top}\bm{X}\bm{V}\|_{\text{F}} \\
    & \stackrel{\text{\ding{192}}}{\leq} \| \bm{X}^{\top}\bm{X}\| \|\bm{U} - \bm{V}\|_{\text{F}} \\
    & \stackrel{\text{\ding{193}}}{\leq} \gamma_{\max}(\bm{X}^{\top}\bm{X}) \|\bm{U} - \bm{V}\|_{\text{F}},
    \end{aligned}
\end{equation*}
where  the first inequality \ding{192} follows from the fact that for any two matrices $\bm{A}$ and $\bm{B}$ it holds that $\|\bm{A} \bm{B}\|_{\text{F}} \leq \|\bm{A}\| \|\bm{B}\|_{\text{F}}$, and \ding{193} follows from the definition of spectral norm. The above inequality indicates that the objective corresponding to the overall reconstruction error is smooth with parameter $\gamma_{\max}(\bm{X}^{\top}\bm{X})$.

To show the smoothness of disparity errors,  for simplicity, we only focus on one of the objectives between a single pair of sensitive features, say $s_i, s_j$,  as the argument easily generalizes to other objectives/pairs due to symmetry. We also  drop the subscript from function and use  $f(\bm{U})$ to denote the regularized disparity error between groups $s_i$ and $s_j$ defined as
\begin{equation}\label{eq:objective-smoothness}
\begin{aligned}
      f(\bm{U}) &=  \mathcal{E}_i\left(\bm{U}\right) - \mathcal{E}_j\left(\bm{U}\right) + \frac{\alpha}{2} \| \bm{U}\|_{\text{F}}^2 \; \\
      &= \mathcal{L}_i(\bm{U}) - \mathcal{L}_i(\bm{U}_i^*) - \mathcal{L}_j(\bm{U}) + \mathcal{L}_j(\bm{U}_j^*) +  \frac{\alpha}{2} \| \bm{U}\|_{\text{F}}^2 \\
      & = -\frac{1}{2}\trace\left( \bm{U}^\top\bm{X}_i^\top\bm{X}_i\bm{U}\right) + \frac{1}{2}\trace\left( \bm{U}^\top\bm{X}_j^\top\bm{X}_j\bm{U}\right)
      + \frac{\alpha}{2} \| \bm{U}\|_{\text{F}}^2 \\
      & \quad + \frac{1}{2}\trace\left( {\bm{U}_i^*}^\top\bm{X}_i^\top\bm{X}_i\bm{U}_i^*\right) -\frac{1}{2}\trace\left( {\bm{U}_j^*}^\top\bm{X}_j^\top\bm{X}_j\bm{U}_j^*\right)
           \end{aligned}
\end{equation}

Following the definition of smoothness, we have
\begin{equation*}
    \begin{aligned}
    & \|\nabla f(\bm{U}) - \nabla f(\bm{V})\|_{\text{F}} \\
    &= \|\bm{X}_i^{\top}\bm{X}_i \bm{U} - \bm{X}_j^{\top}\bm{X}_j \bm{U} + \alpha\bm{U} - \bm{X}_i^{\top}\bm{X}_i \bm{V} + \bm{X}_j^{\top}\bm{X}_j \bm{V} - \alpha\bm{V} \|_{\text{F}} \\
    &= \left\| \left(\bm{X}_i^{\top}\bm{X}_i - \bm{X}_j^{\top}\bm{X}_j + \alpha \bm{I} \right) \left(\bm{U} - \bm{V}\right)\right\|_{\text{F}}\\
    &\leq \| \bm{X}_i^{\top}\bm{X}_i - \bm{X}_j^{\top}\bm{X}_j + \alpha \bm{I} \|_{2} \|\bm{U} - \bm{V}\|_{\text{F}}
    \end{aligned}
\end{equation*}
Again, we can further upper bound the right hand side by using the definition of the spectral norm of a matrix:
\begin{equation}
\begin{aligned}
    \| \bm{X}_i^{\top}\bm{X}_i - \bm{X}_j^{\top}\bm{X}_j + \alpha \bm{I} \|_2 &= \sup_{\bm{v} \in \mathbb{S}^{d-1}} \bm{v}^{\top} \left( \bm{X}_i^{\top}\bm{X}_i - \bm{X}_j^{\top} \bm{X}_j + \alpha \bm{I} \right) \bm{v} \\
    &\leq \gamma_{\max}\left( \bm{X}_i^{\top}\bm{X}_i \right) - \gamma_{\min}\left( \bm{X}_j^{\top}\bm{X}_j \right) + \alpha
\end{aligned}
\end{equation}
where $\mathbb{S}^{d-1} = \{\bm{x} \in \mathbb{R}^d \; |\; \|\bm{x}\|_2 = 1\}$ is the sphere in $d$ dimensions.

As a result, as long as the regularization parameter $\alpha$ satisfies the following condition 
\textcolor{black}{\[\alpha > \max\left\{0, \gamma_{\max}\left( \bm{X}_j^{\top}\bm{X}_j \right) - \gamma_{\min}\left( \bm{X}_i^{\top}\bm{X}_i \right)\right\},\] }
the disparity error objective between groups $s_i$ and $s_j$ is smooth with smoothness parameter $\gamma_{\max}\left( \bm{X}_i^{\top}\bm{X}_i \right) - \gamma_{\min}\left( \bm{X}_j^{\top}\bm{X}_j \right) + \alpha > 0$. By symmetry, we can derive the smoothness condition for other pairs of groups as well, which results in the following condition on the regularization parameter:
$$\alpha \geq \underset{i,j \in [k]}{\max} \gamma_{\min}\left(\bm{X}_i^\top\bm{X}_i \right) - \gamma_{\max}\left(\bm{X}_j^\top\bm{X}_j\right)$$
to satisfy the smoothness of all objectives $f_i(\cdot), i=2, \ldots, m$. We note that one can use different regularization parameters for each pair   depending on the eigen-gap between their covariance matrices as well.

We now turn to prove the convergence rate of the proposed algorithm to a Pareto fair subspace in general case as stated in~(\ref{eqn-g}), where we assume that the individual loss functions $f_i(\bm{U}), i=1,2, \ldots, m$ satisfy Lipschitz continuous  gradient condition (smoothness) with smoothness parameters $L_i, i=1,2, \ldots, m$. We also use $L$ to denote the maximum smoothness parameter, i.e., $L = \max_{i=1, 2, \ldots, m} L_i$.

\begin{proof}
The proof begins by first bounding the difference in function values of each objective $f_i\left(\bm{U}_{t}\right) -
f_i(\bm{U}^*)$, individually, following the  convexity assumption:
\begin{equation*}
    f_i(\bm{U}_{t}) - f_i(\bm{U}^*) \leq \trace \left(\left(\mathrm{\mathbf{G}}_i^{(t)}\right)^\top\left(\bm{U}_{t} - \bm{U}^*\right)\right)
\end{equation*}
Then we can multiply both sides by $\hat{\lambda}_i$ and sum over $i$:
\begin{equation*}
    \begin{aligned}
    \sum_{i=1}^m \hat{\lambda}_i^{(t)} \left(f_i(\bm{U}_{t}) - f_i(\bm{U}^*)\right) & \leq \sum_{i=1}^m \hat{\lambda}_i^{(t)} \trace \left(\left(\mathrm{\mathbf{G}}_i^{(t)}\right)^\top\left(\bm{U}_{t} - \bm{U}^*\right)\right)\\
    &  =  \trace \left(\left(\sum_{i=1}^m \hat{\lambda}_i^{(t)} \mathrm{\mathbf{G}}_i^{(t)}\right)^\top\left(\bm{U}_{t} - \bm{U}^*\right)\right)\\
    & = -\trace \left(\mathrm{\mathbf{D}}_t^\top\left(\bm{U}_{t} - \bm{U}^*\right)\right)\\
    & = \frac{1}{\eta} \trace\left( \left(\bm{U}_{t} - \bm{U}_{t+1}\right)^\top\left(\bm{U}_{t} - \bm{U}^*\right)\right)\\
    & = \frac{1}{2\eta} \left(\lVert \bm{U}_{t} - \bm{U}^* \rVert^2_{\text{F}} + \lVert \bm{U}_{t} - \bm{U}_{t+1}
\rVert^2_{\text{F}} - \lVert \bm{U}_{t+1} - \bm{U}^* \rVert^2_{\text{F}} \right)  \\
    & = \frac{1}{2\eta} \left(\lVert \bm{U}_{t} - \bm{U}^* \rVert^2_{\text{F}} - \lVert \bm{U}_{t+1} - \bm{U}^*
\rVert^2_{\text{F}} \right) + \frac{\eta}{2} \lVert\mathrm{\mathbf{D}}_t \rVert^2_{\text{F}}  \\
    &\stackrel{\text{\ding{192}}}{\leq} \frac{1}{2\eta} \left(\lVert \bm{U}_{t} - \bm{U}^* \rVert^2_{\text{F}} - \lVert \bm{U}_{t+1} - \bm{U}^*\rVert^2_{\text{F}} \right) + \frac{\eta L^2}{2}  \\
\end{aligned}
\end{equation*}
where \ding{192} follows from the smoothness assumption and definition of $L$. By summing up above inequality for all iterations  $t=1,2,\ldots, T$ gives:
\begin{align}
   \nonumber\sum_{t=1}^{T} \sum_{i=1}^m \hat{\lambda}_i^{(t)} \left(f_i(\bm{U}_{t}) - f_i(\bm{U}^*)\right) &\leq  \frac{1}{2\eta} \sum_{t=1}^{T} \left(\lVert
\bm{U}_{t} - \bm{U}^* \rVert^2_{\text{F}} - \lVert \bm{U}_{t+1} - \bm{U}^* \rVert^2_{\text{F}}  + \frac{\eta L^2}{2}\right) \notag \\\nonumber
   & = \frac{1}{2\eta} \left(\lVert \bm{U}_{1} - \bm{U}^* \rVert^2_{\text{F}} - \lVert \bm{U}_{T} - \bm{U}^{*}
\rVert^2_{\text{F}} \right) + \frac{\eta L^2 T}{2} \\\label{eq:lastconv}
   &\leq \frac{1}{2\eta} \lVert \bm{U}_{1} - \bm{U}^* \rVert^2_\text{F} + \frac{\eta L^2 T}{2} 
\end{align}
For the left hand side, since the $f_i(\bm{U}_{t})$ is a decreasing function by increasing $t$, we can bound it by:
\begin{align}
    \nonumber\sum_{t=1}^{T} \sum_{i=1}^m \hat{\lambda}_i^{(t)} \left(f_i(\bm{U}_{t}) - f_i(\bm{U}^*)\right) & \geq \sum_{i=1}^{m}\left(\sum_{t=1}^T \hat{\lambda}_i^{(t)}\right) \left(f_i(\bm{U}_{T}) - f_i(\bm{U}^*)\right) \\\label{eq:alpha_avg}
    & = \sum_{i=1}^{m} T\cdot \bar{\lambda}_i \left(f_i(\bm{U}_{T}) - f_i(\bm{U}^*)\right),
\end{align}
where $\bar{\lambda}_i = \frac{1}{T}\sum_{t=1}^T \hat{\lambda}_i^{(t)}$. By plugging (\ref{eq:alpha_avg}) back into (\ref{eq:lastconv}) we have:
\begin{equation}
    \sum_{i=1}^{m} \bar{\lambda}_i \left(f_i(\bm{U}_{T}) - f_i(\bm{U}^*)\right) \leq \frac{1}{2\eta T} R^2 + \frac{\eta L^2}{2}, 
\end{equation}
where $\lVert \bm{U}_{1} - \bm{U}^* \rVert^2_{\text{F}} = R^2$. By setting $\eta=\frac{R}{L\sqrt{T}}$, the convergence inequality reduces to:
\begin{equation}
    \sum_{i=1}^{m} \bar{\lambda}_i \left(f_i(\bm{U}_{T}) - f_i(\bm{U}^*)\right) \leq \frac{R L}{2\sqrt{T}}, 
\end{equation}
We note that by setting $\beta = \sqrt{T}/R$, the sufficient decrease condition in (\ref{eq:back}) is satisfied if the backtacking is employed.
\end{proof}

\section{Proof of Theorem~\ref{thm:converg}}\label{app:nonconvex}

\begin{proof}[Proof of Theorem~\ref{thm:converg}]
The proof proceeds using the smoothness condition. In particular, for a smooth function $f:\mathbb{R}^{d \times r} \mapsto \mathbb{R}$ with smoothness parameter $L$  it holds that (descent lemma),
\[ f(\bm{V}) \leq f(\bm{U}) + \trace\left(\nabla f(\bm{U}), \bm{V} - \bm{U}\right) + \frac{L}{2} \| \bm{V} - \bm{U}\|_{\text{F}}^2. \]

From the backtracking line search we can find the learning rate at each step that gives us the maximum decrease. To that end, we will start from $\frac{1}{2}$ and decrease it each time by half until all the objective have a maximum decrease defined in (\ref{eq:back}). Thus, if an $\eta$ satisfies the condition the one step before that, $2\eta^*$, there is at least one objective not satisfying the condition. For instance, we consider the $i$th objective does not satisfy the condition with $2\eta^*$:
 \begin{equation}\label{eq:maxdec}
     f_i\left(\bm{U}_t + \left(2\eta^*\right)\mathrm{\mathbf{D}}_t\right) \geq f_i(\bm{U}_t) + \beta \left(2\eta^*\right) \trace \left( \mathrm{\mathbf{D}}_t^\top\mathrm{\mathbf{G}}_i^{(t)}\right)
\end{equation}
Now from the Lipschitz continuity of the function as:
\begin{align}\nonumber
    \nonumber f_i\left(\bm{U}_{t+1}\right) &\leq f_i(\bm{U}_t) + \trace\left(\nabla f_i(\bm{U}_t)^\top \left(\bm{U}_{t+1} - \bm{U}_t\right)\right) + \frac{L_i}{2} \left\| \bm{U}_{t+1} - \bm{U}_t\right\|_{\text{F}}^2 \\
    f_i\left(\bm{U}_t + \left(2\eta^*\right)\mathrm{\mathbf{D}}_t\right) & \leq f_i(\bm{U}_t) + \left(2\eta^*\right)\trace\left(\nabla f_i(\bm{U}_t)^\top \mathrm{\mathbf{D}}_t\right) + \frac{L_i\left(2\eta^*\right)^2}{2} \| \mathrm{\mathbf{D}}_t\|_{\text{F}}^2,\label{eq:lipschitz}
\end{align}
We proceed by combining (\ref{eq:maxdec}) with this Lipschitz continuity (\ref{eq:lipschitz}) inequality which results in:
\begin{equation}
    \trace \left( \mathrm{\mathbf{D}}_t^\top\mathrm{\mathbf{G}}_i^{(t)}\right) \geq \frac{-L_i\eta^*}{1-\beta} \| \mathrm{\mathbf{D}}_t\|_{\text{F}}^2,
\end{equation}
Also, from (\ref{eq:bound_trace}), we note that the left hand side  term has a upper bound of $-\| \mathrm{\mathbf{D}}_t\|_{\text{F}}^2$,  implying
\begin{equation}\label{eq:eta_bound}
    \eta^* \geq \frac{1-\beta}{L_i},
\end{equation}
which we can replace $L_i$ with $L_{\text{max}} = \underset{1\leq i \leq m}{\max} L_i$, to obtain the lower bound on learning rate, that is $\eta_t \geq C_1 = \min\{1,\frac{1-\beta}{L_{\text{max}}}\}$. Hence, by the choice of learning rate, we know that at every step, we have the maximum decrease for every objective $1\leq i \leq m$:
\begin{equation*}
    \begin{aligned}
        f_i\left(\bm{U}_{t+1} \right) &\leq f_i(\bm{U}_t) + \beta \eta_t \trace \left( \mathrm{\mathbf{D}}_t^\top\mathrm{\mathbf{G}}_i^{(t)}\right)\\
         &\stackrel{\text{\ding{192}}}{\leq} f_i(\bm{U}_t) -
\beta \eta_t\| \mathrm{\mathbf{D}}_t\|_{\text{F}}^2 \\
    &\stackrel{\text{\ding{193}}}{\leq} f_i(\bm{U}_t) -
\beta C_1\| \mathrm{\mathbf{D}}_t\|_{\text{F}}^2 
    \end{aligned}
\end{equation*}
where \ding{192} comes from Lemma~\ref{lemma:descent} and (\ref{eq:bound_trace}), and  \ding{193} follows from the bound on $\eta_t$ in (\ref{eq:eta_bound}). 

Summing up the last inequality for all iterations  $t=1,\ldots,T$ and setting $C=\beta C_1$, we obtain:
\begin{equation}
\label{eqn-33}
        \sum_{t=1}^T C\| \mathrm{\mathbf{D}}_t\|_{\text{F}}^2 \leq  \sum_{t=1}^T \left(f_i\left(\bm{U}_{t} \right) - f_i\left(\bm{U}_{t+1} \right) \right), 
\end{equation}
In~(\ref{eqn-33}),  the left hand side is greater than $\underset{t=1, 2, \ldots, T}{\min} CT \| \mathrm{\mathbf{D}}_t\|_{\text{F}}^2$; and the right hand telescopes and can be upper bounded  by $\mathsf{M}_u - \mathsf{M}_l$, where $\mathsf{M}_u=\underset{i=1, 2, \ldots, m }{\max}f_i(\bm{U}_1)$ is the maximum value among objective functions at starting point and $\mathsf{M}_l$ is the lower bound on all objectives. Then the inequality becomes:
\begin{equation}
    \min_{t=1,2, \ldots, T} \|\mathrm{\mathbf{D}}_t  \|_{\text{F}} \leq  \sqrt{\frac{\mathsf{M}_u - \mathsf{M}_l}{ C T}} \end{equation}
indicating that the proposed algorithm convergence to a Pareto stationary point (a point where the descent direction is $\bm{0}$ and none of the objectives can be further improved). 
\end{proof}

\end{document}